\documentclass[11pt]{article}

\usepackage[margin=1in]{geometry}
\usepackage[T1]{fontenc}
\usepackage{lmodern}
\usepackage{amsmath,amssymb,amsthm,mathtools}
\usepackage{enumitem,booktabs,microtype}
\usepackage{xcolor}
\usepackage{algorithm}
\usepackage[noend]{algpseudocode}
\usepackage[round]{natbib}
\usepackage{hyperref}
\hypersetup{colorlinks=true,citecolor=blue!70!black,linkcolor=black,urlcolor=blue!70!black}
\usepackage{aliascnt}
\usepackage[capitalise,nameinlink]{cleveref}

\newtheorem{theorem}{Theorem}[section]
\newaliascnt{lemma}{theorem}
\newtheorem{lemma}[lemma]{Lemma}
\aliascntresetthe{lemma}
\newaliascnt{proposition}{theorem}
\newtheorem{proposition}[proposition]{Proposition}
\aliascntresetthe{proposition}
\newaliascnt{corollary}{theorem}
\newtheorem{corollary}[corollary]{Corollary}
\aliascntresetthe{corollary}
\theoremstyle{definition}
\newaliascnt{definition}{theorem}
\newtheorem{definition}[definition]{Definition}
\aliascntresetthe{definition}
\newaliascnt{assumption}{theorem}

\aliascntresetthe{assumption}
\theoremstyle{remark}
\newaliascnt{remark}{theorem}
\newtheorem{remark}[remark]{Remark}
\aliascntresetthe{remark}

\crefname{theorem}{Theorem}{Theorems}
\crefname{lemma}{Lemma}{Lemmas}
\crefname{proposition}{Proposition}{Propositions}
\crefname{corollary}{Corollary}{Corollaries}
\crefname{definition}{Definition}{Definitions}
\crefname{assumption}{Assumption}{Assumptions}
\crefname{remark}{Remark}{Remarks}
\crefname{algorithm}{Algorithm}{Algorithms}
\crefname{equation}{}{}
\Crefname{equation}{Equation}{Equations}
\crefname{section}{Section}{Sections}
\crefname{subsection}{Section}{Sections}
\crefname{subsubsection}{Section}{Sections}
\crefname{appendix}{Appendix}{Appendices}

\newenvironment{restatement}[1]%
  {\par\medskip\noindent\textbf{#1, restated.}\ \itshape}%
  {\par\medskip}

\newcommand{\cX}{\mathcal{X}}
\newcommand{\cH}{\mathcal{H}}
\newcommand{\cD}{\mathcal{D}}
\newcommand{\cS}{\mathcal{S}}
\newcommand{\cP}{\mathcal{P}}
\newcommand{\cF}{\mathcal{F}}
\newcommand{\cG}{\mathcal{G}}
\newcommand{\cV}{\mathcal{V}}
\newcommand{\cI}{\mathcal{I}}
\newcommand{\cA}{\mathcal{A}}
\newcommand{\cB}{\mathcal{B}}
\newcommand{\E}{\mathbb{E}}
\newcommand{\R}{\mathbb{R}}
\newcommand{\MC}{\mathrm{MC}}
\newcommand{\SMC}{\mathrm{SMC}}
\newcommand{\VC}{\mathrm{VC}}
\newcommand{\wtO}{\widetilde{O}}
\newcommand{\Reg}{\mathrm{Reg}}
\newcommand{\Root}{\mathrm{Root}}
\newcommand{\clip}{\mathrm{clip}}
\newcommand{\Hpm}{\cH^{\pm}}
\newcommand{\bh}{\bar h}
\newcommand{\fs}{f^{\star}}
\newcommand{\Qbar}{\overline{Q}}
\newcommand{\vlo}{v^{\mathrm{lo}}}
\newcommand{\vhi}{v^{\mathrm{hi}}}
\newcommand{\ylo}{y^{\mathrm{lo}}}
\newcommand{\yhi}{y^{\mathrm{hi}}}
\newcommand{\ncf}{n_{\mathrm{conf}}}
\newcommand{\npt}{n_{\mathrm{part}}}
\newcommand{\ncv}{n_{\mathrm{cov}}}
\newcommand{\Shad}{\cS_{\mathrm{Had}}}
\newcommand{\Scube}{\cS_{\mathrm{cube}}}
\newcommand{\Ber}{\mathrm{Ber}}
\newcommand{\KL}{\mathrm{KL}}
\newcommand{\gs}{\mathsf{G}}
\newcommand{\cN}{\mathcal{N}}
\newcommand{\cL}{\mathcal{L}}
\newcommand{\DIS}{\mathrm{DIS}}
\newcommand{\dis}{\mathrm{dis}}
\newcommand{\Rcls}{R^{\mathrm{cls}}}
\DeclareMathOperator{\sgn}{sgn}
\DeclareMathOperator*{\argmax}{arg\,max}

\title{\bf The Sample Complexity of Multicalibration with an ERM Oracle}
\author{Princewill Okoroafor\\Harvard University}
\date{}

\begin{document}
\maketitle

\begin{abstract}
A predictor is multicalibrated with respect to a class of subpopulations
(formally, of functions on the inputs, called auditors) if, for every
subpopulation in the class and every value the predictor outputs, the
average outcome among the members of that subpopulation who receive that
value is close to the value itself.  For finite classes, about
$\varepsilon^{-3}$ samples are necessary and sufficient to learn a predictor
with multicalibration error $\varepsilon$, but the algorithms that attain
this rate enumerate the class, which is impossible for the large classes of
interest.  We ask how many samples suffice when the learner accesses the
class only through an empirical risk minimization (ERM) oracle.  We give an
algorithm that uses $\wtO(d/\varepsilon^{3})$ samples and polynomially many
ERM calls, where $d$ is the VC dimension of the class.  Its dependence on $\varepsilon$ is optimal for all learners.  When $d$ is at least about
$1/\varepsilon$, every learner that can only access the class through
polynomially many ERM calls needs $\widetilde\Omega(d/\varepsilon^{3})$
samples.  Learners with unlimited computation need only
$\wtO(\varepsilon^{-3}+d\varepsilon^{-2})$ samples, so the extra factor
$1/\varepsilon$ on the dimension term is the price of oracle efficiency.
For swap multicalibration, which allows a different function from the class
for each predicted value, $\widetilde\Theta(d/\varepsilon^{3})$ samples are
necessary and sufficient.  For estimating the error of a given predictor,
the ordinary notion is the harder one: the swap error can be estimated with an ERM
oracle, whereas the ordinary error is NP-hard to approximate even for
classes with linear-time ERM.  Consequently, there are classes for which it is
computationally easier to learn a multicalibrated predictor than to compute
the multicalibration error of a given predictor.
\end{abstract}

\section{Introduction}
\label{sec:intro}

Consider a risk score that a hospital uses to decide which patients receive
a follow-up visit.  Calibration asks that among patients with score $0.3$,
about $30\%$ have the adverse outcome.  This is a guarantee about the
population as a whole; it says nothing about older patients, patients from a
particular clinic, or patients at the intersection of several such
attributes, and the score may be systematically too high on one of these
groups and too low on another.  Recalibrating separately on each group of
interest fails as soon as the groups overlap or are too numerous to list,
which is the typical situation when the groups are described by, say, all
small decision trees over the patient features.

Multicalibration, introduced by \citet{hebert2018multicalibration}, asks for
calibration simultaneously on every group in a rich class.  Following the
literature, we describe each group by an \emph{auditor}
$h:\cX\to\{-1,+1\}$, for instance the function that is $+1$ on the group
and $-1$ off it, and write $\cH$ for the class of auditors.  We ask for a
predictor $p:\cX\to[0,1]$ taking finitely many values such that
\begin{equation}
  \label{eq:intro-mc}
  \sup_{h\in\cH}\ \sum_{v}
  \Bigl|\,\E\bigl[h(X)\,\mathbf 1\{p(X)=v\}\,(v-Y)\bigr]\Bigr|
  \ \le\ \varepsilon .
\end{equation}
Here $(X,Y)$ is drawn from an unknown distribution over inputs and outcomes
in $[0,1]$, the sum ranges over the values of $p$, and each term measures,
with each input weighted by $h$, how far the outcome among inputs that
receive prediction $v$ is from $v$.  We call the left side of~\eqref{eq:intro-mc} the
\emph{multicalibration error} of $p$.  Multicalibration is a strong
guarantee: it implies that the predictor can be post-processed into
near-optimal decisions for a broad family of loss functions
\citep{gopalan2022omnipredictors}, and it is equivalent to a form of
indistinguishability from the true outcome distribution
\citep{dwork2021outcome}.

Three properties matter for an algorithm that learns such a predictor.  The
first is \emph{data}: the number of labeled samples should be as small as
possible, since labels are often the scarce resource.  The second is
\emph{computation}: the class $\cH$ is usually exponentially large or
infinite, so the algorithm should access it only through a learning
primitive.  Following the standard model, we call an algorithm
\emph{oracle-efficient} if it accesses $\cH$ only through an oracle for
weighted empirical risk minimization (ERM) over $\cH$, and makes a number of
oracle calls and other operations polynomial in $1/\varepsilon$ and the VC
dimension $d$ of $\cH$.  The output should also be efficient to use: its
prediction at a given input should be computable in polynomial time from
polynomially many auditors returned by the oracle, together with a
polynomial amount of stored data.  The third is \emph{determinism}: the output should
be a fixed function of the input, so that the same patient always receives
the same score and the guarantee~\eqref{eq:intro-mc} applies to the score
that is actually reported rather than to a random draw.  (The learning
algorithm itself may use randomness.)

\citet{collina2026sample} recently showed that, for finite classes,
$\wtO(\varepsilon^{-3}\log|\cH|)$ samples suffice and that
$\widetilde\Omega(\varepsilon^{-3})$ samples are necessary, but their
algorithm outputs a randomized predictor and enumerates the class.
\citet{noarov2026optimal} then gave a deterministic predictor with
$\wtO(\varepsilon^{-3}+\varepsilon^{-2}\log|\cH|)$ samples, and handled
infinite classes through a supplied finite cover; their algorithm also runs
in time polynomial in the size of the class or cover.  Earlier
oracle-efficient algorithms, based on boosting, use $\wtO(\varepsilon^{-7})$
samples or more (\cref{sec:related}).  This leaves the following questions
open.
\begin{enumerate}[label=(\arabic*),leftmargin=2em]
\item \emph{Can an oracle-efficient algorithm learn a multicalibrated
  predictor with the optimal $\varepsilon^{-3}$ dependence on the accuracy,
  and with a deterministic output?}
\item \emph{Does access to the class only through an ERM oracle cost
  samples, compared with enumerating the class?}
\item \emph{How do the answers change for the stronger notion of swap
  multicalibration, and for estimating the multicalibration error of a
  given predictor?}
\end{enumerate}

We address all three.  For the first, we give an oracle-efficient algorithm
whose sample complexity has the optimal $\varepsilon^{-3}$ dependence on the
accuracy, and which can output a deterministic predictor under a mild
assumption on the inputs.  For the second,
our algorithm uses $\wtO(d\varepsilon^{-3})$ samples, more than the
$\wtO(\varepsilon^{-3}+d\varepsilon^{-2})$ achievable by enumerating a cover
\citep{noarov2026optimal}; we show that, when $d$ is between about
$1/\varepsilon$ and a polynomial in $1/\varepsilon$, the bound
$\wtO(d\varepsilon^{-3})$ is optimal up to logarithmic factors for learners that access the class only
through polynomially many oracle calls.  For the
third, we determine the sample complexity of swap multicalibration, and show
that estimating the multicalibration error of a given predictor can be
computationally harder than learning a multicalibrated predictor.  \Cref{sec:related}
discusses related work in detail.

\subsection{Main results}
\label{sec:main-results}

Let $\cH\subseteq\{-1,+1\}^{\cX}$ have VC dimension $d\ge1$, and let $\cD$ be
a distribution over $\cX\times[0,1]$.  For an integer $K\ge1$, our
predictors take values in the grid $\{0,1/K,\ldots,1\}$.  A \emph{randomized
predictor} $Q$ assigns to each input $x$ a distribution $Q_x$ over the grid;
it is used by drawing a fresh value from $Q_x$ at each query, and its
multicalibration error is defined as in~\eqref{eq:intro-mc} with
$\mathbf 1\{p(X)=v\}$ replaced by the probability $Q_X(v)$
(\cref{def:mc}).  A \emph{weighted ERM oracle} for $\cH$ receives points
$x_1,\ldots,x_n$ with real weights $u_1,\ldots,u_n$ and returns an auditor
$h\in\cH$ maximizing $\sum_i u_ih(x_i)$, breaking ties in an arbitrary way
(\cref{def:erm}).  A \emph{black-box ERM learner} accesses
$\cH$ only through such an oracle, with an unknown rule for breaking ties,
but otherwise has unlimited computation (\cref{def:black-box});
\cref{alg:main} is such a learner.  Our learning algorithms also work when
the oracle is only $\xi$-approximate for some $\xi\le1$: on the weights
that these algorithms submit, it may return any auditor whose weighted sum
is within $\xi$ of the maximum (\cref{def:erm,thm:hybrid}).

In our theorems, samples are independent draws from an arbitrary unknown
distribution over $\cX\times[0,1]$, and a learner must succeed for every
such distribution; unless stated otherwise, it succeeds with probability at
least $2/3$.  The notations $\wtO$, $\widetilde\Omega$, and $\widetilde\Theta$ hide factors polylogarithmic in $d$,
$1/\varepsilon$, $1/\delta$, and, in \cref{thm:audit}, $m$.

\begin{theorem}[Oracle-efficient multicalibration]
\label{thm:main}
\mbox{}
\begin{enumerate}[label=(\alph*),leftmargin=2em]
\item \emph{Upper bound.} There is an algorithm (\cref{alg:main}) that,
  given $\varepsilon,\delta\in(0,1)$ and a weighted ERM oracle for $\cH$,
  uses $\wtO((d+\log(1/\delta))/\varepsilon^{3})$ samples and
  $\wtO((d+\log(1/\delta))/\varepsilon^{6})$ oracle calls, and with
  probability at least $1-\delta$ outputs a randomized predictor with
  multicalibration error at most $\varepsilon$.  If $\cX$ carries a
  computable total order (\cref{rem:order}), the output can instead be a
  deterministic predictor, with the same guarantees.
\item \emph{Lower bound.} When $\widetilde\Omega(1/\varepsilon)\le
  d\le\mathrm{poly}(1/\varepsilon)$, let $A$ be a black-box ERM learner that
  makes $\mathrm{poly}(d/\varepsilon)$ oracle calls and, for every class of
  VC dimension at most $d$, outputs a randomized predictor with
  multicalibration error at most $\varepsilon$.  Then $A$ uses
  $\widetilde\Omega(d/\varepsilon^{3})$ samples.
\end{enumerate}
In part~(b), the constants hidden by $\widetilde\Omega$ depend on the degrees
of the polynomials; \cref{cor:erm-lb} gives the explicit form.
\end{theorem}

The best exponent of $1/\varepsilon$ achieved by earlier oracle-efficient
algorithms that we are aware of is $7$ (\cref{sec:related}); part~(a)
improves it to $3$, which cannot be improved \citep{collina2026sample}.  The
dimension, however, enters
multiplicatively, whereas enumeration achieves
$\wtO(\varepsilon^{-3}+d\varepsilon^{-2})$ \citep{noarov2026optimal}.
Part~(b) shows that this is the price of oracle efficiency: for $d$ between
about $1/\varepsilon$ and a polynomial in $1/\varepsilon$, part~(a) is
optimal up to polylogarithmic factors among learners that make polynomially
many oracle calls.  For $d=\varepsilon^{-2}$, for example, such learners
need $\widetilde\Theta(\varepsilon^{-5})$ samples, whereas
$\wtO(\varepsilon^{-4})$ suffice with enumeration.  The hard classes in part~(b) depend on hidden
parameters, so the lower bound concerns access to the class through the
oracle and does not assert computational hardness for any explicitly
represented class (\cref{rem:erm-lb-scope}).

The output of part~(a) is efficient to evaluate.  The randomized predictor
is the uniform average of $\wtO(d\varepsilon^{-3})$ simpler predictors, one
per sample, each computed from $\wtO(\varepsilon^{-3})$ auditors returned by
the oracle.  A prediction at an input is drawn by evaluating one of them,
chosen uniformly at random, and the full distribution of predictions at an
input takes $\wtO(d\varepsilon^{-6})$ auditor evaluations.  The
deterministic predictor is also computable in polynomial time, but besides
auditors returned by the oracle it uses the order on $\cX$ and $\wtO(\varepsilon^{-3}+d\varepsilon^{-2})$ stored sample
inputs, so it is not a function of members of $\cH$ alone
(\cref{rem:representation}).  

Our second result concerns \emph{swap multicalibration}
\citep{gopalan2023swap}, which lets the auditor depend on the predicted
value: it requires
$\sum_v\sup_{h\in\cH}|\E[h(X)\mathbf 1\{p(X)=v\}(v-Y)]|\le\varepsilon$.  It
is at least as strong as~\eqref{eq:intro-mc}.

\begin{theorem}[Sample complexity of swap multicalibration]
\label{thm:swap}
Let $\varepsilon,\delta\in(0,1)$.
\begin{enumerate}[label=(\alph*),leftmargin=2em]
\item \emph{Sample complexity.} There is an algorithm
  (\cref{alg:swap-cover}), running in time exponential in $d$, that uses
  $\wtO((d+\log(1/\delta))/\varepsilon^{3})$ samples and with probability at
  least $1-\delta$ outputs a randomized predictor with swap
  multicalibration error at most $\varepsilon$.  Conversely, there is an absolute
  constant $c>0$ such that, if $\varepsilon\le10^{-6}$, there exist a class
  of VC dimension at most $d$ with linear-time weighted ERM and a family of
  distributions on which
  every learner that achieves swap multicalibration error at most
  $\varepsilon$ uses at least $c\,d/\varepsilon^{3}$ samples.
\item \emph{Oracle-efficient upper bound.} There is an algorithm
  (\cref{alg:swap}) that, given $\varepsilon$, $\delta$, and a weighted ERM
  oracle for $\cH$, uses $\wtO((d+\log(1/\delta))/\varepsilon^{4})$ samples
  and $\wtO((d+\log(1/\delta))/\varepsilon^{7})$ oracle calls, and with
  probability at least $1-\delta$ outputs a randomized predictor with swap
  multicalibration error at most $\varepsilon$.
\end{enumerate}
\end{theorem}

Part~(a) shows that swap multicalibration requires
$\widetilde\Theta(d\varepsilon^{-3})$ samples even with unlimited
computation.  The upper bound uses known techniques: we sharpen the
finite-class bound $\wtO(\varepsilon^{-3}\log^{3/2}|\cH|)$ of
\citet{collina2026sample} to $\wtO(\varepsilon^{-3}\log|\cH|)$ and extend it
to VC classes by a cover.  The lower bound is new, and shows that the
resulting dependence on $d$ is tight.  For ordinary multicalibration, by
contrast, $\wtO(\varepsilon^{-3}+d\varepsilon^{-2})$ samples suffice for a
learner that may enumerate the class.  For learners that access the class only through
polynomially many ERM calls, however, \cref{thm:main} shows that ordinary
multicalibration also requires $\widetilde\Theta(d\varepsilon^{-3})$ samples
when $d$ is between about $1/\varepsilon$ and a polynomial in
$1/\varepsilon$, so the statistical separation
between the two notions relies on the ability to enumerate the class.
Part~(b) gives an oracle-efficient swap algorithm at the cost of a factor
$1/\varepsilon$ in the sample complexity (\cref{rem:quartic}).

Our third result concerns \emph{error estimation}: estimating the
multicalibration error of a given predictor from samples.  This is how a
regulator or a downstream user would check a predictor that someone else
trained.

\begin{theorem}[Error estimation]
\label{thm:audit}
Let $\varepsilon,\delta\in(0,1)$, and consider randomized predictors with at
most $m$ values, given through the probabilities they assign to each value
at each input.  There are absolute constants $c,c_0>0$ such that the
following hold.
\begin{enumerate}[label=(\alph*),leftmargin=2em]
\item \emph{Swap error.} With $\wtO((md+\log(1/\delta))/\varepsilon^2)$
  samples and $2m$ weighted ERM calls, one can estimate the swap
  multicalibration error within $\varepsilon$ with probability at least
  $1-\delta$.  If $d$ is even, $\varepsilon\le1/1200$, and
  $m\le1/(1200\varepsilon)$, then $\Omega(\sqrt{md}/\varepsilon^2)$ samples
  are necessary.
\item \emph{Ordinary error.} With $\wtO((d+m+\log(1/\delta))/\varepsilon^2)$
  samples and $2^m$ weighted ERM calls, one can estimate the
  multicalibration error within $\varepsilon$ with probability at least
  $1-\delta$, and with $\wtO((d+\log(m/\delta))/\varepsilon^2)$ samples and
  fewer than $4m$ calls, within a factor $\sqrt m$ up to an additive
  $\varepsilon$.  If $d$ is even and $\varepsilon\le1/1200$, then
  $\Omega(\sqrt d/\varepsilon^2)$ samples are necessary.  Unless
  $\mathrm{NP}\subseteq\mathrm{BPP}$, no algorithm running in time
  polynomial in $m$ and $d$ estimates the error within a factor $1\pm c_0$,
  or within an additive $c/m$, even for deterministic predictors and classes
  with linear-time weighted ERM.
\end{enumerate}
\end{theorem}

For swap multicalibration, error estimation is oracle-efficient, and for
predictors with $\Theta(1/\varepsilon)$ values it uses $\wtO(d\varepsilon^{-3})$
samples, the same as learning (\cref{thm:swap}).  Estimating the
ordinary error needs fewer samples than estimating the swap error, but more
computation: the error of a given predictor
with $\Theta(1/\varepsilon)$ values can be estimated from
$\wtO(\varepsilon^{-3}+d\varepsilon^{-2})$ samples, but not
oracle-efficiently, even approximately, unless
$\mathrm{NP}\subseteq\mathrm{BPP}$.  For some classes, learning a
multicalibrated predictor with an ERM oracle (\cref{thm:main}) is thus
easier than estimating the error of a given one.

\subsection{Related work}
\label{sec:related}

Multicalibration was introduced by \citet{hebert2018multicalibration}, and
the weaker notion of multiaccuracy was studied by
\citet{hebert2018multicalibration} and \citet{kim2019multiaccuracy}.
  The consequences of multicalibration have been characterized through outcome indistinguishability
\citep{dwork2021outcome}, omniprediction
\citep{gopalan2022omnipredictors,gopalan2023loss}, and swap agnostic learning
\citep{gopalan2023swap}.  It has been extended to moments
\citep{jung2021moment}, online prediction \citep{gupta2022online}, low-degree
auditors \citep{gopalan2022lowdegree}, and multiclass outcomes
\citep{gopalan2024multiclass}.  A separate line of work studies efficient
algorithms, through game dynamics \citep{haghtalab2023unifying}, boosting
with regression oracles \citep{globusharris2023multicalibration}, and online
oracles \citep{garg2024oracle,ghuge2025improved,hu2025efficient,
luo2025improved,farina2026efficient}.

\paragraph{Sample complexity and deterministic multicalibration.}
\citet{shabat2020sample} bounded the number of samples needed for the
empirical multicalibration error of every predictor in a class to be close
to its population error.  \citet{collina2026sample} determined the sample
complexity of $\ell_1$ multicalibration (the error in~\eqref{eq:intro-mc})
for finite classes: $\wtO(\varepsilon^{-3}\log|\cH|)$ samples suffice, by an
online-to-batch conversion of an online algorithm of
\citet{noarov2025highdimensional}, and $\widetilde\Omega(\varepsilon^{-3})$
are necessary, already for classes of size $O(\log^4(1/\varepsilon))$.
Their algorithm enumerates the class and outputs a randomized predictor;
they ask whether randomization is necessary, and what the optimal joint
dependence on $\varepsilon$ and the complexity of the class is, including in
terms of VC dimension.  In their accounting, which suppresses the
logarithmic dependence on the size of the class, the oracle-efficient
algorithms of \citet{hebert2018multicalibration},
\citet{gopalan2022lowdegree}, and \citet{globusharris2023multicalibration}
need $\wtO(\varepsilon^{-7})$, $\wtO(\varepsilon^{-8})$, and
$\wtO(\varepsilon^{-10})$ samples.  \citet{noarov2026optimal} showed that
randomization can be removed: they give a deterministic predictor with
$\wtO(\varepsilon^{-3}+\varepsilon^{-2}\log|\cH|)$ samples, and the same rate
for infinite classes given a finite cover, by running exponential weights
over pairs of an auditor and a choice of sign for each grid value,
implemented efficiently in the signs by a product formula but explicitly in
the auditor.  We follow
their rounding construction for the deterministic part of \cref{thm:main}
and extend its analysis to VC classes without a cover; unlike them, our
learning algorithm itself remains randomized (only its output is
deterministic), whereas for finite classes they also derandomize the
algorithm.  \Cref{thm:main,thm:swap} address both of their questions.

\paragraph{Oracle-efficient learning.}
In the fully adversarial online setting, no algorithm with access to an
optimization oracle can have vanishing regret in time much smaller than the
square root of the number of experts \citep{hazan2016computational}, which
motivates relaxations of the adversary
\citep{rakhlin2011online,dudik2020oracle,syrgkanis2016efficient,block2022smoothed,haghtalab2022oracle};
when the future inputs are known in advance, \citet{kakade2005batch} reduce
online learning to batch ERM.  In the hybrid setting of
\citet{lazaric2009hybrid}, inputs are independent draws from an unknown
distribution while the losses are adversarial; they achieve regret
$O(\sqrt{Td\log T})$ with high probability, but not with an efficient
algorithm.  \citet{wu2024oracle} obtained regret $\wtO(T^{3/4})$ with an ERM
oracle, and \citet{buzaglo2026oracle} obtained expected regret
$O(\sqrt{Td\log(1+T/d)}\log^2T)$ with one ERM call per round, using
follow-the-perturbed-leader \citep{kalai2005efficient} with Gaussian
perturbations; we prove a high-probability version of their bound
(\cref{thm:hybrid}).  Related hybrid learners with constrained adversaries
appear in \citet{okoroafor2026hybrid}.  Online multicalibration algorithms
use other oracles: \citet{garg2024oracle} an online squared-error regression
oracle, \citet{farina2026efficient} an adversarial online learner over the
auditor class, and \citet{ghuge2025improved} an offline oracle that
maximizes jointly over an auditor and a sign vector, with expected average
error $\wtO(T^{-1/4})$ for binary-valued classes when the inputs come from a
known finite set or admit a small separator.  \Cref{thm:np-hard} shows that
this joint problem can be NP-hard even when ordinary weighted ERM over the class
runs in linear time; our algorithm calls only ordinary weighted ERM.

\paragraph{Swap multicalibration.}
\citet{gopalan2023swap} introduced swap multicalibration, showed that the
boosting algorithm of \citet{hebert2018multicalibration} achieves it, and
showed that rounding a multicalibrated predictor to a coarser grid yields
swap multicalibration with a polynomially worse error; \cref{prop:coarsen}
is the analogue of their bucketing argument for our $\ell_1$ error.
\citet{hu2025efficient} gave online swap multicalibration algorithms for
elicitable properties, with cumulative error $\wtO(T^{2/3})$ in our $\ell_1$
norm, given an adversarial online agnostic learner for the class, and
\citet{collina2026sample} convert their algorithm into a batch algorithm for
finite classes with $\wtO(\varepsilon^{-3}\log^{3/2}|\cH|)$ samples.
\citet{luo2025improved} gave efficient algorithms for linear classes, and
\citet{okoroafor2026fast} obtained fast rates for swap-agnostic learning of
proper losses through a reduction to multicalibration.  \Cref{thm:swap}
settles the sample complexity for VC classes, and gives an oracle-efficient
algorithm with a hybrid learner in place of the adversarial one.

\paragraph{Estimating calibration and multicalibration error.}
\citet{blasiok2023unifying} proposed the distance to calibration and a
property-testing view of measuring calibration, and \citet{hu2024testing}
gave a nearly-linear-time tester.  \citet{qiao2026hardness} showed that the
distance to calibration can be NP-hard to compute, and that
$\Theta(\varepsilon^{-3})$ samples are necessary and sufficient for its
empirical estimate to converge.  \citet{kearns2018preventing} showed that
auditing for their subgroup-fairness notions is computationally equivalent
to weak agnostic learning of the class of groups.  For multicalibration,
\citet{derhake2026auditability} study which error measures can be audited
from samples, and \citet{gollakota2025loss} show that nontrivial loss
prediction is equivalent to auditing for multicalibration.
\Cref{thm:audit} concerns the $\ell_1$ and swap multicalibration errors,
and shows that estimating the ordinary error can be NP-hard even when ERM
over the class runs in linear time, in contrast with the equivalence of
\citet{kearns2018preventing}.

\paragraph{Omniprediction and calibration.}
Omnipredictors \citep{gopalan2022omnipredictors} can be learned with fewer
samples than multicalibrated predictors: \citet{okoroafor2025nearoptimal}
gave oracle-efficient randomized omnipredictors with
$\wtO(\varepsilon^{-2})$ samples, times the complexity of a class derived
from the losses, \citet{gibbs2026sample} a deterministic omnipredictor for
proper losses with the same rate, \citet{lu2025sample} sample-efficient
omnipredictors for Lipschitz non-linear losses, and
\citet{balakrishnan2025panprediction} panpredictors, which serve any
downstream task and loss, with $\wtO(\varepsilon^{-3})$ samples
deterministically and $\wtO(\varepsilon^{-2})$ with randomization.  These
guarantees do not directly yield multicalibration; \cref{app:proper}
records one connection, which loses a square root.  For calibration alone,
the $\varepsilon^{-3}$ rate of batch multicalibration mirrors the $T^{2/3}$
barrier of online calibration, which \citet{dagan2025breaking} broke by a
small polynomial factor; lower bounds of $\Omega(T^{0.528})$
\citep{qiao2021stronger} and $\Omega(T^{0.54389})$ \citep{dagan2025breaking}
rule out $\sqrt T$, and for online multicalibration
\citet{collina2026online} show that $T^{2/3}$ is optimal.  The role of
randomization in calibration has a long history: the forecaster of
\citet{foster1998asymptotic} randomizes, and \citet{kakade2008deterministic}
obtain deterministic forecasts for a weaker notion of calibration.

\subsection{Technical overview}
\label{sec:overview}

We describe the ideas behind each of the three theorems in turn.

\subsubsection{\texorpdfstring{Oracle-efficient multicalibration (\cref{thm:main})}{Oracle-efficient multicalibration (Theorem 1.1)}}
\label{sec:overview-main}

\paragraph{Obstacles.}
The algorithms of \citet{collina2026sample} and \citet{noarov2026optimal}
list the auditors, which an ERM oracle does not allow.  Their
$\varepsilon^{-3}$ rate comes from converting an online algorithm into a
batch one, and online learning is hard with only an optimization oracle.  Finally, rounding a randomized predictor can destroy
its calibration when the input distribution has atoms
\citep{noarov2026optimal}.

\paragraph{Our approach.}
For a fixed auditor, the multicalibration error is a sum of absolute biases,
one for each of the $K+1$ grid values, so it equals the largest signed sum
obtained by choosing a sign for each grid value; we call such a choice a
\emph{sign pattern}.  For a fixed sign pattern, this signed sum is the
correlation of the auditor with a single feedback value per example: the
residual of the forecast, multiplied by the sign assigned to the predicted
value.  The algorithm of part~(a) has four stages: it replaces the sign
patterns by a few fixed sign vectors, which we call \emph{tests}, chooses
forecasts against these tests, converts the online guarantee into a batch
one, and, optionally, makes the output deterministic.  The first stage is
new, and the following explains why it is needed.  A natural approach runs, for each sign
pattern, an online learner over $\cH$ that receives this feedback, so that
it competes with every auditor for that pattern, and converts the online
guarantee into a batch one.  With only an ERM oracle, however, online
learning is hard: against adversarial inputs, no algorithm with an
optimization oracle has vanishing regret in time much smaller than the
square root of the size of the class \citep{hazan2016computational}.  Our
samples are independent, so the inputs are random draws from a fixed
distribution and only the feedback is adversarial.  In this \emph{hybrid}
setting \citep{lazaric2009hybrid}, the Gaussian follow-the-perturbed-leader learner of
\citet{buzaglo2026oracle} makes one ERM call per round and has expected
regret $\wtO(\sqrt{Td})$ over $T$ rounds.  The remaining problem is that
there are $2^{K+1}$ sign patterns, and hence $2^{K+1}$ learners.  We replace them by fewer than $4(K+1)$ tests, with one learner per test:
the rows of a Hadamard matrix, whose order is the smallest power of two
that is at least $K+1$, restricted to $K+1$ columns, together with their
negations.  The columns of a Hadamard matrix are orthogonal, so for any
vector $b$ of biases, one for each grid value, the squared correlations
$\langle w,b\rangle^2$ of the rows $w$ average to $\|b\|_2^2$, and some row
satisfies
$|\langle w,b\rangle|\ge\|b\|_2\ge\|b\|_1/\sqrt{K+1}$ (\cref{lem:hadamard}).
The tests therefore lose a factor $\sqrt{K+1}$ in converting back to
multicalibration error, and no fixed family of polynomially many tests does
much better (\cref{prop:distortion-lb}).  Forecasting on a grid of $K+1$ values also incurs an error of order $1/K$
per round.  If the tests are applied to the residuals at the drawn grid
values, the factor $\sqrt{K+1}$ multiplies this error too, inflating it to order $1/\sqrt K$, which leads to $\varepsilon^{-4}$
samples rather than $\varepsilon^{-3}$.

The second stage chooses the forecasts.  The regret bounds say that no
auditor correlates with the feedback of a test much more than the test's
learner does.  So it suffices to choose each forecast so that the online learners
themselves gain nothing: on every round, the predictions of the online
learners, aggregated over the tests by exponential weights, should have
nonpositive correlation with the feedback.  The aggregated predictions give a coefficient for each grid
value, and the forecast is a distribution on one or two adjacent grid values
at which this coefficient changes sign.  Forecasts of this kind go back to
the calibration algorithm of \citet{foster1998asymptotic} and are used
throughout the online approach to multicalibration
\citep{gupta2022online,noarov2025highdimensional,noarov2026optimal}.  A
single grid value cannot make the correlation nonpositive for every outcome,
but such a distribution can, if the residual is measured at the mean of the
forecast rather than at the grid value drawn from it (\cref{lem:root}).  We therefore feed the online learners the residuals at the mean.  This avoids the
$\varepsilon^{-4}$ rate described above: residuals at the mean differ
from residuals at the grid values by at most $1/(2K)$ per round, and this
difference is paid once, outside the conversion (\cref{lem:adjacent}).  In place of enumerating the class, as \citet{noarov2026optimal} do, we give
each test the hybrid learner described above (\cref{thm:online}).  The analysis uses
all the learners at once, so it needs their regret bounds to hold
simultaneously, and we show that the expected-regret bound of
\citet{buzaglo2026oracle} also holds with high probability
(\cref{thm:hybrid}).  Measuring residuals at the mean, the combination with
hybrid learners, and the high-probability bound are new.

The third stage converts the online guarantee into a batch guarantee: the
average of the forecasts is a randomized predictor.  Because each forecast
depends on the earlier samples, uniform convergence does not apply
directly; instead, for each auditor in a cover of $\cH$, which is used only
in the analysis, the deviations form a martingale, which gives
$\wtO(d\varepsilon^{-3})$ samples (\cref{cor:batch-rates}).  The fourth
stage makes the predictor deterministic by fixing its random choices once
per cell of a partition of the inputs.  When the input distribution has
atoms, fixing these choices can destroy calibration, and
\citet{noarov2026optimal} avoid this by restricting the forecasts at
repeated inputs to confidence intervals for their mean outcomes.  We use
their construction; our analysis of it for VC classes
(\cref{thm:vc-rounding}) and a version of our forecasts restricted to these
intervals are new.

\paragraph{The lower bound.}
Part~(b) must hold even for learners with unlimited computation, so it
rests on what the oracle reveals about the class.  Each oracle call is a
query: the learner chooses a weight for each of finitely many inputs, and
the oracle returns an auditor that maximizes the weighted sum of its
values.

The hard instances use a finite input space, divided into about
$1/\varepsilon$ blocks of about $d$ inputs each (up to logarithmic factors).
The mean outcome at each input is its block's center plus or minus a gap of
order $\varepsilon$, with a hidden sign at each input.  A predictor with
small multicalibration error must track most of the hidden signs.
Auditors that are constant on each block force it to predict near the
block centers.  A few \emph{planted} auditors detect any predictor near the
centers that does not track the hidden signs; each equals, at every input,
the hidden sign times a random sign shared by the input's block.  The
construction needs at least as many inputs per block as there are blocks,
which is why the bound requires $d$ to be at least about $1/\varepsilon$.

A planted auditor returned by the oracle would reveal the hidden signs at
once.  The class therefore also contains exponentially many random
auditors, a random code.  For every query whose weights do not depend on
the hidden signs, some random auditor scores at least as well as every
planted auditor, and the oracle returns it instead.  Since the answers then
do not depend on the hidden signs, by induction no query does either, so
the oracle never reveals a planted auditor.  The learner must therefore
learn the hidden signs at about $d/\varepsilon$ inputs from samples, at a
cost of about $\varepsilon^{-2}$ samples each, which gives
$\widetilde\Omega(d\varepsilon^{-3})$ samples in total (\cref{thm:erm-lb}).

\subsubsection{\texorpdfstring{Swap multicalibration (\cref{thm:swap})}{Swap multicalibration (Theorem 1.2)}}
\label{sec:overview-swap}

Swap multicalibration lets the auditor depend on the predicted value, so its
error is a sum over values of separate suprema over the class.  For a fixed
predictor this can be larger than the ordinary error by a factor equal to the
number of values, already for a class of VC dimension one
(\cref{prop:ratio}); conversely, coarsening a predictor converts ordinary
error $\varepsilon$ into swap error $O(\sqrt\varepsilon)$
(\cref{prop:coarsen}).

\paragraph{Upper bound.}
Because the suprema are separate, each grid value has its own auditor,
and including the negation of each auditor takes care of the absolute
value; so the upper bound of part~(a) needs no sign patterns.  We run exponential weights over
the class separately at each grid value, rewarding an auditor by its
correlation with the residual at the mean of the forecast, weighted by the
probability the forecast places on that value, and choose forecasts as in
\cref{thm:main}.  On each
round the residuals at the different values have total magnitude at most
one, so the second-order regret bound of exponential weights gives online
error $\wtO(\sqrt{TK\log N})$ for a class of size $N$, and $\wtO(\varepsilon^{-3}
\log N)$ samples (\cref{thm:swap-finite}).  For a VC class, the labelings of
$\wtO(d\varepsilon^{-2})$ unlabeled inputs give a cover with
$(d/\varepsilon)^{O(d)}$ elements, which yields $\wtO(d\varepsilon^{-3})$ samples in
time exponential in $d$ (\cref{thm:swap-vc}).

\paragraph{Lower bound.}
For the lower bound of part~(a), we place the inputs in about $1/\varepsilon$ blocks of $d$
inputs each.  The mean outcome in each block is the block's center plus or
minus a gap of order $\varepsilon$, with hidden signs that cancel in pairs
within the block, and the class consists of thresholds over the ordered
inputs together with all labelings of a single block.  The predictor that
outputs the block center has small ordinary error, because a single auditor
sees only one block, where the signs cancel.  A swap-multicalibrated
predictor, however, faces a different auditor in each block:
thresholds force it to predict near the centers, and the labelings of each
block then force its predictions to reveal most of the hidden signs.  There are about $d/\varepsilon$ hidden signs, each
perturbing the mean at its inputs by order $\varepsilon$, so Assouad's lemma
requires $\Omega(d\varepsilon^{-3})$ samples (\cref{thm:swap-lb}).

\paragraph{Oracle-efficient upper bound.}
Part~(b) replaces each exponential-weights learner by the hybrid learner of
\cref{thm:hybrid}, over the class and its negations, and adds up the regret
bounds over the grid values (\cref{thm:swap-poly}).  This costs a factor
$1/\varepsilon$: exponential weights at a grid value pay regret proportional
to the probability mass the forecasts place on that value, whereas the regret
of the hybrid learner grows with the number of rounds.  Running the
learner only on rounds selected by the forecast does not help, because the
selected inputs are no longer independent draws (\cref{rem:quartic}).

\subsubsection{\texorpdfstring{Error estimation (\cref{thm:audit})}{Error estimation (Theorem 1.3)}}
\label{sec:overview-audit}

An estimator receives samples and a predictor with $m$ values, and estimates
the predictor's error.  For both swap and ordinary multicalibration, we
analyze the plug-in estimate, the error of the
predictor on the empirical distribution, and the question is how many
samples make it accurate and whether it can be computed with an ERM oracle.

\paragraph{Estimating the swap error.}
The empirical swap error splits into one ERM problem per value: at each
value, two weighted ERM calls find the auditors with the largest positive
and negative bias (\cref{prop:swap-audit}).  Its accuracy requires uniform
convergence over tuples of $m$ auditors, one per value, which costs
$\wtO(md/\varepsilon^2)$ samples (\cref{thm:swap-audit}).  For
$m=\Theta(1/\varepsilon)$ this is $\wtO(d\varepsilon^{-3})$, the same as
learning.

\paragraph{Estimating the ordinary error.}
The ordinary error uses one auditor for all values, so the plug-in estimate
needs uniform convergence only over pairs of an auditor and a sign pattern.
The $2^m$ sign patterns add only $m$ to the logarithm of the number of such
pairs, so $\wtO((d+m)/\varepsilon^2)$ samples suffice
(\cref{prop:mc-audit}).
Computing it, however, is a joint maximization over the auditor and the
sign pattern.  Enumerating the $2^m$ sign patterns takes $2^m$ ERM calls, and
the Hadamard tests of \cref{lem:hadamard} give a $\sqrt m$-approximation with
fewer than $4m$ calls (\cref{prop:mc-audit-approx}).  A reduction from
\textsc{Max-Cut} shows that a constant-factor approximation is NP-hard even
when weighted ERM over the class runs in linear time: each vertex of the graph becomes a grid
value and each edge contributes three inputs, so that a sign pattern encodes
a cut and the best auditor collects the cut edges (\cref{thm:np-hard}).

\paragraph{Sample lower bounds.}
We reuse the block family of the swap lower bound with an unperturbed
alternative: an estimator must distinguish the base means from a random
perturbation whose signs cancel in pairs.  The chi-square divergence between
the two sample distributions stays small unless the number of samples is of
order $\sqrt{Jd}/\varepsilon^2$ for $J$ blocks of $d$ inputs.  With $J=m$
blocks this gives $\Omega(\sqrt{md}/\varepsilon^2)$ samples for estimating
the swap error, and with a single block it gives
$\Omega(\sqrt d/\varepsilon^2)$ for the ordinary error
(\cref{thm:audit-lb}).

\subsection{Paper organization}
\label{sec:organization}

\Cref{sec:prelims} defines multicalibration, the oracles, and hybrid
online learning, states the high-probability regret bound for Gaussian
follow-the-perturbed-leader (\cref{thm:hybrid}), and records the
probabilistic tools.  \Cref{sec:mc} develops our algorithm for
multicalibration: signed tests and the choice of forecasts
(\cref{sec:signs}), the
reduction to hybrid online learning (\cref{sec:reduction}), the conversion
to a batch guarantee (\cref{sec:batch}), and the matching lower bound for
learners that access the class only through an ERM oracle
(\cref{sec:erm-lb}).  \Cref{sec:swap} determines the sample complexity of
swap multicalibration and proves \cref{thm:swap}.  \Cref{sec:audit} studies
error estimation and proves \cref{thm:audit}.  \Cref{sec:det} constructs
deterministic predictors, gives the complete algorithm, and proves
\cref{thm:main}.  \Cref{sec:discussion} discusses open problems.

\Cref{app:hybrid} proves \cref{thm:hybrid}.  \Cref{app:mc} proves the online
and batch guarantees of \cref{sec:mc}, and \cref{app:erm-lb} proves its
lower bound (\cref{thm:erm-lb}).  \Cref{app:swap-all,app:audit,app:det}
contain the deferred proofs for \cref{sec:swap,sec:audit,sec:det}.
\Cref{app:proper} relates proper-loss guarantees to multicalibration, and
\cref{app:params} lists explicit parameter choices.

\section{Setting and definitions}
\label{sec:prelims}

\paragraph{Notation.}
We write $[n]=\{1,\ldots,n\}$, $\Delta(S)$ for the set of probability
distributions on a finite set $S$, and $\log$ for the natural logarithm.  The
letters $C,C'$ denote absolute constants whose values may change between
occurrences.  We write $\clip(a,[\ylo,\yhi])=\min\{\yhi,\max\{\ylo,a\}\}$.
The notation $\wtO(\cdot)$ hides factors polylogarithmic in $d$,
$1/\varepsilon$, $T$, and $1/\delta$.

Let $\cX$ be a set of inputs and
$\cH\subseteq\{-1,+1\}^{\cX}$ a class of auditors of VC dimension $d$.
Throughout, we assume $1\le d<\infty$.  We write $\Hpm=\cH\cup(-\cH)$; its
growth function is at most twice that of $\cH$, so $\VC(\Hpm)\le Cd$.  Let
$\cD$ be a distribution over $\cX\times[0,1]$ with marginal $\cD_X$ on
inputs, and let $\fs(x)=\E[Y\mid X=x]$ denote the conditional mean outcome.
We assume the standard measurability conditions for suprema over $\cH$
\citep[see, e.g.,][]{shalev2014understanding}.

\paragraph{Predictors.}
Let $\cV\subset[0,1]$ be a finite set of prediction values.  A \emph{randomized
predictor} is a map $Q:\cX\to\Delta(\cV)$, and $Q_x(v)$ denotes the
probability it assigns to $v$ at input $x$.  It is used by drawing, at each
query, a fresh value $V\sim Q_X$ independently of $Y$ given $X$.  A
\emph{deterministic predictor} $p:\cX\to\cV$ is identified with the randomized
predictor $Q_x=\delta_{p(x)}$.  For an integer $K\ge1$, we write
\[
  \Lambda=\Lambda_K=\{0,1/K,2/K,\ldots,1\},\qquad m=K+1=|\Lambda|,
\]
for the uniform grid.  (In \cref{sec:audit}, $m$ denotes instead the number
of values of a given predictor, which need not lie on $\Lambda$.)  A distribution on $\Lambda$ has \emph{adjacent
support} if it is supported on one grid value or on two grid values at
distance $1/K$.

\subsection{Multicalibration}
\label{sec:prelims-mc}

\begin{definition}[Ordinary and swap multicalibration error]
\label{def:mc}
For a randomized predictor $Q$ with values in $\cV$, an auditor $h$, and $v\in\cV$,
the \emph{bias} of $Q$ at $v$ with respect to $h$ is
\[
  B_h(Q,v)=\E_{(X,Y)\sim\cD}\bigl[h(X)\,Q_X(v)\,(v-Y)\bigr].
\]
The ($\ell_1$) \emph{multicalibration error} and the \emph{swap
multicalibration error} of $Q$ are
\[
  \MC(Q)=\MC(Q;\cH,\cD)=\sup_{h\in\cH}\sum_{v\in\cV}\bigl|B_h(Q,v)\bigr|,
  \qquad
  \SMC(Q)=\SMC(Q;\cH,\cD)=\sum_{v\in\cV}\sup_{h\in\cH}\bigl|B_h(Q,v)\bigr|.
\]
\end{definition}

Clearly $\MC(Q)\le\SMC(Q)$.  For a deterministic predictor, $B_h(p,v)=\E[h(X)\mathbf 1\{p(X)=v\}(v-Y)]$, so
\cref{def:mc} agrees with~\eqref{eq:intro-mc}.  For a randomized predictor,
$B_h(Q,v)=\E[h(X)\mathbf 1\{V=v\}(v-Y)]$ when $V\sim Q_X$ is drawn
independently of $Y$ given $X$; thus $\MC(Q)$ is the error of the randomized
predictor that draws a fresh value at each query.  This is the
``mean ECE'' notion of \citet{collina2026sample} and the error used by
\citet{noarov2026optimal}.  \citet{collina2026sample} use groups
$g:\cX\to\{0,1\}$ and \citet{noarov2026optimal} use group weights
$g:\cX\to[0,1]$; for $\{0,1\}$-valued groups $g$, the auditor $h=2g-1$ shows
that the two conventions agree up to a factor $3$ once the constant auditor
is included in the class.

For a sequence of inputs, outcomes, and randomized forecasts $(X_t,Y_t,Q_t)_{t\le
T}$, the cumulative errors are
\[
  \MC_T=\sup_{h\in\cH}\sum_{v\in\Lambda}
    \Bigl|\sum_{t=1}^Th(X_t)\,Q_t(X_t)(v)\,(v-Y_t)\Bigr|,
  \qquad
  \SMC_T=\sum_{v\in\Lambda}\sup_{h\in\cH}
    \Bigl|\sum_{t=1}^Th(X_t)\,Q_t(X_t)(v)\,(v-Y_t)\Bigr|,
\]
where $Q_t(x)(v)$ is the probability $Q_t$ assigns to $v$ at $x$.  These are
errors of the announced distributions, not of particular draws from them.

The following representation is the starting point of our reduction.

\begin{lemma}[Signed representation, \citealp{noarov2026optimal}]
\label{lem:signed}
For real numbers $(A_{h,v})_{h\in\cH,v\in\Lambda}$,
\[
  \sup_{h\in\cH}\sum_{v\in\Lambda}|A_{h,v}|
  =\sup_{h\in\cH,\ \sigma\in\{-1,+1\}^{\Lambda}}
  \sum_{v\in\Lambda}\sigma_vA_{h,v}.
\]
In particular,
$\MC(Q)=\sup_{h,\sigma}\E\bigl[h(X)\sum_v\sigma_vQ_X(v)(v-Y)\bigr]$.
\end{lemma}

\begin{proof}
For fixed $h$, maximize coordinatewise:
$\max_{\sigma_v\in\{-1,+1\}}\sigma_vA_{h,v}=|A_{h,v}|$.
\end{proof}

\subsection{Oracles and hybrid online learning}
\label{sec:prelims-oracles}

\begin{definition}[Weighted ERM oracle]
\label{def:erm}
A \emph{weighted ERM oracle} for a class $\cG\subseteq\{-1,+1\}^{\cX}$
receives a finite list of pairs $(x_i,u_i)\in\cX\times\R$ and returns some
$g\in\cG$ with $\sum_iu_ig(x_i)=\max_{g'\in\cG}\sum_iu_ig'(x_i)$.  Each call
counts as one oracle call, and returned hypotheses can be stored and
evaluated at any input.
For $\xi\ge0$, a \emph{$\xi$-approximate} weighted ERM oracle may instead
return any $g\in\cG$ with
$\sum_iu_ig(x_i)\ge\max_{g'\in\cG}\sum_iu_ig'(x_i)-\xi$.
In both cases, we assume that the answer depends only on the input list
and on randomness independent of everything else.
\end{definition}

Since
$\sum_iu_ig(x_i)=\sum_i|u_i|-2\sum_i|u_i|\mathbf 1\{g(x_i)\ne\sgn(u_i)\}$,
this is ordinary ERM for weighted classification with labels $\sgn(u_i)$
and nonnegative weights $|u_i|$.  An ERM oracle for
$\Hpm$ is implemented by two calls to an ERM oracle for $\cH$, and the same
holds for $\xi$-approximate oracles.

\begin{definition}[Hybrid online learning]
\label{def:hybrid}
A \emph{hybrid online learner} for a class $\cG\subseteq\{-1,+1\}^{\cX}$
interacts with an environment for $T$ rounds.  Before round $t$ it holds a
distribution $\pi_t$ over $\cG$ that is determined by its observations in
earlier rounds and its earlier internal randomness.  In round $t$:
\begin{enumerate}[leftmargin=2em]
\item an input $X_t\sim\cD_X$ is drawn independently of everything before;
\item the environment chooses a feedback value $Z_t\in[-1,1]$ as a function
  of $X_t$, everything that happened before round $t$, and additional
  randomness that is independent of the learner's round-$t$ draw;
\item the learner draws $g_t\sim\pi_t$ using fresh internal randomness;
\item the learner observes $(X_t,Z_t)$.
\end{enumerate}
The \emph{mixed prediction} of the learner is
$\bh_t(x)=\E_{g\sim\pi_t}[g(x)]\in[-1,1]$.
\end{definition}

The additional randomness in step~2 may include independent draws from
$\pi_t$ requested by an algorithm that uses the learner; our algorithms use
such draws to estimate mixed predictions.

\paragraph{Gaussian follow-the-perturbed-leader.}
The hybrid learner we use is the Gaussian follow-the-perturbed-leader
algorithm of \citet{buzaglo2026oracle}, which builds on
\citet{kalai2005efficient}.  Fix a perturbation scale $\gamma>0$.  On round
$t$, the learner draws independent standard Gaussians $\gs_{t,1},\ldots,
\gs_{t,t-1}$ and returns
\begin{equation}
  \label{eq:ftpl}
  g_t\in\argmax_{g\in\cG}\ \sum_{s<t}\bigl(Z_s+\gamma\,\gs_{t,s}\bigr)\,g(X_s),
\end{equation}
computed by one call to a weighted ERM oracle for $\cG$ on the $t-1$ past
inputs.  With a $\xi$-approximate oracle, $g_t$ is instead any hypothesis
whose perturbed score is within $\xi$ of the maximum.  The distribution $\pi_t$ is the law of $g_t$; it is determined by
the past observations $(X_s,Z_s)_{s<t}$, and an independent additional draw
from $\pi_t$ repeats~\eqref{eq:ftpl} with fresh Gaussians at the cost of one
more ERM call.  \Citet[Theorem~2]{buzaglo2026oracle} show that this learner
has \emph{expected} regret $O(\sqrt{Td\log(1+T/d)}\log^2T)$ against
adaptively chosen losses.  Since our algorithms run $O(1/\varepsilon)$
learners at once and take a union bound over them, we need the regret bound
to hold with failure probability $\zeta$ at a cost logarithmic in
$1/\zeta$, not the factor $1/\zeta$ that Markov's inequality gives.  \Cref{thm:hybrid} shows that the same algorithm achieves this, with the
regret bound, for an absolute constant $C_0$,
\[
  \Reg(T,d_{\cG},\zeta)
  =C_0\sqrt{T\bigl(d_{\cG}\log(eT)+\log(1/\zeta)\bigr)}\,\log^2(2T)
  =\wtO\Bigl(\sqrt{T\bigl(d_{\cG}+\log(1/\zeta)\bigr)}\Bigr).
\]

\begin{theorem}[High-probability hybrid regret]
\label{thm:hybrid}
Let $\cG\subseteq\{-1,+1\}^{\cX}$ have VC dimension at most $d_{\cG}$, where
$d_{\cG}\ge1$, and let $T\ge1$.  Run Gaussian follow-the-perturbed-leader
\eqref{eq:ftpl} with $\gamma=4\sqrt{\log(2T)}$.  For every input
distribution, every environment as in \cref{def:hybrid}, and every
$\zeta\in(0,1)$, with probability at least $1-\zeta$,
\[
  \sup_{g\in\cG}\sum_{t=1}^TZ_t\bigl(g(X_t)-\bh_t(X_t)\bigr)
  \le\Reg(T,d_{\cG},\zeta).
\]
Each draw from $\pi_t$, including
independent draws in addition to $g_t$, costs one call to a weighted ERM
oracle for $\cG$.  The regret bound also holds, with a larger absolute
constant $C_0$, when every call in~\eqref{eq:ftpl}, including those for the
additional draws, is answered by a $\xi$-approximate oracle with $\xi\le1$.
\end{theorem}

\paragraph{Proof sketch.}
The proof is in \cref{app:hybrid}.  It follows the expected-regret analysis
of \citet{buzaglo2026oracle} up to the point where they take expectations.
With the losses $(1-Z_t)/2$ and $(1+Z_t)/2$ for the two labels, the
correlation regret above is twice the regret of a binary classification
problem.  A potential argument bounds this regret, pathwise, by a sum of
stability terms: the stability term of round $t$ is, roughly, the probability
over the perturbations that hypotheses whose perturbed score is nearly
maximal disagree at the fresh input $X_t$.  \citet{buzaglo2026oracle} show
that the population mass of this disagreement region has an exponential tail
over the joint randomness of the inputs and the perturbations.  We convert
this joint tail bound by Markov's inequality into a bound that holds,
conditional on the realized history, for most histories, and sum the
resulting conditional means over rounds.  Finally, Freedman's inequality
passes from the conditional means of the stability terms to their realized
values; because the conditional variance of each term is proportional to its
conditional mean, this step contributes only a lower-order term.  Since the
regret is measured against the mixed predictions $\bh_t$, the learner's own
draws $g_t$ do not enter the analysis.

We use \cref{thm:hybrid} for $\cG=\cH$ in
\cref{sec:reduction,sec:batch,sec:det} and for $\cG=\Hpm$ in
\cref{sec:swap-poly}.  The algorithm does not depend on $\zeta$, so for each $\zeta$ a single run
satisfies the bound with probability at least $1-\zeta$.  For a finite class, exponential
weights over $\cG$ satisfy the bound of \cref{thm:hybrid} pathwise with
$d_{\cG}$ replaced by $\log|\cG|$, but in time linear in $|\cG|$.

\subsection{Uniform convergence and a martingale lemma}
\label{sec:prelims-tools}

We use three standard facts about a class $\cG$ of VC dimension $d_{\cG}\ge1$.
First, for every distribution $\cD_X$ and $\rho\in(0,1/2)$, there is a
\emph{$\rho$-cover} $\cG_\rho\subseteq\cG$ such that every $g\in\cG$ has some
$g'\in\cG_\rho$ with $\cD_X(g\ne g')\le\rho$, and
$\log|\cG_\rho|\le Cd_{\cG}\log(e/\rho)$ \citep{haussler1995sphere}.  Second,
for a class of sets whose growth function on $n$ points is at most
$(en)^{k}$, the empirical frequencies of $n$ independent draws are uniformly
within $C\sqrt{n(k\log(en)+\log(1/\zeta))}$ of $n$ times the probabilities,
with probability at least $1-\zeta$ \citep{vapnik1971uniform,sauer1972density}.
Third, for a class of sets of VC dimension $d'$, a sample of size
$O((d'\log(1/\rho)+\log(1/\zeta))/\rho)$ is, with probability at least
$1-\zeta$, a \emph{$\rho$-net}: every set in the class of mass greater than
$\rho$ contains a sample point \citep{haussler1987epsilon}.  We apply this to
the disagreement sets $\{x:g(x)\ne g'(x)\}$ with $g,g'\in\cG$, whose VC
dimension is $O(d_{\cG})$, and to the open order intervals of
\cref{rem:order}, whose VC dimension is at most two.

A sequence of maps $z_t:\cX\times[0,1]\to\R$ is \emph{predictable} with
respect to independent examples $(X_1,Y_1),(X_2,Y_2),\ldots$ if each $z_t$ is
determined by $(X_s,Y_s)_{s<t}$ and by randomness independent of
$(X_s,Y_s)_{s\ge t}$.  For predictable maps we write
$\E_\cD[h(X)z_t(X,Y)]$ for the expectation over a fresh example with $z_t$
held fixed.

\begin{lemma}[Predictable multipliers]
\label{lem:martingale}
Let $(X_t,Y_t)_{t\le T}$ be independent draws from $\cD$ and let
$\zeta\in(0,1)$.
\begin{enumerate}[label=(\alph*),leftmargin=2em]
\item Let $\cI$ be a finite index set and, for each $i\in\cI$, let
  $(z_{t,i})_{t\le T}$ be predictable maps into $[-1,1]$.  With probability
  at least $1-\zeta$,
  \[
    \sup_{h\in\cH,\ i\in\cI}
    \Bigl|\sum_{t=1}^T\bigl(h(X_t)z_{t,i}(X_t,Y_t)
       -\E_\cD[h(X)z_{t,i}(X,Y)]\bigr)\Bigr|
    \le C\sqrt{T\bigl(d\log(eT)+\log(2|\cI|/\zeta)\bigr)}.
  \]
\item Let $(z_{t,v})_{t\le T,v\in\Lambda}$ be predictable maps with
  $\sum_v|z_{t,v}(x,y)|\le1$ for all $t,x,y$.  With probability at least
  $1-\zeta$,
  \begin{multline*}
    \sup_{(h_v)_{v\in\Lambda}\in(\Hpm)^{\Lambda}}
    \Bigl|\sum_{t=1}^T\sum_{v\in\Lambda}\bigl(h_v(X_t)z_{t,v}(X_t,Y_t)
       -\E_\cD[h_v(X)z_{t,v}(X,Y)]\bigr)\Bigr|\\
    \le C\sqrt{T\bigl(md\log(emT)+\log(1/\zeta)\bigr)}.
  \end{multline*}
\end{enumerate}
\end{lemma}

The proof (\cref{app:batch}) applies Azuma's inequality to the auditors in a
fixed cover of $\cH$, which is used only in the analysis, and bounds the
difference between an auditor and its representative by ordinary uniform
convergence for disagreement sets.  This avoids applying a growth-function
bound to the data-dependent multipliers.

\subsection{Remarks on the assumptions}
\label{sec:prelims-remarks}

\begin{remark}[Order on inputs]
\label{rem:order}
Our deterministic predictor (\cref{sec:det}) uses a computable total order
on $\cX$ whose order intervals are measurable, together with computable
equality tests; the lexicographic order on $\R^n$ suffices.  The order is
used only to partition the input space into cells of small
probability mass.  The randomized guarantees do not use it.
\end{remark}

\begin{remark}[Values of outcomes]
Outcomes take values in $[0,1]$.  The lower bounds of
\cref{thm:swap,thm:erm-lb,thm:audit} use binary outcomes, which makes them stronger.
\end{remark}

\section{Multicalibration}
\label{sec:mc}

This section gives an algorithm that learns a multicalibrated randomized
predictor with a weighted ERM oracle, proves its sample complexity, and
shows that this sample complexity is optimal for such algorithms.
\Cref{sec:signs} introduces signed tests, the root rule for choosing
forecasts, and the cost of measuring residuals at the forecast mean, which together let $O(K)$
tests replace the $2^{K+1}$ sign patterns.  \Cref{sec:reduction} reduces online
multicalibration with independent inputs to hybrid online learning
(\cref{thm:online}).  \Cref{sec:batch} converts this into a batch guarantee
(\cref{thm:batch,cor:batch-rates}), which proves the randomized guarantee of
\cref{thm:main}(a).  \Cref{sec:erm-lb} shows that this sample complexity is
optimal, up to logarithmic factors, among learners that access the class
only through polynomially many ERM calls, when $d$ is at least about
$1/\varepsilon$ (\cref{thm:erm-lb}); this gives \cref{thm:main}(b)
(\cref{cor:erm-lb}).  \Cref{sec:det} later makes the output deterministic.
Proofs are in \cref{app:mc,app:erm-lb}.

\subsection{Signed tests and the root rule}
\label{sec:signs}

This subsection develops the three ingredients that let a few signed tests
replace the $2^{K+1}$ sign patterns of \cref{lem:signed} without amplifying
the discretization error: test families with small distortion
(\cref{sec:signs-family}), the root rule, which chooses each forecast so
that its coefficient-weighted residual, measured at the forecast's mean, is
nonpositive for every outcome (\cref{sec:signs-root}), and virtual residuals
(\cref{sec:signs-round}).

\subsubsection{Families of signed tests}
\label{sec:signs-family}

\begin{definition}[Distortion]
\label{def:distortion}
A finite family $\cS\subseteq[-1,1]^{\Lambda}$ has \emph{distortion}
$\kappa\ge1$ if
\[
  \|b\|_1\le\kappa\max_{\sigma\in\cS}\langle\sigma,b\rangle
  \qquad\text{for every }b\in\R^{\Lambda}.
\]
\end{definition}

The definition, and the construction of $\Shad$ below, apply verbatim to
any finite index set of size $m$ in place of $\Lambda$; \cref{sec:audit}
uses them for the set of values of a predictor.  By \cref{lem:signed}, the
cube $\Scube=\{-1,+1\}^{\Lambda}$ has
$|\Scube|=2^m$ and distortion $1$.  Let $\nu=2^{\lceil\log_2m\rceil}$, so
$m\le\nu<2m$, and let $W_\nu\in\{-1,+1\}^{\nu\times\nu}$ be the Sylvester
Hadamard matrix \citep{sylvester1867thoughts}, defined by $W_1=(1)$ and
$W_{2k}=\bigl(\begin{smallmatrix}W_k&W_k\\W_k&-W_k\end{smallmatrix}\bigr)$.
Its columns are orthogonal.  Let $M\in\{-1,+1\}^{\nu\times m}$ consist of its
first $m$ columns, so that $M^{\top}M=\nu I_m$, and let $\Shad$ be the
list, with repetitions, of the $\nu$ rows of $M$ and their negations, so
that $|\Shad|=2\nu<4m$.

\begin{lemma}[Hadamard tests]
\label{lem:hadamard}
The family $\Shad$ has distortion $\sqrt m$.
\end{lemma}

\begin{proof}
Column orthogonality gives
$\|b\|_2^2=\nu^{-1}\|Mb\|_2^2\le\max_{i\le\nu}|(Mb)_i|^2$.  Hence
$\|b\|_1\le\sqrt m\|b\|_2\le\sqrt m\max_i|(Mb)_i|$, and the negated rows turn
the maximum absolute value into a maximum of $\langle\sigma,b\rangle$ over
$\sigma\in\Shad$.
\end{proof}

No fixed family of polynomially many tests does substantially better.

\begin{proposition}[Distortion of small families]
\label{prop:distortion-lb}
If a family $\cS\subseteq[-1,1]^{\Lambda}$ with $N$ elements satisfies
$\|b\|_1\le\kappa\max_{\sigma\in\cS}|\langle\sigma,b\rangle|$ for all
$b\in\R^{\Lambda}$, then $\kappa\ge\sqrt{m/(2\log(4N))}$.  In particular,
every family with $N$ elements and distortion $\kappa$ satisfies
$\kappa\ge\sqrt{m/(2\log(4N))}$.
\end{proposition}

\begin{proof}
Let $b$ have independent uniform $\pm1$ coordinates, so $\|b\|_1=m$.  For a
fixed $\sigma\in[-1,1]^{\Lambda}$, Hoeffding's inequality
\citep{hoeffding1963probability} gives
$\Pr\{|\langle\sigma,b\rangle|>u\}\le2e^{-u^2/(2m)}$.  For
$u=\sqrt{2m\log(4N)}$, a union bound over $\cS$ shows that with probability at
least $1/2$ every $|\langle\sigma,b\rangle|$ is at most $u$.  For such a $b$,
$m\le\kappa u$, which is the claim.
\end{proof}

\Cref{prop:distortion-lb} limits only analyses that bound the $\ell_1$ norm
through a fixed family of tests; it is not a lower bound for multicalibration
algorithms.  A lower bound for all learners that access the class through
polynomially many ERM calls is \cref{thm:erm-lb}.

\subsubsection{The root rule}
\label{sec:signs-root}

The root rule chooses, given a coefficient for each grid value, a
distribution $P$ with adjacent support and mean $\mu$ such that
$(\mu-y)\sum_vP(v)a_v\le0$ for every outcome $y\in[0,1]$
(\cref{lem:root}); \cref{sec:det} uses a version restricted to a set of
consecutive grid values (\cref{lem:root-sub}).  Rules of this
root-finding kind go back to the randomized calibration algorithm of
\citet{foster1998asymptotic}.

\begin{lemma}[Root rule]
\label{lem:root}
Let $a\in[-1,1]^{\Lambda}$.  One can compute, in $O(m)$ time, a distribution
$P=\Root(a)\in\Delta(\Lambda)$ with adjacent support such that, with
$\mu=\sum_vP(v)\,v$,
\[
  (\mu-y)\sum_{v\in\Lambda}P(v)\,a_v\le0
  \qquad\text{for every }y\in[0,1].
\]
\end{lemma}

\begin{proof}
If $a_0\ge0$, let $P$ be the point mass at $0$; then
$(0-y)a_0\le0$ since $y\ge0$.  Otherwise, if $a_1\le0$, let
$P$ be the point mass at $1$; then $(1-y)a_1\le0$.  Otherwise
$a_0<0<a_1$, so there are adjacent grid values $v,v+1/K$ in
$\Lambda$ with $a_v\le0<a_{v+1/K}$.  Choose $\lambda\in(0,1]$ with
$\lambda a_v+(1-\lambda)a_{v+1/K}=0$, and put mass $\lambda$ on $v$ and
$1-\lambda$ on $v+1/K$; the coefficient of $\mu-y$ is then zero.
\end{proof}

\subsubsection{Virtual residuals}
\label{sec:signs-round}

For a forecast $P$ with mean $\mu$ and an outcome $y$, we call
$P(v)(\mu-y)$ the \emph{virtual residual} at $v$ and $P(v)(v-y)$ the
\emph{actual residual}; an auditor with value $h\in\{-1,+1\}$ at the input
is scored by $h$ times the residual.  The online learners of our forecaster
see virtual residuals, because the root rule controls the sign of the
coefficient-weighted virtual residual, not of the actual one.  With actual
residuals, an uncontrolled term of up to $1/K$ per round would be multiplied by the distortion
$\sqrt m$, giving about $1/\sqrt K$; keeping this below $\varepsilon$
forces $K\approx\varepsilon^{-2}$, and hence $\varepsilon^{-4}$ samples
instead of $\varepsilon^{-3}$.  With virtual
residuals the discrepancy is paid once, outside the distortion, and the next
lemma bounds it by $1/(2K)$.

\begin{lemma}[Adjacent support]
\label{lem:adjacent}
If $P\in\Delta(\Lambda)$ has adjacent support and mean $\mu$, then
$\sum_vP(v)|v-\mu|\le1/(2K)$.  Consequently, for every $y\in[0,1]$ and
$h\in\{-1,+1\}$,
\[
  \sum_{v\in\Lambda}\bigl|hP(v)(v-y)-hP(v)(\mu-y)\bigr|\le\frac1{2K}.
\]
\end{lemma}

\begin{proof}
For masses $\lambda$ and $1-\lambda$ on $v$ and $v+1/K$, the left side equals
$\lambda(1-\lambda)/K+(1-\lambda)\lambda/K=2\lambda(1-\lambda)/K\le1/(2K)$.
The second claim follows since $|h|=1$.
\end{proof}

\Cref{lem:adjacent} is applied to the cumulative online errors and, inside
an expectation, to population errors; because each auditor in either error
has absolute value one, it applies equally to ordinary and swap errors.

\subsection{Reduction to hybrid online learning}
\label{sec:reduction}

This subsection reduces online multicalibration with independent inputs to
hybrid online learning over $\cH$.  The reduction follows the structure of
\citet{noarov2026optimal}: exponential weights over signed tests, and a
root rule that makes the forecast's coefficient-weighted virtual residual
nonpositive.  In
place of their enumeration of the auditors, each test has its own hybrid
online learner.  \Cref{sec:reduction-alg} presents the algorithm,
\cref{sec:reduction-online} proves its online guarantee
(\cref{thm:online}), and \cref{sec:reduction-inst} instantiates it with the
Hadamard and cube families (\cref{cor:online-rates}).  The proof is in
\cref{app:reduction}.

\subsubsection{The algorithm}
\label{sec:reduction-alg}

\Cref{alg:signed} takes a family $\cS$ of signed tests and runs one hybrid
online learner $\cA_\sigma$ over $\cH$ for each test $\sigma\in\cS$, with
distributions $\pi_{t,\sigma}$ and mixed predictions $\bh_{t,\sigma}$.
Exponential weights $w_t\in\Delta(\cS)$ aggregate the tests into the
coefficient vector
\[
  A_t(x)=\sum_{\sigma\in\cS}w_t(\sigma)\,\bh_{t,\sigma}(x)\,\sigma
  \ \in[-1,1]^{\Lambda}.
\]
The algorithm estimates $A_t$ by sampling $L$ tests from $w_t$ and one
hypothesis from each sampled test's learner, and forecasts with the root rule
of \cref{lem:root}.  After the outcome is revealed, each learner receives the
virtual residual weighted by its test, and exponential weights are updated
using one fresh hypothesis from each learner.  The auxiliary hypotheses are
drawn before the round's example, so that the forecast $Q_t$ is a fixed map
of the input given the past; they are retained, since they define the
forecast.

\begin{algorithm}[t]
\caption{$\textsc{SignedTests}(\cS,K,T,\eta,L)$: signed-test forecaster}
\label{alg:signed}
\begin{algorithmic}[1]
\Require family $\cS\subseteq\{-1,+1\}^{\Lambda}$ of tests; grid size $K$;
  number of rounds $T$; learning rate $\eta$; number $L$ of auxiliary draws
\State for each $\sigma\in\cS$, start a Gaussian follow-the-perturbed-leader
  learner $\cA_\sigma$ over $\cH$ (\cref{thm:hybrid}) with horizon $T$; set
  $w_1(\sigma)=1/|\cS|$
\For{$t=1,\ldots,T$}
  \For{$i=1,\ldots,L$}
    \State draw $\sigma_t^{(i)}\sim w_t$ and
      $h_t^{(i)}\sim\pi_{t,\sigma_t^{(i)}}$
      \Comment{one ERM call}
  \EndFor
  \State define $\widehat A_t(x)=\frac1L\sum_{i=1}^Lh_t^{(i)}(x)\,
    \sigma_t^{(i)}$, \ $Q_t(x)=\Root\bigl(\widehat A_t(x)\bigr)$, \ $\mu_t(x)=\sum_vQ_t(x)(v)\,v$
  \State receive $X_t$, forecast $Q_t(X_t)$, and observe $Y_t$
  \For{each $\sigma\in\cS$}
    \State $Z_{t,\sigma}=\bigl(\mu_t(X_t)-Y_t\bigr)\sum_{v}
      \sigma_vQ_t(X_t)(v)$
    \State draw $h_{t,\sigma}\sim\pi_{t,\sigma}$ with fresh randomness and
      set $\hat s_{t,\sigma}=h_{t,\sigma}(X_t)Z_{t,\sigma}$
      \Comment{one ERM call}
    \State give $(X_t,Z_{t,\sigma})$ to $\cA_\sigma$
  \EndFor
  \State $w_{t+1}(\sigma)\propto w_t(\sigma)\exp(\eta\,\hat s_{t,\sigma})$
    for each $\sigma\in\cS$
\EndFor
\State \Return the retained draws $(\sigma_t^{(i)},h_t^{(i)})_{t\le T,i\le
  L}$, which define the maps $Q_1,\ldots,Q_T$
\end{algorithmic}
\end{algorithm}

Each feedback value satisfies $|Z_{t,\sigma}|\le1$, since
$|\mu_t-Y_t|\le1$ and $\sum_v|Q_t(X_t)(v)|=1$.  The feedback depends on the
auxiliary draws but not on the learner's round-$t$ draw $h_{t,\sigma}$, as
\cref{def:hybrid} requires.

\subsubsection{Online guarantee}
\label{sec:reduction-online}

The online guarantee bounds the cumulative error by the grid term plus the
distortion times the regret of the learners, of exponential weights, and of
the estimation step.

\begin{theorem}[Online guarantee]
\label{thm:online}
Let a family $\cS\subseteq\{-1,+1\}^{\Lambda}$ have $N$ elements and
distortion $\kappa$, and let $\delta\in(0,1]$.  Suppose the inputs $X_t$ are
independent draws from $\cD_X$ and each outcome $Y_t\in[0,1]$ is chosen as a
function of the past, of $X_t$, and of the auxiliary draws of round $t$, but
not of the draws $h_{t,\sigma}$.  Let $\alpha\in(0,1]$ be a target accuracy for the estimates $\widehat A_t$
of the coefficients $A_t$, and run
\cref{alg:signed} with $\eta=\min\{1,\sqrt{\log N/T}\}$ and
$L\ge2\alpha^{-2}\log(16mT/\delta)$.  Then, with probability at least
$1-\delta$,
\[
  \MC_T\le\frac{T}{2K}
  +\kappa\Bigl[\Reg\Bigl(T,d,\frac{\delta}{4N}\Bigr)
  +C\sqrt{T\log\frac{4N}{\delta}}+T\alpha\Bigr].
\]
The algorithm makes $N+L$ weighted ERM calls and $O(m(N+L))$ other
operations per round.
\end{theorem}

\paragraph{Proof sketch.}
Let $s_{t,\sigma}=\bh_{t,\sigma}(X_t)Z_{t,\sigma}$ be the \emph{reward} of
test $\sigma$ on round $t$.  The proof has three steps.  First, the
$w_t$-weighted reward equals
$(\mu_t(X_t)-Y_t)\sum_vQ_t(X_t)(v)A_t(X_t)_v$; \cref{lem:root} makes this quantity nonpositive with $\widehat A_t$ in place
of $A_t$, and the substitution costs at most $\alpha$.
By the regret bound of exponential weights and Azuma's inequality, no single
test accumulates reward larger than $T\alpha+C\sqrt{T\log(N/\delta)}$.
Second, by \cref{thm:hybrid}, for each test $\sigma$, no auditor $h$
correlates with the feedback of $\sigma$ more than the learner of $\sigma$
does, up to the regret $\Reg$.  Third, the correlation of $h$ with the
feedback of $\sigma$ is $\langle\sigma,b_h\rangle$, where
$b_h=\bigl(\sum_th(X_t)Q_t(X_t)(v)(\mu_t(X_t)-Y_t)\bigr)_{v\in\Lambda}$ is
the vector of cumulative virtual residuals weighted by $h$; the distortion
(\cref{def:distortion}) converts $\max_\sigma\langle\sigma,b_h\rangle$ into
$\|b_h\|_1$, and \cref{lem:adjacent} converts virtual into actual residuals at a
cost of $T/(2K)$.  The full proof is in \cref{app:reduction}.

\subsubsection{Instantiations}
\label{sec:reduction-inst}

With Hadamard tests the online error is $\wtO(T^{2/3})$ for fixed $d$ with
polynomially many ERM calls; with all sign patterns the dimension term
becomes additive, at exponential cost.

\begin{corollary}[Online rates]
\label{cor:online-rates}
Under the conditions of \cref{thm:online}, with probability at least
$1-\delta$:
\begin{enumerate}[label=(\roman*),leftmargin=2em]
\item with $\cS=\Shad$ and $\alpha=1/(K\sqrt m)$, using
  $O(K^3\log(KT/\delta))$ ERM calls per round,
  \[
    \MC_T=\wtO\Bigl(\frac TK+\sqrt{TK\bigl(d+\log(1/\delta)\bigr)}\Bigr);
  \]
  choosing $K\asymp(T/(d+\log(1/\delta)))^{1/3}$ gives
  $\MC_T=\wtO\bigl((d+\log(1/\delta))^{1/3}T^{2/3}\bigr)$ whenever
  $T\ge d+\log(1/\delta)$;
\item with $\cS=\Scube$ and $\alpha=1/K$, using
  $2^{K+1}+O(K^2\log(KT/\delta))$ ERM calls per round,
  \[
    \MC_T=\wtO\Bigl(\frac TK+\sqrt{T\bigl(d+K+\log(1/\delta)\bigr)}\Bigr);
  \]
  choosing
  $K=\lceil T^{1/3}\rceil$ gives
  $\MC_T=\wtO\bigl(T^{2/3}+\sqrt{T(d+\log(1/\delta))}\bigr)$.
\end{enumerate}
\end{corollary}

\begin{proof}
For (i), \cref{lem:hadamard} gives $\kappa=\sqrt m$ and $N<4m$.  The
bracket in \cref{thm:online} is then at most
$\wtO(\sqrt{T(d+\log(1/\delta))})+T/(K\sqrt m)$, and multiplying by
$\sqrt m$ gives the claim; also $L=O(K^2m\log(mT/\delta))$.

For (ii), we have $\kappa=1$ and $\log N=m\log2$, so
\[
  \Reg\Bigl(T,d,\frac{\delta}{4N}\Bigr)
  =\wtO\Bigl(\sqrt{T\bigl(d+K+\log(1/\delta)\bigr)}\Bigr).
\]
\end{proof}

\subsection{From online to batch multicalibration}
\label{sec:batch}

This subsection shows that, on independent examples, the average of the
forecasts of \cref{alg:signed} is a randomized predictor with small population
multicalibration error.  \Cref{thm:batch} gives the general guarantee, and \cref{cor:batch-rates} gives the
sample complexities for the Hadamard and cube families.  The main tool is
\cref{lem:martingale}; the proofs are in \cref{app:batch}.

Given the output of \cref{alg:signed}, define the averaged randomized predictor
\[
  \Qbar_x(v)=\frac1T\sum_{t=1}^TQ_t(x)(v),
\]
which is computed from the retained draws.

\begin{theorem}[Batch guarantee]
\label{thm:batch}
Let $\cS\subseteq\{-1,+1\}^{\Lambda}$
have $N$ elements and distortion $\kappa$, and let $\delta,\alpha\in(0,1]$.
Let $(X_t,Y_t)_{t\le T}$ be independent draws from $\cD$.  Run \cref{alg:signed} with $\eta=\min\{1,\sqrt{\log N/T}\}$ and
$L\ge2\alpha^{-2}\log(16mT/\delta)$.  Then, with probability at least
$1-\delta$,
\[
  \MC(\Qbar)\le\frac1{2K}+\kappa\Bigl[\frac1T\Reg\Bigl(T,d,
  \frac{\delta}{8N}\Bigr)
  +C\sqrt{\frac{d\log(eT)+\log(8N/\delta)}{T}}+\alpha\Bigr].
\]
\end{theorem}

\paragraph{Proof sketch.}
The argument of \cref{thm:online} bounds, for every auditor and test, the
cumulative correlation of the auditor with the test's feedback.  The feedback of test $\sigma$ is the
value at $(X_t,Y_t)$ of the predictable map
$z_{t,\sigma}(x,y)=(\mu_t(x)-y)\sum_v\sigma_vQ_t(x)(v)$, and
\cref{lem:martingale}(a) replaces each empirical correlation by its
expectation.  For each auditor $h$, the vector
$\bigl(\frac1T\sum_t\E[h(X)Q_t(X)(v)(\mu_t(X)-Y)]\bigr)_v$ of expected
virtual residuals, averaged over the $T$ forecasts, is then converted to the
$\ell_1$ norm by the distortion, and \cref{lem:adjacent}, applied to each
$Q_t$, gives the term $1/(2K)$.

\begin{corollary}[Sample complexity of randomized predictors]
\label{cor:batch-rates}
Under the conditions of \cref{thm:batch}, let $\varepsilon\in(0,1)$.
\begin{enumerate}[label=(\roman*),leftmargin=2em]
\item With $\cS=\Shad$, $K=\lceil C/\varepsilon\rceil$, $\alpha=\varepsilon/(C\sqrt m)$,
  and
  \[
    T=\wtO\Bigl(\frac{d+\log(1/\delta)}{\varepsilon^{3}}\Bigr),
  \]
  we have $\MC(\Qbar)\le\varepsilon$ with probability at least $1-\delta$.
  The algorithm makes
  $\wtO((d+\log(1/\delta))\varepsilon^{-6})$ ERM calls in total.
\item With $\cS=\Scube$, $K=\lceil C/\varepsilon\rceil$,
  $\alpha=\varepsilon/C$, and
  \[
    T=\wtO\Bigl(\frac1{\varepsilon^3}+\frac{d+\log(1/\delta)}{\varepsilon^2}\Bigr),
  \]
  we have $\MC(\Qbar)\le\varepsilon$ with probability at least $1-\delta$,
    with $2^{O(1/\varepsilon)}$ ERM calls per round, and hence
  $2^{O(1/\varepsilon)}(d+\log(1/\delta))$ in total.
\end{enumerate}
\end{corollary}

\begin{proof}
For (i), the grid and estimation terms are each at most $\varepsilon/4$ for a
suitable $C$.  Since $N<4m$ and $\kappa=\sqrt m$, the remaining terms are at
most $\varepsilon/2$ once $T$ exceeds $m(d+\log(1/\delta))/\varepsilon^2$ by a
sufficiently large polylogarithmic factor.  The number of ERM calls is
$T(N+L)$ with $N=O(1/\varepsilon)$ and $L=\wtO(m/\varepsilon^2)=
\wtO(\varepsilon^{-3})$.  For (ii), $\kappa=1$ and $\log N=O(1/\varepsilon)$,
so the remaining terms are at most $\varepsilon/2$ once $T$ exceeds
$(d+1/\varepsilon+\log(1/\delta))/\varepsilon^2$ by a sufficiently large
polylogarithmic factor.
\end{proof}

\paragraph{Evaluating the averaged predictor.}
Evaluating $\Qbar_x$ requires recomputing $Q_t(x)$ for all $t$, which costs
$TL$ hypothesis evaluations and no oracle calls.  Drawing a single prediction
from $\Qbar_x$ requires only one round: choose $t$ uniformly from $[T]$ and
draw from $Q_t(x)$, at a cost of $L$ hypothesis evaluations.

\paragraph{Comparison with enumeration.}
Part (ii) attains the additive rate $\wtO(\varepsilon^{-3}+d\varepsilon^{-2})$
using ERM over $\cH$ but enumerating the $2^{K+1}$ sign patterns.
\citet{noarov2026optimal} attain the same additive rate with a deterministic
predictor by enumerating a cover of $\cH$ instead, and they evaluate the maximum over sign patterns by a product formula rather
than enumerating them.  When $\widetilde\Omega(1/\varepsilon)\le d\le\mathrm{poly}(1/\varepsilon)$,
\cref{thm:erm-lb} shows that the additive rate requires either access to
$\cH$ beyond the oracle, such as enumerating a cover, or superpolynomially
many ERM calls, such as enumerating sign patterns.

\subsection{Optimality among black-box ERM learners}
\label{sec:erm-lb}

This subsection shows that no learner that accesses the class only through
polynomially many weighted ERM calls can improve the sample complexity of
\cref{cor:batch-rates}(i) by more than polylogarithmic factors when
$\widetilde\Omega(1/\varepsilon)\le d\le\mathrm{poly}(1/\varepsilon)$.  \Cref{thm:erm-lb} states the lower
bound for an explicit family of hard instances, and \cref{cor:erm-lb}
restates it in terms of the VC dimension, which is \cref{thm:main}(b).  The
proof of \cref{thm:erm-lb} is in \cref{app:erm-lb}.

\begin{definition}[Black-box ERM learner]
\label{def:black-box}
Let $q\ge1$.  A \emph{$q$-query black-box ERM learner} receives $\varepsilon$,
$n$ independent labeled samples from $\cD$, and a weighted ERM oracle for $\cH$
(\cref{def:erm}), which it may call adaptively at most $q$ times.  It has no
other access to $\cH$, but it may use unlimited computation and randomness.
After its last oracle call it outputs a finitely supported randomized
predictor.  The oracle may break ties among maximizers by any rule, and the
learner does not know the rule.
\end{definition}

\Cref{alg:signed} with Hadamard tests, as in \cref{cor:batch-rates}(i), is a
$q$-query black-box ERM learner with
$q=\wtO(d\varepsilon^{-6})$; it evaluates the hypotheses that the oracle
returns, which \cref{def:erm} allows.

\paragraph{The construction.}
The inputs are $\cX=[J]\times[D]$ with the uniform distribution, arranged in
$J$ blocks of $D$ inputs, and the mean outcome at an input of block $j$ is a
block center $c_j$ plus or minus a gap $\beta=8\varepsilon$, with a hidden
sign $\tau_x\in\{-1,+1\}$ at each input $x$.  The class has three parts.
First, it contains all \emph{anchors}, the sign functions that are constant
on each block; they force a predictor with small multicalibration error to
put most of its mass close to the center of each block.  Second, it contains
$16$ \emph{planted auditors} $h^{(\ell)}(j,i)=\tau_{(j,i)}\theta^{(\ell)}_j$,
for $\ell\in[16]$, with random block signs $\theta^{(\ell)}\in\{-1,+1\}^J$.  A predictor whose values are near the block centers but
that does not know the hidden signs has error of order $\beta$ against a
random planted auditor.  With $n$ samples, each input is observed about
$n/|\cX|$ times, so the hidden signs remain nearly unknown unless
$n\gtrsim|\cX|/\beta^2$.  Third, the class contains a random \emph{code} of
$e^{\Theta(d)}$ sign vectors.  For every weight vector that a learner could
choose without knowing the hidden signs, the code contains, with high
probability, a hypothesis at least as good as every planted auditor, so the
oracle never reveals a planted auditor.  The block centers are spaced $\Theta(\varepsilon\log(1/\varepsilon))$
apart, so there are $J=\Theta(1/(\varepsilon\log(1/\varepsilon)))$ blocks,
and the class has VC dimension $d=\Theta(D\log|\cX|)$.  Hence
$|\cX|/\beta^2=JD/\beta^2$ is of order $d\varepsilon^{-3}$ up to logarithmic
factors.

\paragraph{Parameters.}
For $\varepsilon>0$ and a query budget $q\ge1$, let
\[
  \beta=8\varepsilon,\qquad
  \kappa_{\mathrm{anc}}=2\lceil\log_2(1/\beta)\rceil+4,\qquad
  \Delta_c=128\kappa_{\mathrm{anc}}\beta,\qquad
  J=\lfloor1/(4\Delta_c)\rfloor.
\]
Here $\beta$ is the gap, $\kappa_{\mathrm{anc}}$ is the factor that the
anchors lose in controlling the distance of a predictor from the block
centers, $\Delta_c$ is the spacing of the block centers, and $J$ is the
number of blocks.  Choose an integer $D\ge\max\{128,J,C\log(100q)\}$, the
number of inputs per block, let $\cX=[J]\times[D]$, and let
\[
  d_0=\bigl\lceil16D\log(8|\cX|)\bigr\rceil,\qquad
  \mathsf t=\log(1000|\cX|),\qquad
  \mathsf A=\kappa_{\mathrm{anc}}^2+\log(1000q).
\]
Here $d_0$ is the number of coordinates on which the code realizes every
sign pattern, which bounds the VC dimension from below, and $\mathsf t$ and
$\mathsf A$ are logarithmic factors from union bounds over the inputs and
over the queries.  We assume
\begin{equation}
  \label{eq:erm-lb-cond}
  d_0\le|\cX|/2,\qquad |\cX|\ge C\kappa_{\mathrm{anc}}^2,\qquad
  \beta\,\mathsf t\sqrt{\mathsf A}\le c,
\end{equation}
where $C$ is a sufficiently large and $c$ a sufficiently small absolute
constant; in this subsection and in \cref{app:erm-lb}, $c$ and $C$ refer to
these constants.  The binding requirement is $D\ge J$, which forces the
VC dimension to be $\widetilde\Omega(1/\varepsilon)$ (\cref{cor:erm-lb}
assumes $d\ge\varepsilon^{-1}\log(1/\varepsilon)/c_k$);
the conditions in~\eqref{eq:erm-lb-cond} hold for all sufficiently small
$\varepsilon$ when $D$ and $q$ are polynomial in $1/\varepsilon$.

\begin{theorem}[Black-box ERM lower bound]
\label{thm:erm-lb}
Under the conditions above, for every $q$-query black-box ERM learner that
uses
\[
  n\le\frac{c\,|\cX|}{\beta^2\mathsf A\,\mathsf t}
\]
labeled samples, there exist a class $\cH\subseteq\{-1,+1\}^{\cX}$ with
$d_0\le\VC(\cH)\le24d_0$, a weighted ERM oracle for $\cH$, and a distribution
over $\cX\times\{0,1\}$ with uniform input marginal, on which the output $Q$
of the learner satisfies $\Pr\{\MC(Q)>\varepsilon\}\ge2/3$.  This holds even
if the learner knows the input distribution, the bound $24d_0$ on the VC
dimension, and every part of the instance except the hidden signs and the
planted block signs.
\end{theorem}

\paragraph{Proof sketch.}
We draw the hidden signs and the planted block signs at random and show that
the learner fails with probability greater than $2/3$ on average; some fixed
instance then has the same property.  The proof has four steps.  First, a
dyadic decomposition of the sign matrix $(\sgn(v-c_j))_{j,v}$ shows that the
anchors control the total displacement of a predictor from the block
centers, with a loss of only $\kappa_{\mathrm{anc}}=O(\log(1/\varepsilon))$
(\cref{lem:anchor}).  Second, a random code of size $e^{8d_0+8D}$ contains,
for every set of at most $d_0$ coordinates and every weight vector whose
remaining mass is not concentrated on a few inputs, a hypothesis that
matches the signs of the weights on those coordinates and correlates well
with the rest (\cref{lem:code}).  Third, given the sample, the hidden signs are nearly
uniform, and every query that does not depend on them is answered by the
code with high probability; comparing the learner with a run against an
oracle that ignores the planted auditors shows that, with high probability,
the two runs are identical (\cref{lem:coupling}).  Fourth, the output of
that run does not depend on the hidden signs or the planted block signs, and
either it is far from the block centers, so that an anchor detects it, or a
random planted auditor detects error at least $2\varepsilon>\varepsilon$ with constant
probability, by the Paley--Zygmund inequality (\cref{lem:plant-error}).

In terms of the VC dimension, \cref{thm:erm-lb} gives the following
statement, which is \cref{thm:main}(b).

\begin{corollary}[Black-box lower bound in terms of the VC dimension]
\label{cor:erm-lb}
For every $k\ge1$ there is $c_k>0$ such that, for every
$\varepsilon\in(0,c_k)$ and every $d$ with
$\varepsilon^{-1}\log(1/\varepsilon)/c_k\le d\le\varepsilon^{-k}$, the following holds.  If a $\lfloor(d/\varepsilon)^k\rfloor$-query
black-box ERM learner, for every class of VC dimension at most $d$ and every
distribution, outputs a randomized predictor with multicalibration error at
most $\varepsilon$ with probability at least $2/3$, then it uses more than
$c_k\,d/(\varepsilon^{3}\log^{5}(1/\varepsilon))$ samples.
\end{corollary}

\begin{proof}
Let $C'$ and $c'$ denote absolute constants, and apply \cref{thm:erm-lb}
with $\varepsilon$, $q=\lfloor(d/\varepsilon)^k\rfloor$, and
$D=\lfloor d/(C_D(k+3)\log(1/\varepsilon))\rfloor$, where $C_D$ is an absolute
constant chosen below.

\emph{Bounds on the parameters.}  Since $d\le\varepsilon^{-k}$, we have
$\log(1+q)\le k\log(d/\varepsilon)+1\le(k+1)^2\log(1/\varepsilon)$, and
$\kappa_{\mathrm{anc}}\le C'\log(1/\varepsilon)$.  The lower bound on $d$
gives $D\ge2/\varepsilon$.  The number of blocks satisfies
$J\le1/(4\Delta_c)\le1/\varepsilon\le D$ and
$J\ge1/(8\Delta_c)\ge c'/(\varepsilon\log(1/\varepsilon))$.  Hence
$|\cX|=JD\le\varepsilon^{-k-1}$, so $\mathsf t\le C'(k+2)\log(1/\varepsilon)$
and $\mathsf A\le C'(k+1)^2\log^2(1/\varepsilon)$.

\emph{Conditions of the theorem.}  For $\varepsilon<c_k$, these bounds give
$D\ge J$ and $D\ge\max\{128,\,C\log(100q)\}$.  They also give
\[
  \beta\mathsf t\sqrt{\mathsf A}
  \le C'(k+2)^2\varepsilon\log^2(1/\varepsilon)\le c
  \qquad\text{and}\qquad
  J\ge34\log(8|\cX|),
\]
and the latter gives $d_0\le|\cX|/2$.  Finally,
$|\cX|\ge D\ge2/\varepsilon\ge C\kappa_{\mathrm{anc}}^2$.

\emph{Conclusion.}  For $C_D$ large enough relative to these constants, the class that
\cref{thm:erm-lb} provides has VC dimension at most
$24d_0\le C'D\log(8|\cX|)\le d$, so the learner must succeed on it with
probability at least $2/3$.
\Cref{thm:erm-lb} shows that it fails with probability at least $2/3$ unless
it uses more than
\[
  \frac{c\,|\cX|}{\beta^2\mathsf A\,\mathsf t}
  \ge\frac{c\,c'\,D}{64\varepsilon^3\log(1/\varepsilon)\cdot C'^2(k+2)^3
    \log^3(1/\varepsilon)}
  \ge\frac{c_k\,d}{\varepsilon^3\log^5(1/\varepsilon)}
\]
samples.
\end{proof}

\begin{remark}[Scope of the lower bound]
\label{rem:erm-lb-scope}
The hard class depends on the hidden signs, and a description of the class
would reveal the planted auditors.  \Cref{thm:erm-lb} is therefore a lower
bound for algorithms that access the class only through the oracle, in the
spirit of the oracle lower bound of \citet{hazan2016computational}; it does
not assert that multicalibration is computationally hard for any explicitly
represented class.  It requires $d\ge\widetilde\Omega(1/\varepsilon)$, since
$D\ge J$, and it leaves open the regime $d\ll1/\varepsilon$, where
$d\varepsilon^{-3}$ and $\varepsilon^{-3}+d\varepsilon^{-2}$ differ by a
factor of at most $d$.  The dependence on $q$ is logarithmic.  With
$q=2^{\Theta(1/\varepsilon)}$ queries the bound of \cref{thm:erm-lb}
drops to $\widetilde\Omega(d\varepsilon^{-2})$, in line with
\cref{cor:batch-rates}(ii), which attains the additive rate with
$2^{O(1/\varepsilon)}$ ERM calls per round.
\end{remark}

\section{Swap multicalibration}
\label{sec:swap}

This section compares ordinary and swap multicalibration and determines the
sample complexity of learning swap-multicalibrated predictors.
\Cref{sec:swap-relation} shows that the two errors of a fixed predictor can
differ by a factor equal to the number of values, already for a class of VC
dimension one, but that coarsening turns small ordinary error into small
swap error with a square-root loss.  \Cref{sec:swap-upper} proves the upper bound of \cref{thm:swap}(a)
(\cref{thm:swap-vc}), \cref{sec:swap-poly} proves \cref{thm:swap}(b) with an
oracle-efficient algorithm based on the hybrid learner of \cref{thm:hybrid}
(\cref{thm:swap-poly}), and \cref{sec:swap-lb} proves the lower bound of
\cref{thm:swap}(a) (\cref{thm:swap-lb}) and then assembles the proof of
\cref{thm:swap}.  Long proofs are in
\cref{app:swap-all}.

\subsection{Relation between the two errors}
\label{sec:swap-relation}

For a fixed predictor, the swap error can exceed the ordinary error by a
factor equal to the number of values (\cref{prop:ratio}), but coarsening
converts small ordinary error into swap error with a square-root loss
(\cref{prop:coarsen}).  In the rest of this section and in \cref{sec:audit}, $m$ denotes the number
of values of the predictor at hand; for predictors on the grid $\Lambda$, it
is $K+1$ as before.

\begin{proposition}[Factor-$m$ separation]
\label{prop:ratio}
If $Q$ has $m$ values, then $\MC(Q)\le\SMC(Q)\le m\,\MC(Q)$.  For every
$m\ge2$, there are a deterministic predictor $p$ with $m$ values, a distribution
with binary outcomes, and a class of VC dimension one such that
$\MC(p)=1/(4m)$ and $\SMC(p)=1/4$, and the predictor is calibrated, that is,
its bias with respect to the constant auditor is zero at every value.
\end{proposition}

\begin{proof}
The first inequality exchanges a supremum and a sum.  The second holds
because, for every auditor $h$ and value $v$, the bias $|B_h(Q,v)|$ is at
most $\MC(Q)$.

For the example, take $2m$ equally likely inputs
$a_1,b_1,\ldots,a_m,b_m$.  Let $p_j=1/4+j/(2(m+1))$ and
$p(a_j)=p(b_j)=p_j$, and let the outcome be Bernoulli with
mean $p_j-1/4$ at $a_j$ and $p_j+1/4$ at $b_j$.  Let $h_0\equiv1$, and for
$i\in[m]$ let $h_i$ equal $1$ except $h_i(b_i)=-1$.  This class has VC
dimension one: it shatters each $b_i$ but no auditor is $-1$ at two points.
Direct calculation gives $B_{h_0}(p,p_j)=0$ and
$B_{h_i}(p,p_j)=\mathbf 1\{i=j\}/(4m)$, so $\MC(p)=1/(4m)$ and
$\SMC(p)=1/4$; the constant auditor shows that $p$ is calibrated.
\end{proof}

\begin{proposition}[Coarsening]
\label{prop:coarsen}
Partition $[0,1]$ into $K'$ intervals of width $1/K'$, and let $Q'$ be
obtained from $Q$ by moving each value to the midpoint of its interval,
aggregating probability mass.  Then $\SMC(Q')\le K'\,\MC(Q)+1/(2K')$.  In
particular, if $\MC(Q)\le\varepsilon$, choosing $K'=\lceil\varepsilon^{-1/2}\rceil$
gives $\SMC(Q')=O(\sqrt{\varepsilon})$ while moving each prediction by at
most $1/(2K')=O(\sqrt{\varepsilon})$.
\end{proposition}

\begin{proof}
For an interval $U$ with midpoint $w$, the bias of $Q'$ at $w$ is
\[
  \sum_{v\in U}B_h(Q,v)+\E\Bigl[h(X)\sum_{v\in U}Q_X(v)(w-v)\Bigr].
\]
The first term is
at most $\MC(Q)$ in absolute value, uniformly in $h$, and the second is at
most $\Pr\{V\in U\}/(2K')$ for $V\sim Q_X$.  Take the supremum over $h$ at
each midpoint and sum over the $K'$ intervals.
\end{proof}

Unlike \cref{prop:ratio}, \cref{prop:coarsen} modifies the predictor; it is
the analogue for our error of the bucketing argument of
\citet[Claim~3.4]{gopalan2023swap}.

\subsection{Sample complexity upper bounds}
\label{sec:swap-upper}

Exponential weights run separately at each grid value give a swap guarantee
for finite classes (\cref{thm:swap-finite}), and a cover built from
unlabeled inputs extends it to VC classes (\cref{thm:swap-vc}).

\begin{theorem}[Finite classes]
\label{thm:swap-finite}
Let $\cG\subseteq\{-1,+1\}^{\cX}$ be finite and closed under negation, with
$N=|\cG|$, and let $K\ge1$ and $m=K+1$.  There is a forecaster on the grid
$\Lambda$ with $O(mN)$ operations and hypothesis
evaluations per round such that, for every sequence of inputs and outcomes,
$\SMC_T\le T/(2K)+\sqrt{2Tm\log N}$ with respect to $\cG$.  On $T$
independent examples from $\cD$, the average $\Qbar$ of its forecasts
satisfies, with
probability at least $1-\delta$,
\[
  \SMC(\Qbar;\cG,\cD)\le\frac1{2K}+C\sqrt{\frac{m\log N+\log(1/\delta)}{T}}.
\]
\end{theorem}

The forecaster is the loop of \cref{alg:swap-cover}, run with the given
$\cG$, $K$, and $T$ and with $\eta=\sqrt{2m\log N/T}$: exponential weights at each grid value, combined by the root rule.  The analysis uses the
second-order regret bound of exponential weights and the fact that the
virtual residuals at a round have $\ell_1$ norm at most one
(\cref{app:swap}).

With $K=\Theta(1/\varepsilon)$, \cref{thm:swap-finite} uses
$\wtO(\varepsilon^{-3}\log N+\varepsilon^{-2}\log(1/\delta))$ samples, which
for the mean ECE slightly sharpens the
$\wtO(\varepsilon^{-3}\log^{3/2}N)$ bound that \citet{collina2026sample}
derive for finite classes.

For a VC class, \cref{alg:swap-cover} runs this forecaster over a finite
cover built from unlabeled inputs.

\begin{algorithm}[t]
\caption{$\textsc{SwapCover}(\varepsilon,\delta)$: per-value exponential weights over a cover}
\label{alg:swap-cover}
\begin{algorithmic}[1]
\Require accuracy $\varepsilon$; confidence $\delta$; VC dimension $d$ of $\cH$
\State set $K=\lceil C/\varepsilon\rceil$, $m=K+1$, $\rho=\varepsilon/(8m)$,
  and $\ncv=\lceil C(d\log(1/\rho)+\log(2/\delta))/\rho\rceil$
\State draw $\ncv$ unlabeled inputs; let $\cG$ contain one auditor for each
  labeling of these inputs realized by $\cH$, together with its negation
\State set $N=|\cG|$, $T=\wtO((m\log N+\log(1/\delta))/\varepsilon^2)$, and
  $\eta=\sqrt{2m\log N/T}$; for each $v\in\Lambda$, let $w_{1,v}$ be uniform
  on $\cG$
\For{$t=1,\ldots,T$}
  \State define $\bar g_{t,v}(x)=\sum_{g\in\cG}w_{t,v}(g)\,g(x)$ for
    $v\in\Lambda$, \ $Q_t(x)=\Root\bigl((\bar g_{t,v}(x))_{v\in\Lambda}\bigr)$,
    and $\mu_t(x)=\sum_vQ_t(x)(v)\,v$
  \State receive $X_t$, forecast $Q_t(X_t)$, and observe $Y_t$
  \For{each $v\in\Lambda$}
    \State $Z_{t,v}=Q_t(X_t)(v)\bigl(\mu_t(X_t)-Y_t\bigr)$ and
      $w_{t+1,v}(g)\propto w_{t,v}(g)\exp\bigl(\eta\,g(X_t)Z_{t,v}\bigr)$ for
      $g\in\cG$
  \EndFor
\EndFor
\State \Return the randomized predictor $\Qbar=\frac1T\sum_tQ_t$
\end{algorithmic}
\end{algorithm}

\begin{theorem}[VC classes]
\label{thm:swap-vc}
For every class $\cH$ of VC dimension $d$ and every $\delta\in(0,1)$,
\cref{alg:swap-cover} uses $\wtO((d+\log(1/\delta))\varepsilon^{-3})$
independent examples and outputs, with probability at least $1-\delta$, a
randomized predictor $Q$ with
$\SMC(Q;\cH,\cD)\le\varepsilon$.  Its cover has $(d/\varepsilon)^{O(d)}$
elements, and its running time is exponential in $d$.
\end{theorem}

\paragraph{Proof sketch.}
With the parameters of \cref{alg:swap-cover}, $m=O(1/\varepsilon)$.  A
sample of
$\ncv=\wtO(d/\rho)=\wtO(d\varepsilon^{-2})$ unlabeled inputs is a $\rho$-net
for the disagreement sets of $\cH$, so one representative of each labeling of
this sample, together with its negation, forms a cover $\cG$ with
$\log|\cG|=O(d\log\ncv)$.  \Cref{thm:swap-finite} on this cover needs
$\wtO(md/\varepsilon^2)=\wtO(d\varepsilon^{-3})$ examples, and replacing the auditor at each of the $m$ values by its representative in
$\cG$ changes each bias by at most $2\rho$, a total of
$2m\rho\le\varepsilon/4$.

\subsection{An oracle-efficient swap algorithm}
\label{sec:swap-poly}

With the hybrid learner of \cref{thm:hybrid} for $\Hpm$, whose VC dimension
is at most $Cd$, one learner per grid value gives swap multicalibration with
$\wtO(d\varepsilon^{-4})$ samples (\cref{thm:swap-poly}).  The forecaster, \cref{alg:swap}, is the swap
analogue of \cref{alg:signed} without tests: the learner $\cB_v$ at grid
value $v$ receives the virtual residual at $v$.

\begin{algorithm}[t]
\caption{$\textsc{PerValue}(K,T,L)$: per-value hybrid learners}
\label{alg:swap}
\begin{algorithmic}[1]
\Require grid size $K$; number of rounds $T$; number $L$ of auxiliary draws
\State for each $v\in\Lambda$, start a Gaussian follow-the-perturbed-leader
  learner $\cB_v$ over $\Hpm$ (\cref{thm:hybrid}) with horizon $T$
\For{$t=1,\ldots,T$}
  \State for each $v\in\Lambda$, draw $L$ hypotheses from the current
    distribution of $\cB_v$ and let $\hat h_{t,v}(x)$ be their average
    \Comment{$2Lm$ ERM calls over $\cH$ per round}
  \State define $Q_t(x)=\Root\bigl((\hat h_{t,v}(x))_{v\in\Lambda}\bigr)$
    and $\mu_t(x)=\sum_vQ_t(x)(v)\,v$
  \State receive $X_t$, forecast $Q_t(X_t)$, and observe $Y_t$
  \State for each $v\in\Lambda$, give $\bigl(X_t,\;Q_t(X_t)(v)(\mu_t(X_t)-Y_t)\bigr)$
    to $\cB_v$ \Comment{each $\cB_v$ makes its round draw (step~3 of \cref{def:hybrid}), unused by the forecaster: $2m$ ERM calls}
\EndFor
\State \Return the retained draws, which define $Q_1,\ldots,Q_T$
\end{algorithmic}
\end{algorithm}

\begin{theorem}[Oracle-efficient swap multicalibration]
\label{thm:swap-poly}
Let $\delta,\alpha\in(0,1]$, and run \cref{alg:swap} on independent inputs,
with outcomes chosen as in \cref{thm:online} (not depending on the
learners' round draws), and with $L\ge2\alpha^{-2}\log(4mT/\delta)$ auxiliary draws, so that each average
$\hat h_{t,v}$ estimates the mixed prediction $\bh_{t,v}$ of $\cB_v$ within
$\alpha$.  Then, with
probability at least $1-\delta$,
\[
  \SMC_T\le\frac{T}{2K}+T\alpha+m\Reg\Bigl(T,Cd,\frac{\delta}{2m}\Bigr).
\]
On $T$ independent examples from $\cD$, the averaged predictor satisfies,
with probability at least $1-\delta$,
\[
  \SMC(\Qbar)\le\frac1{2K}+\alpha+\frac mT\Reg\Bigl(T,Cd,\frac{\delta}{4m}\Bigr)
  +C\sqrt{\frac{md\log(emT)+\log(1/\delta)}{T}}.
\]
Consequently, $T=\wtO((d+\log(1/\delta))\varepsilon^{-4})$ examples and
$\wtO((d+\log(1/\delta))\varepsilon^{-7})$ ERM calls suffice for
$\SMC(\Qbar)\le\varepsilon$ with probability at least $1-\delta$.
\end{theorem}

\paragraph{Proof sketch.}
The swap error is a sum over grid values of separate suprema, so it splits
into $m$ online problems, one per value, each handled by its own hybrid
learner, while the root rule makes the forecaster's own contribution
nonpositive.  Precisely, the root rule makes
$\sum_v\hat h_{t,v}(X_t)Q_t(X_t)(v)(\mu_t(X_t)-Y_t)$ nonpositive, and by
Hoeffding's inequality and a union bound over values and rounds, replacing
$\hat h_{t,v}$ by $\bh_{t,v}$ costs at most $\alpha$ per round.  Applying
\cref{thm:hybrid} at each value and summing bounds the \emph{virtual swap
error}, the swap error with virtual residuals in place of actual ones, and
\cref{lem:adjacent} adds $T/(2K)$.  The batch bound uses
\cref{lem:martingale}(b).  With $K,1/\alpha=\Theta(1/\varepsilon)$ the
dominant term is $m\Reg/T=\wtO(\varepsilon^{-1}\sqrt{d/T})$, which gives
$T=\wtO(d\varepsilon^{-4})$, and the number of ERM calls is
$2Tm(L+1)=\wtO(d\varepsilon^{-4}\cdot\varepsilon^{-1}\cdot\varepsilon^{-2})$.

\begin{remark}[Why the swap rate is quartic]
\label{rem:quartic}
The learner at grid value $v$ receives feedback on every round, but the
feedback is nonzero only to the extent that the forecast puts mass on $v$.
If its regret scaled with $\sqrt{d\sum_tQ_t(X_t)(v)}$ rather than
$\sqrt{dT}$, summing over values would give $\wtO(\sqrt{TKd})$ and the cubic
rate.  \Cref{thm:hybrid} does not provide this, and running the learner only
on rounds selected according to $Q_t(X_t)(v)$ does not help: the selected
inputs are no longer independent draws from $\cD_X$.  The adversarial online
learners assumed by \citet{hu2025efficient} tolerate this selection, and the
exponential weights of \cref{thm:swap-finite} obtain the aggregate rate
through their second-order regret bound.  We do not know an algorithm that
obtains it with polynomially many ERM calls.
\end{remark}

\subsection{Lower bound}
\label{sec:swap-lb}

We construct a family of distributions on which learning a
swap-multicalibrated predictor requires $\Omega(d/\varepsilon^3)$ samples,
while ordinary multicalibration is achieved by a fixed predictor
(\cref{thm:swap-lb}); a variant gives the same for VC dimension one
(\cref{prop:swap-lb-vc1}).  The same family gives the sample lower bounds for
error estimation in \cref{sec:audit-lb}.  The construction resembles that of
\cref{sec:erm-lb}, and we reuse $\beta$, $J$, $c_j$, and $\tau$ for the gap,
the number of blocks, the block centers, and the hidden signs, with the
values given below; here $\theta$ instead denotes the parameter from which
$\tau$ is built, and $k$ indexes the inputs within a block.

\paragraph{The balanced family.}
Let $d\ge2$ be even, $\beta\in(0,1/400]$, and let $J$ be an integer with
$1\le J\le1/(400\beta)$.  The inputs are the pairs $(j,k)\in[J]\times[d]$,
ordered lexicographically, with the uniform distribution; block $j$ is the
set $\{j\}\times[d]$.  Let $c_j=1/4+100j\beta$, so $c_j\in(1/4,1/2]$.  For a
parameter $\theta\in\{-1,+1\}^{J\times(d/2)}$, define signs
$\tau_{j,2i-1}=\theta_{j,i}$ and $\tau_{j,2i}=-\theta_{j,i}$, which satisfy
$\sum_k\tau_{j,k}=0$ in each block, and let $\cD_\theta$ be the distribution
in which $Y$ given $X=(j,k)$ is Bernoulli with mean
$\fs_\theta(j,k)=c_j+\beta\tau_{j,k}\in[1/8,5/8]$.  Let $\cD_\emptyset$ be the
distribution in which this mean is $c_j$.  The class $\cH_{J,d}$ consists of
\begin{enumerate}[leftmargin=2em]
\item the \emph{thresholds}, equal to $+1$ on an initial segment of the
  ordered inputs and $-1$ afterwards, including both constants; and
\item the \emph{local auditors} $h_{j,u}$ for $j\in[J]$ and $u\in\{-1,+1\}^d$,
  equal to $u_k$ at $(j,k)$ and $+1$ outside block $j$.
\end{enumerate}
Neither the class nor the input distribution depends on $\theta$.  Let
$Q^0$ be the deterministic predictor with $Q^0(j,k)=c_j$, which has $J$
values.

\begin{lemma}[The balanced family]
\label{lem:lb-class}
The class $\cH_{J,d}$ has VC dimension $d$, and weighted ERM over it takes
$O(Jd)$ time.  For every $\theta$,
$\MC(Q^0;\cH_{J,d},\cD_\theta)=\beta/J$ and
$\SMC(Q^0;\cH_{J,d},\cD_\theta)=\beta$, while both errors are zero under
$\cD_\emptyset$.
\end{lemma}

\begin{theorem}[Lower bound for swap multicalibration]
\label{thm:swap-lb}
Let $d\ge2$ be even, $\beta\in(0,1/8400]$, and $J=\lfloor1/(400\beta)\rfloor$,
so that $J\ge21$.  Suppose a learner, given $n$ independent samples from $\cD_\theta$, outputs a
finitely supported randomized predictor $Q$ with
$\SMC(Q;\cH_{J,d},\cD_\theta)\le\beta/20$ with probability at least $2/3$,
for every $\theta$.  Then
\[
  n>\frac{Jd}{6400\beta^2}=\Omega\Bigl(\frac{d}{\beta^3}\Bigr).
\]
This holds for learners with unlimited computation that know the input
distribution and the class.  By contrast, the fixed predictor $Q^0$, which
uses no samples, has $\MC(Q^0;\cH_{J,d},\cD_\theta)=\beta/J\le\beta/20$ for
every $\theta$ (\cref{lem:lb-class}).
\end{theorem}

\paragraph{Proof sketch.}
Let $Q$ have swap error at most $\beta/20$, and let $j(X)$ denote the block of
$X$.  The thresholds that separate the centers $c_j$ below and above a value
$v$ show that $Q$ rarely predicts far from $c_{j(X)}$
(\cref{lem:decode}).  Near $c_j$, the local auditors of block $j$ can match
the sign of the residual at every input of the block, so the swap error is at
least the mean absolute error of predictions near the centers.  Hence the
sign of the deviation of the prediction from $c_j$ recovers $\tau_{j,k}$ at
all but a small constant fraction of the inputs.  Changing one coordinate of
$\theta$ changes the mean outcome at two of the $Jd$ inputs by $2\beta$, so
with fewer than $Jd/(6400\beta^2)$ samples Assouad's lemma
\citep{assouad1983deux,yu1997assouad} forces a larger constant fraction of
errors in recovering the coordinates of $\theta$, a contradiction.  The full proof is
in \cref{app:lb}.

\begin{proposition}[VC dimension one]
\label{prop:swap-lb-vc1}
With $d=2$, the conclusion of \cref{thm:swap-lb} holds for the class of
thresholds alone, which has VC dimension one.
\end{proposition}

With $d=2$, each block has two inputs, and the thresholds that cut a block
between them play the role of its local auditors; the proof is in
\cref{app:lb}.

\begin{remark}[The hard family is computationally easy]
On this family the swap rate is attained in polynomial time: estimating each
of the $Jd$ conditional means to accuracy $O(\beta)$ uses $\wtO(Jd/\beta^2)$
samples, and a deterministic predictor $p$ satisfies
$\SMC(p)\le\E|p(X)-\fs_\theta(X)|$, since each value's bias is bounded by the
absolute error at the inputs that receive it.
\end{remark}

We restate \cref{thm:swap} and prove it.

\begin{restatement}{\Cref{thm:swap}}
Let $\varepsilon,\delta\in(0,1)$.
\begin{enumerate}[label=(\alph*),leftmargin=2em]
\item \emph{Sample complexity.} There is an algorithm
  (\cref{alg:swap-cover}), running in time exponential in $d$, that uses
  $\wtO((d+\log(1/\delta))/\varepsilon^{3})$ samples and with probability at
  least $1-\delta$ outputs a randomized predictor with swap
  multicalibration error at most $\varepsilon$.  Conversely, there is an absolute
  constant $c>0$ such that, if $\varepsilon\le10^{-6}$, there exist a class
  of VC dimension at most $d$ with linear-time weighted ERM and a family of
  distributions on which
  every learner that achieves swap multicalibration error at most
  $\varepsilon$ uses at least $c\,d/\varepsilon^{3}$ samples.
\item \emph{Oracle-efficient upper bound.} There is an algorithm
  (\cref{alg:swap}) that, given $\varepsilon$, $\delta$, and a weighted ERM
  oracle for $\cH$, uses $\wtO((d+\log(1/\delta))/\varepsilon^{4})$ samples
  and $\wtO((d+\log(1/\delta))/\varepsilon^{7})$ oracle calls, and with
  probability at least $1-\delta$ outputs a randomized predictor with swap
  multicalibration error at most $\varepsilon$.
\end{enumerate}
\end{restatement}

\begin{proof}
\emph{Part~(a).}  The upper bound is \cref{thm:swap-vc}.  For the lower
bound, let $\varepsilon\le10^{-6}$ and $\beta=20\varepsilon$, so that
$\beta\le2\cdot10^{-5}\le1/8400$.  If $d\ge2$ is even, apply
\cref{thm:swap-lb} with this $d$; if $d\ge3$ is odd, apply it with $d-1\ge
d/2$, noting that a class of VC dimension $d-1$ has VC dimension at most
$d$; if $d=1$, apply \cref{prop:swap-lb-vc1}.  In each case the required
sample size exceeds $Jd'/(6400\beta^2)$ with $d'\ge d/2$ and
$J\ge1/(800\beta)$, which is at least $c\,d/\varepsilon^3$.

\emph{Part~(b).}  This is the last statement of \cref{thm:swap-poly}; an
ERM oracle for $\Hpm$ is implemented by two calls to an ERM oracle for
$\cH$.
\end{proof}

\section{Estimating the multicalibration error}
\label{sec:audit}

An estimator receives a predictor and samples from the distribution, and
estimates the predictor's multicalibration error.  This section determines
how many samples error estimation needs for the two notions, and when it can be done
with an ERM oracle.  \Cref{sec:audit-def} defines the problem and bounds the
deviation of empirical biases (\cref{lem:audit-deviation}).
\Cref{sec:audit-swap} shows that estimating the swap error is oracle-efficient with
$\wtO(md/\varepsilon^2)$ samples for predictors with $m$ values
(\cref{thm:swap-audit}).  \Cref{sec:audit-ordinary} shows that estimating
the ordinary error can be done with $\wtO((d+m)/\varepsilon^2)$ samples if
computation is unrestricted (\cref{prop:mc-audit}), that Hadamard tests give an
oracle-efficient $\sqrt m$-approximation (\cref{prop:mc-audit-approx}), and
that a constant-factor approximation is NP-hard even when weighted ERM runs in linear time
(\cref{thm:np-hard}).  \Cref{sec:audit-lb} proves sample lower bounds for
both notions (\cref{thm:audit-lb}), and \cref{sec:audit-thm} proves
\cref{thm:audit}.  Long proofs are in \cref{app:audit}.

\subsection{The estimation problem}
\label{sec:audit-def}

We formalize the problem of estimating the error of a given predictor, and show
that empirical biases concentrate uniformly over the relevant tests
(\cref{lem:audit-deviation}).

\begin{definition}[Error estimator]
\label{def:auditor}
Let $\cV\subset[0,1]$ be a set of $m$ values.  An \emph{$\varepsilon$-estimator}
for ordinary (respectively, swap) multicalibration receives $\varepsilon$, $\delta$, access to a randomized predictor $Q$ with
values in $\cV$ through evaluations of the probabilities $Q_x(v)$, and
independent samples from an unknown distribution $\cD$.  It outputs an estimate $\hat e$ such that
$|\hat e-\MC(Q)|\le\varepsilon$ (respectively,
$|\hat e-\SMC(Q)|\le\varepsilon$) with probability at least $1-\delta$.  For
$\kappa\ge1$, a \emph{$(\kappa,\varepsilon)$-approximate estimator} for ordinary
multicalibration instead guarantees
$\MC(Q)/\kappa-\varepsilon\le\hat e\le\MC(Q)+\varepsilon$.  An estimator is
\emph{oracle-efficient} if it accesses $\cH$ only through a weighted ERM
oracle, and its number of oracle calls and other operations is polynomial in
$m$, $d$, $1/\varepsilon$, $\log(1/\delta)$, and its number of samples.
\end{definition}

Given samples $(X_i,Y_i)_{i\le n}$, the \emph{empirical bias} of $Q$ at $v$
with respect to $h$ is
\[
  \widehat B_h(Q,v)=\frac1n\sum_{i=1}^nh(X_i)\,Q_{X_i}(v)\,(v-Y_i).
\]

\begin{lemma}[Deviation of empirical biases]
\label{lem:audit-deviation}
Let $Q$ be a fixed randomized predictor with $m$ values, and let $n\ge2$.
\begin{enumerate}[label=(\alph*),leftmargin=2em]
\item For a family $\cS\subseteq\{-1,+1\}^{\cV}$ of $N$ sign vectors, with
  probability at least $1-\delta$,
  \[
    \sup_{h\in\cH,\ \sigma\in\cS}\Bigl|\sum_{v}\sigma_v
    \bigl(\widehat B_h(Q,v)-B_h(Q,v)\bigr)\Bigr|
    \le C\sqrt{\frac{d\log(en)+\log(2N/\delta)}{n}}.
  \]
\item With probability at least $1-\delta$,
  \[
    \sum_v\sup_{h\in\cH}\bigl|\widehat B_h(Q,v)-B_h(Q,v)\bigr|
    \le C\sqrt{\frac{md\log(emn)+\log(1/\delta)}{n}}.
  \]
\end{enumerate}
\end{lemma}

\begin{proof}
Maps that do not depend on the sample are predictable, so
\cref{lem:martingale} applies; its proof holds verbatim with $\Lambda$
replaced by any set of $m$ values.  For~(a), use the maps
$z_\sigma(x,y)=\sum_v\sigma_vQ_x(v)(v-y)$, which take values in $[-1,1]$, and
divide by $n$.  For~(b), use $z_v(x,y)=Q_x(v)(v-y)$, which satisfy
$\sum_v|z_v|\le\sum_vQ_x(v)=1$.  The deviation at $v$ is linear in $h$ and
$\Hpm$ is closed under negation, so the left side of~(b) equals the supremum
over $(h_v)\in(\Hpm)^{\cV}$ of
$\sum_v\bigl(\widehat B_{h_v}(Q,v)-B_{h_v}(Q,v)\bigr)$, which
\cref{lem:martingale}(b) bounds after dividing by $n$.
\end{proof}

\subsection{Estimating the swap error}
\label{sec:audit-swap}

Estimating the swap error splits into one ordinary ERM problem per value
(\cref{prop:swap-audit}), and the plug-in estimate is accurate from
$\wtO(md/\varepsilon^2)$ samples (\cref{thm:swap-audit}).

\begin{proposition}[The empirical swap error reduces to ERM]
\label{prop:swap-audit}
Given a finite distribution with known probabilities, or an empirical sample,
and a randomized predictor with $m$ values, the swap multicalibration error
and an auditor attaining it at each value can be computed with $2m$ weighted
ERM calls over $\cH$.
\end{proposition}

\begin{proof}
At each value $v$, maximize and minimize the weighted sum
$\sum_ih(x_i)\,\Pr(x_i)\,Q_{x_i}(v)(v-y_i)$ over $h\in\cH$ with two ERM calls,
keep the larger absolute value, and sum over $v$.
\end{proof}

\begin{theorem}[Oracle-efficient estimation of the swap error]
\label{thm:swap-audit}
Let $Q$ have $m$ values, and let
\[
  n\ge C\,\frac{md\log(emd/\varepsilon)+\log(1/\delta)}{\varepsilon^2}
\]
for a sufficiently large absolute constant $C$.  The estimate
$\hat e=\sum_v\max_{h\in\cH}|\widehat B_h(Q,v)|$, computed from $n$ samples
with $2m$ weighted ERM calls, satisfies $|\hat e-\SMC(Q)|\le\varepsilon$ with
probability at least $1-\delta$.  On the same event, the maximizing auditors
$h_v$ satisfy $\sum_v|B_{h_v}(Q,v)|\ge\SMC(Q)-2\varepsilon$.  For
$m=\Theta(1/\varepsilon)$ the sample size is $\wtO(d\varepsilon^{-3})$.
\end{theorem}

\begin{proof}
Let $\mathrm{dev}$ be the left side of \cref{lem:audit-deviation}(b).  Since
$|\sup_h|a_h|-\sup_h|b_h||\le\sup_h|a_h-b_h|$ for real numbers $a_h,b_h$, we
have $|\hat e-\SMC(Q)|\le\mathrm{dev}$ and
$\sum_v|B_{h_v}(Q,v)|\ge\hat e-\mathrm{dev}\ge\SMC(Q)-2\,\mathrm{dev}$.  For
the stated $n$, \cref{lem:audit-deviation}(b) gives
$\mathrm{dev}\le\varepsilon$ with probability at least $1-\delta$.
\end{proof}

\subsection{Estimating the ordinary error}
\label{sec:audit-ordinary}

The plug-in estimate of the ordinary error is accurate from fewer samples
than in the swap case (\cref{prop:mc-audit}), but computing it requires a
joint search over auditors and sign patterns.  Fixed families of tests give
oracle-efficient approximations (\cref{prop:mc-audit-approx}), and
\cref{thm:np-hard} shows that a constant-factor approximation is NP-hard even
when weighted ERM runs in linear time.

\begin{proposition}[Statistical estimation of the ordinary error]
\label{prop:mc-audit}
Let $Q$ have $m$ values, and let
\[
  n\ge C\,\frac{d\log(ed/\varepsilon)+m+\log(1/\delta)}{\varepsilon^2}
\]
for a sufficiently large absolute constant $C$.  The empirical error
\[
  \hat e=\max_{h\in\cH,\,\sigma\in\{-1,+1\}^{\cV}}\sum_v\sigma_v\widehat B_h(Q,v),
\]
computed from $n$ samples with $2^m$ weighted ERM calls, satisfies
$|\hat e-\MC(Q)|\le\varepsilon$ with probability at least $1-\delta$.  For
$m=\Theta(1/\varepsilon)$ the sample size is
$\wtO(\varepsilon^{-3}+d\varepsilon^{-2})$.
\end{proposition}

\begin{proof}
For each $\sigma$, one ERM call with weight
$n^{-1}\sum_v\sigma_vQ_{X_i}(v)(v-Y_i)$ at $X_i$ maximizes over $h$.  By
\cref{lem:signed}, $\hat e$ and $\MC(Q)$ are the suprema of
$\sum_v\sigma_v\widehat B_h(Q,v)$ and of $\sum_v\sigma_vB_h(Q,v)$ over $h$ and
$\sigma$, so they differ by at most the left side of
\cref{lem:audit-deviation}(a) with $N=2^m$.
\end{proof}

\begin{proposition}[Oracle-efficient approximation of the ordinary error]
\label{prop:mc-audit-approx}
Let $\cS\subseteq\{-1,+1\}^{\cV}$ have $N$ elements and distortion $\kappa$
(\cref{def:distortion}), and let
$n\ge C(d\log(ed/\varepsilon)+\log(N/\delta))/\varepsilon^2$.  The estimate
$\hat e=\max_{h\in\cH,\,\sigma\in\cS}\sum_v\sigma_v\widehat B_h(Q,v)$, computed
with $N$ weighted ERM calls, satisfies
$\MC(Q)/\kappa-\varepsilon\le\hat e\le\MC(Q)+\varepsilon$ with probability at
least $1-\delta$.  In particular, Hadamard tests $\cS=\Shad$ give a
$(\sqrt m,\varepsilon)$-approximate estimator with fewer than $4m$ ERM calls and
$O((d\log(ed/\varepsilon)+\log(m/\delta))/\varepsilon^2)$ samples.
\end{proposition}

\begin{proof}
Let $\mathrm{dev}\le\varepsilon$ be the deviation in
\cref{lem:audit-deviation}(a).  Since
$\sum_v\sigma_vB_h(Q,v)\le\sum_v|B_h(Q,v)|$, we have
$\hat e\le\MC(Q)+\mathrm{dev}$.  By the distortion of $\cS$, for every $h$,
\[
  \sum_v|B_h(Q,v)|\le\kappa\max_{\sigma\in\cS}\sum_v\sigma_vB_h(Q,v)
  \le\kappa(\hat e+\mathrm{dev});
\]
take the supremum over $h$.  For $\Shad$ use \cref{lem:hadamard}.
\end{proof}

By \cref{prop:distortion-lb}, no family of $N$ fixed tests yields an
approximation factor better than $\sqrt{m/(2\log(4N))}$ through this
argument.  The next theorem shows that, unless $\mathrm{NP}\subseteq\mathrm{BPP}$, no
oracle-efficient estimator, whether or not it uses fixed tests,
approximates the ordinary error within a factor close to $1$.

\begin{theorem}[The ordinary error is hard to approximate]
\label{thm:np-hard}
There are absolute constants $c,c_0>0$ such that the following holds unless
$\mathrm{NP}\subseteq\mathrm{BPP}$.  No algorithm that runs in time
polynomial in $m$ and $d$, with access to independent samples and to a
weighted ERM oracle for $\cH$, outputs with probability at least $2/3$ an
estimate $\hat e$ with $|\hat e-\MC(p)|\le c/m$ on every instance; nor one
with $(1-c_0)\MC(p)\le\hat e\le(1+c_0)\MC(p)$ on every instance.  This holds
even for instances with finite input spaces, binary outcomes, deterministic
predictors $p$ whose $m$ values lie on a grid, and classes $\cH$ whose
weighted ERM takes linear time.  Moreover, computing $\MC(p)$ exactly for an
explicitly given finite distribution is NP-hard for such classes.
\end{theorem}

\paragraph{Proof idea.}
The reduction is from \textsc{Max-Cut}: three inputs for each edge of a
graph, and one grid value for each vertex.  The sign pattern encodes a cut,
and the best auditor collects the cut edges (\cref{app:audit}).  Since
\textsc{Max-Cut} is NP-hard to approximate within some constant factor
\citep{hastad2001some,trevisan2000gadgets}, so is the ordinary error.

\paragraph{Consequences.}
Take $\varepsilon=c/m$.  An $\varepsilon$-estimator for the ordinary error
then estimates it within $c/m$, and a running time polynomial in $m$, $d$,
and $1/\varepsilon$ is polynomial in $m$ and $d$.  Hence, unless
$\mathrm{NP}\subseteq\mathrm{BPP}$, no oracle-efficient
(\cref{def:auditor}) $\varepsilon$-estimator with polynomially many samples
exists for predictors with $m=c/\varepsilon$ values.  In this regime the
exhaustive search of \cref{prop:mc-audit} uses $2^{O(1/\varepsilon)}$ ERM
calls.  The theorem concerns estimating the error of a given predictor; it does not
rule out learning a multicalibrated predictor with an ERM oracle, which
\cref{thm:main}(a) achieves.

\subsection{Sample lower bounds}
\label{sec:audit-lb}

A testing argument shows that error estimation needs $\Omega(\sqrt{md}/\varepsilon^2)$
samples for swap multicalibration and $\Omega(\sqrt d/\varepsilon^2)$ for
ordinary multicalibration (\cref{thm:audit-lb}).  It uses the balanced
family of \cref{sec:swap-lb} and the values of its errors given by
\cref{lem:lb-class}.

\begin{theorem}[Sample lower bounds for error estimation]
\label{thm:audit-lb}
In the balanced family of \cref{sec:swap-lb}, for every even $d\ge2$, every
$\beta\in(0,1/400]$, and every integer $J$ with $1\le J\le1/(400\beta)$,
every algorithm that outputs $0$ with probability at
least $2/3$ under $\cD_\emptyset$ and $1$ with probability at least $2/3$
under every $\cD_\theta$ uses more than $0.1\sqrt{Jd}/\beta^2$ samples, even
with unlimited computation and knowledge of the input distribution, the
class, and $Q^0$.  Consequently, for every even $d\ge2$ and every
$\varepsilon\le1/1200$:
\begin{enumerate}[label=(\alph*),leftmargin=2em]
\item for every $m\le1/(1200\varepsilon)$, every $\varepsilon$-estimator for swap
  multicalibration with $\delta=1/3$, applied to predictors with $m$ values,
  uses more than $\sqrt{md}/(90\varepsilon^2)$ samples;
\item every $\varepsilon$-estimator for ordinary multicalibration with
  $\delta=1/3$ uses more than $\sqrt d/(90\varepsilon^2)$ samples, already for
  constant predictors.
\end{enumerate}
\end{theorem}

\paragraph{Proof sketch.}
Under $\cD_\emptyset$ the mean outcome in block $j$ is $c_j$; under a
uniformly random $\theta$ it is perturbed by $\pm\beta$, with signs that
cancel in pairs within each block.  The
chi-square divergence between the law of $n$ samples under $\cD_\emptyset$
and the mixture over $\theta$ is at most $\exp(256n^2\beta^4/(9Jd))-1$, which
is small unless $n\gtrsim\sqrt{Jd}/\beta^2$.  For~(a), take $J=m$ and
$\beta=3\varepsilon$; by \cref{lem:lb-class}, an $\varepsilon$-estimator applied
to $Q^0$ distinguishes swap error $0$ from $3\varepsilon$.  For~(b), take
$J=1$, for which the ordinary error of $Q^0$ is also $\beta$ under every
$\cD_\theta$ (\cref{lem:lb-class}).  The full proof is in \cref{app:audit}.

\subsection{The error-estimation theorem}
\label{sec:audit-thm}

We restate \cref{thm:audit} and prove it from the results of this section.

\begin{restatement}{\Cref{thm:audit}}
Let $\varepsilon,\delta\in(0,1)$, and consider randomized predictors with at
most $m$ values, given through the probabilities they assign to each value
at each input.  There are absolute constants $c,c_0>0$ such that the
following hold.
\begin{enumerate}[label=(\alph*),leftmargin=2em]
\item \emph{Swap error.} With $\wtO((md+\log(1/\delta))/\varepsilon^2)$
  samples and $2m$ weighted ERM calls, one can estimate the swap
  multicalibration error within $\varepsilon$ with probability at least
  $1-\delta$.  If $d$ is even, $\varepsilon\le1/1200$, and
  $m\le1/(1200\varepsilon)$, then $\Omega(\sqrt{md}/\varepsilon^2)$ samples
  are necessary.
\item \emph{Ordinary error.} With $\wtO((d+m+\log(1/\delta))/\varepsilon^2)$
  samples and $2^m$ weighted ERM calls, one can estimate the
  multicalibration error within $\varepsilon$ with probability at least
  $1-\delta$, and with $\wtO((d+\log(m/\delta))/\varepsilon^2)$ samples and
  fewer than $4m$ calls, within a factor $\sqrt m$ up to an additive
  $\varepsilon$.  If $d$ is even and $\varepsilon\le1/1200$, then
  $\Omega(\sqrt d/\varepsilon^2)$ samples are necessary.  Unless
  $\mathrm{NP}\subseteq\mathrm{BPP}$, no algorithm running in time
  polynomial in $m$ and $d$ estimates the error within a factor $1\pm c_0$,
  or within an additive $c/m$, even for deterministic predictors and classes
  with linear-time weighted ERM.
\end{enumerate}
\end{restatement}

\begin{proof}
For part~(a), the upper bound is \cref{thm:swap-audit} and the lower bound is
\cref{thm:audit-lb}(a).  For part~(b), the exhaustive estimate is
\cref{prop:mc-audit}, the $\sqrt m$-approximation is
\cref{prop:mc-audit-approx} with $\cS=\Shad$, the sample lower bound is
\cref{thm:audit-lb}(b), and the hardness statement is \cref{thm:np-hard}.
\end{proof}

\section{Deterministic predictors and the main theorem}
\label{sec:det}

This section makes the output deterministic, following
\citet{noarov2026optimal}, and assembles the proof of \cref{thm:main}.
\Cref{sec:det-obstruction} recalls why rounding can fail at atoms of the input
distribution.  \Cref{sec:det-rounding} analyzes rounding with one seed per
cell for a VC class (\cref{thm:vc-rounding}), and shows that it suffices on
its own when there are no atoms (\cref{cor:atomless}).
\Cref{sec:det-intervals} constructs the confidence intervals and the
partition, bounds the resulting variance (\cref{lem:budget}), and shows
that training with forecasts restricted to the intervals keeps the
guarantee of \cref{thm:batch} (\cref{prop:batch-hints}).
\Cref{sec:det-main} gives the complete algorithm and proves both parts of
\cref{thm:main}, using \cref{cor:erm-lb} for part~(b).
Proofs are in \cref{app:det}.

\subsection{The obstruction at atoms}
\label{sec:det-obstruction}

A randomized predictor can be exactly multicalibrated even though every
deterministic predictor with values in its support is far from it; the
example below is the obstruction identified by \citet{noarov2026optimal}.

\begin{proposition}[Two atoms, cf.\ \citealp{noarov2026optimal}]
\label{prop:two-atoms}
Let $\cH=\{1\}$ contain only the constant auditor, let the input be $a$ or $b$
with probability $1/2$ each, and let $Y=0$ at $a$ and $Y=1$ at $b$.  Let
$Q_a(1/3)=2/3$, $Q_a(2/3)=1/3$, $Q_b(1/3)=1/3$, and $Q_b(2/3)=2/3$.  Then
$\MC(Q)=0$, but every deterministic $p$ with values in $\{1/3,2/3\}$ has
$\MC(p)\ge1/6$, and the predictor $x\mapsto\sum_vQ_x(v)v$ has error $4/9$.
\end{proposition}

\begin{proof}
At $v=1/3$ the bias is $\tfrac12\cdot\tfrac23\cdot\tfrac13+\tfrac12\cdot
\tfrac13\cdot(\tfrac13-1)=0$, and similarly at $v=2/3$.  If $p(a)=p(b)=v$,
the error is $|v-1/2|=1/6$.  If $p(a)\ne p(b)$, the error is
$\tfrac12(p(a)+1-p(b))\ge1/3$.  The mean predictor takes the values $4/9$ at
$a$ and $5/9$ at $b$, so its error is $\tfrac12(4/9+1-5/9)=4/9$.
\end{proof}

The example rules out converting an arbitrary randomized predictor by fixing its
random choices.  It does not rule out learning a different predictor, such as
one predicting $0$ at $a$ and $1$ at $b$.

\subsection{Cell rounding for VC classes}
\label{sec:det-rounding}

Rounding with one shared seed per cell changes the multicalibration error by
roughly the square root of the product of two quantities: a variance budget,
which is essentially the sum of squared cell masses, and the complexity of
the class plus $\log(1/\delta)$ (\cref{thm:vc-rounding}).

\paragraph{The rounding procedure.}
Let $Q$ be a randomized predictor with values in $\Lambda$.  Let $O\subseteq\cX$ be
a finite set of inputs, each with an interval $I_x=[\ylo_x,\yhi_x]$ of
length at most $2r_x$ that contains $\fs(x)$, and suppose that $Q_x$ is
supported within distance $1/K$ of $I_x$ for each $x\in O$.  Let $\cP$ be a
partition of $\cX$ consisting of the singletons $\{x\}$, $x\in O$, and of
finitely many further measurable cells $\cP_{\mathrm{rest}}$ that cover
$\cX\setminus O$.  Write $\omega_x=\cD_X(\{x\})$ and define the
\emph{variance budget}
\[
  V_{\cP}=\sum_{x\in O}\omega_x^2r_x^2+\sum_{R\in\cP_{\mathrm{rest}}}\cD_X(R)^2.
\]
Let $\Xi$ be a probability space and let $G_x:\Xi\to\Lambda$ be maps such that
$G_x(\xi)\sim Q_x$ when $\xi\sim\Xi$; for example, $G_x$ is the inverse
distribution function of $Q_x$ and $\Xi$ is uniform on $[0,1]$.  Draw
independent seeds $\xi_R\sim\Xi$, one for each cell $R\in\cP$, and set
\[
  p(x)=G_x(\xi_{R(x)}),
\]
where $R(x)$ is the cell containing $x$.  The seeds are drawn once, so $p$ is
a deterministic function.

The rounding error is a sum of independent contributions, one for each
cell.  \citet{noarov2026optimal} control it for a finite family of tests by
a union bound.  The next lemma controls it for a VC class of tests directly,
by chaining, so that no cover of the class is needed.

\begin{lemma}[Rounding against a VC class of tests]
\label{lem:cell-tests}
Let $\cG$ be a class of functions $g:\cX\times\Lambda\to\{-1,+1\}$ of VC
dimension $d_{\cG}\ge1$.  Let $\varphi:\cX\times\Lambda\to\R$ be a fixed weight function,
and for each cell $R\in\cP$ let $\varphi_R\ge|\varphi(x,v)|$ for all $x\in R$ and
$v\in\Lambda$.  Let $V_\varphi=\sum_{R\in\cP}\varphi_R^2\,\cD_X(R)^2$.  With probability
at least $1-\delta$ over the seeds,
\[
  \sup_{g\in\cG}\Bigl|\E\bigl[g(X,p(X))\,\varphi(X,p(X))\bigr]
  -\E\Bigl[\sum_{v}Q_X(v)\,g(X,v)\,\varphi(X,v)\Bigr]\Bigr|
  \le C\sqrt{V_\varphi\bigl(d_{\cG}+\log(2/\delta)\bigr)}.
\]
\end{lemma}

We apply \cref{lem:cell-tests} with $\varphi$ the residual $v-\fs(x)$, modified at
the inputs of $O$, and with the tests $g(x,v)=h(x)\sigma_v$ for ordinary
multicalibration and $g(x,v)=h_v(x)$, with a separate auditor at each grid
value, for swap multicalibration.  The tests appear only in the analysis.

\begin{theorem}[Cell rounding for VC classes]
\label{thm:vc-rounding}
Under the conditions above, with probability at least $1-\delta$ over the
seeds,
\begin{align*}
  \MC(p)&\le\MC(Q)+C\sqrt{V_{\cP}\bigl(d+m+\log(2/\delta)\bigr)}+\frac2K,\\
  \SMC(p)&\le\SMC(Q)+C\sqrt{V_{\cP}\bigl(md+\log(2/\delta)\bigr)}+\frac2K,
\end{align*}
and the terms $2/K$ can be omitted when $O=\emptyset$.
\end{theorem}

\paragraph{Proof sketch.}
At an input of $O$, replacing the residual $v-\fs(x)$ by
$\clip(v,I_x)-\fs(x)$ costs at most $1/K$ in each expectation, because the
forecast is supported within $1/K$ of $I_x$, and makes the residual at most
$2r_x$ in absolute value; elsewhere it is at most one.  So $V_\varphi\le4V_{\cP}$.
The tests $h(x)\sigma_v$ form a class of VC dimension $O(d+m)$, since their
number of patterns on any set of points is at most $2^m$ times that of
$\cH$, and the tests $h_v(x)$ form a class of VC dimension $O(md)$.  By
\cref{lem:signed}, the largest deterministic expectation over these tests is
the ordinary (respectively, swap) error of $p$, and the corresponding
randomized expectation is at most that of $Q$; \cref{lem:cell-tests} bounds
the difference.

\begin{corollary}[Inputs without atoms]
\label{cor:atomless}
Suppose that $\cX$ carries an order as in \cref{rem:order} and that $\cD_X$
has no atoms.  For every $\varepsilon,\delta\in(0,1)$, any randomized
predictor $Q$ with values in $\Lambda$ can be converted, using no labels or
oracle calls and
\[
  O\Bigl(\frac{d+m+\log(1/\delta)}{\varepsilon^2}
  \log\frac{e(d+m+\log(1/\delta))}{\varepsilon\delta}\Bigr)
\]
unlabeled inputs, into a deterministic predictor $p$ that satisfies
$\MC(p)\le\MC(Q)+\varepsilon$ with probability at least $1-\delta$.  For swap
multicalibration, the same holds with $md$ in place of $d+m$.
\end{corollary}

The cells are the gaps between consecutive points of a sorted unlabeled
sample (\cref{app:det}).  More generally, the conversion works whenever
$\sum_x\cD_X(\{x\})^2$ is small; large atoms are exactly the obstruction of
\cref{prop:two-atoms}.

\subsection{Confidence intervals and the partition}
\label{sec:det-intervals}

The intervals and the partition are those of \citet{noarov2026optimal}.
For a failure parameter $\delta_0\in(0,1)$, the intervals are valid with
probability $1-\delta_0$ (\cref{lem:coverage}), and the variance budget of
the resulting partition is $\wtO(1/\ncf+1/\npt)$ with probability
$1-2\delta_0$
(\cref{lem:budget}).  Training restricts the forecasts to the intervals,
with the same guarantee as \cref{thm:batch} (\cref{prop:batch-hints}).

\paragraph{Confidence intervals.}
Draw a confidence sample of $\ncf$ labeled examples, independent of all other
data, and let $O$ be the set of inputs that occur in it.  For $x\in O$, let
$n_x$ be its number of occurrences and $\hat y_x$ the average of its outcomes.
Let $\iota=\max\{1,\log(8\ncf/\delta_0)\}$, and set
$r_x=\min\{1,\sqrt{\iota/n_x}\}$ and
$I_x=[\hat y_x-r_x,\hat y_x+r_x]\cap[0,1]$; since $\iota\ge1$, $I_x=[0,1]$
when $n_x=1$.

\begin{lemma}[Coverage]
\label{lem:coverage}
With probability at least $1-\delta_0$, $\fs(x)\in I_x$ for every $x\in O$.
\end{lemma}

\begin{proof}
Conditional on the inputs of the confidence sample, the outcomes at $x$ are
independent with mean $\fs(x)$.  If $r_x=1$ there is nothing to prove;
otherwise Hoeffding's inequality gives
$\Pr\{|\hat y_x-\fs(x)|>\sqrt{\iota/n_x}\}\le2e^{-2\iota}\le\delta_0/\ncf$.  A
union bound over the at most $\ncf$ elements of $O$ completes the proof.
\end{proof}

\paragraph{The partition.}
Draw $\npt$ unlabeled inputs, independent of all other data, and sort their
distinct values $z_1<\cdots<z_k$ using the order of \cref{rem:order}.  The
cells of $\cP_{\mathrm{rest}}$ are the sets $\{z_i\}\setminus O$, the open
order intervals $(z_i,z_{i+1})\setminus O$, and the two unbounded intervals
before $z_1$ and after $z_k$, minus $O$; empty sets are discarded.  Together
with the singletons of $O$, this gives a partition $\cP$ with at most
$\ncf+2\npt+1$ cells.

\begin{lemma}[Variance budget, adapted from \citealp{noarov2026optimal}]
\label{lem:budget}
For this construction, with probability at least $1-2\delta_0$,
\[
  V_{\cP}\le C\Bigl(\frac{\log(e\ncf/\delta_0)}{\ncf}
  +\frac{\log(e\npt/\delta_0)}{\npt}\Bigr).
\]
\end{lemma}

The first term accounts for the atoms observed in the confidence sample,
whose interval widths shrink with their counts, and for the atoms not
observed, all of which are light; the second accounts for the gaps of the
sorted partition sample, none of which has large mass.

\paragraph{Forecasts restricted to the intervals.}
The rounding analysis of \cref{sec:det-rounding} needs every forecast at an
input $x\in O$ to be supported within $1/K$ of $I_x=[\ylo_x,\yhi_x]$.  Let
\[
  \Lambda_x=\Lambda\cap[\vlo_x,\vhi_x],\qquad
  \vlo_x=\lfloor K\ylo_x\rfloor/K,\quad \vhi_x=\lceil K\yhi_x\rceil/K,
\]
a nonempty set of consecutive grid values whose hull $[\vlo_x,\vhi_x]$
contains $I_x$ and lies within $1/K$ of it, and let
$\Lambda_x=\Lambda$ for $x\notin O$.  The root rule extends to such sets.

\begin{lemma}[Root rule on consecutive grid values]
\label{lem:root-sub}
Let $\Lambda'=\Lambda\cap[\vlo,\vhi]$ for grid values $\vlo\le\vhi$, and let
$a\in[-1,1]^{\Lambda'}$.  One can compute, in $O(m)$ time, a distribution
$P=\Root_{\Lambda'}(a)\in\Delta(\Lambda')$ with adjacent support such that,
with $\mu=\sum_vP(v)\,v$,
\[
  (\mu-y)\sum_{v\in\Lambda'}P(v)\,a_v\le0
  \qquad\text{for every }y\in[\vlo,\vhi].
\]
\end{lemma}

\begin{proof}
The proof of \cref{lem:root} applies with $0$ and $1$ replaced by $\vlo$ and
$\vhi$.  If $\vlo=\vhi$, the first case applies when $a_{\vlo}\ge0$ and the
second otherwise, since then $a_{\vhi}=a_{\vlo}<0$.
\end{proof}

We define \textsc{SignedTests} \emph{with interval hints} to be \cref{alg:signed} with
the forecast $Q_t(x)=\Root\bigl(\widehat A_t(x)\bigr)$ replaced by
$Q_t(x)=\Root_{\Lambda_x}\bigl(\widehat A_t(x)|_{\Lambda_x}\bigr)$; nothing else
changes.  Its forecasts at each $x\in O$ are supported in $\Lambda_x$.  The
intervals are \emph{valid} if $\fs(x)\in I_x$ for every $x\in O$.

\begin{proposition}[Batch guarantee with interval hints]
\label{prop:batch-hints}
Under the conditions of \cref{thm:batch}, suppose that the intervals are
fixed independently of the examples $(X_t,Y_t)_{t\le T}$ and are valid.
Then, conditionally on the intervals, the conclusion of \cref{thm:batch}
holds for \cref{alg:signed} with interval hints.
\end{proposition}

The restricted rule makes the estimated $w_t$-weighted reward
$(\mu_t-y)\sum_vQ_t(X_t)(v)\widehat A_t(X_t)_v$ nonpositive only for outcomes
$y\in[\vlo_{X_t},\vhi_{X_t}]$.  So this step, which holds pathwise in
\cref{thm:batch}, now holds only in conditional expectation: the reward is
linear in the outcome, and conditioning on $X_t$ replaces $Y_t$ by
$\fs(X_t)\in I_{X_t}$.  Azuma's inequality controls the difference within
the existing terms of the bound; the proof is in \cref{app:det}.

\subsection{The complete algorithm}
\label{sec:det-main}

\Cref{alg:main} runs \cref{alg:signed} with Hadamard tests.  For a
randomized output it returns the averaged forecast.  For a deterministic
output it first builds confidence intervals, trains with forecasts restricted to these intervals, builds the partition, and then
rounds with one seed per cell.  The seed of a cell $R$ is a
pair $(t_R,\xi_R)$ of a uniformly random round and a uniform number in
$[0,1]$; the $\xi_R$-quantile of $Q_{t_R}(x)$ is then a draw from the
averaged forecast $\Qbar_x$ that requires evaluating a single round.

\begin{algorithm}[t]
\caption{$\textsc{OEMC}(\varepsilon,\delta,\textsf{mode})$: oracle-efficient multicalibration}
\label{alg:main}
\begin{algorithmic}[1]
\Require accuracy $\varepsilon$; confidence $\delta$; \textsf{mode} $\in\{$randomized, deterministic$\}$
\State set $K,T,L,\eta,\ncf,\npt$ as in \cref{app:params} (in deterministic mode,
  for accuracy $\varepsilon/2$ and confidence $\delta/4$)
\If{\textsf{mode} $=$ randomized}
  \State run $\textsc{SignedTests}(\Shad,K,T,\eta,L)$ on $T$ examples
  \State \Return the randomized predictor $\Qbar=\frac1T\sum_tQ_t$
\EndIf
\State draw $\ncf$ examples; build $O$ and the intervals $I_x$, $x\in O$,
  with failure parameter $\delta_0=\delta/8$ (\cref{sec:det-intervals})
\State run $\textsc{SignedTests}(\Shad,K,T,\eta,L)$ with interval hints
  $\{I_x\}_{x\in O}$ (\cref{sec:det-intervals}) on $T$ fresh examples
\State draw $\npt$ fresh inputs and build the partition $\cP$
  (\cref{sec:det-intervals})
\For{each cell $R\in\cP$}
  \State draw $t_R$ uniformly from $[T]$ and $\xi_R$ uniformly from $[0,1]$
\EndFor
\State \Return $p(x)=$ the $\xi_{R(x)}$-quantile of $Q_{t_{R(x)}}(x)$
\end{algorithmic}
\end{algorithm}

We restate \cref{thm:main} and prove it.

\begin{restatement}{\Cref{thm:main}}
\mbox{}
\begin{enumerate}[label=(\alph*),leftmargin=2em]
\item \emph{Upper bound.} There is an algorithm (\cref{alg:main}) that,
  given $\varepsilon,\delta\in(0,1)$ and a weighted ERM oracle for $\cH$,
  uses $\wtO((d+\log(1/\delta))/\varepsilon^{3})$ samples and
  $\wtO((d+\log(1/\delta))/\varepsilon^{6})$ oracle calls, and with
  probability at least $1-\delta$ outputs a randomized predictor with
  multicalibration error at most $\varepsilon$.  If $\cX$ carries a
  computable total order (\cref{rem:order}), the output can instead be a
  deterministic predictor, with the same guarantees.
\item \emph{Lower bound.} When $\widetilde\Omega(1/\varepsilon)\le
  d\le\mathrm{poly}(1/\varepsilon)$, let $A$ be a black-box ERM learner that
  makes $\mathrm{poly}(d/\varepsilon)$ oracle calls and, for every class of
  VC dimension at most $d$, outputs a randomized predictor with
  multicalibration error at most $\varepsilon$.  Then $A$ uses
  $\widetilde\Omega(d/\varepsilon^{3})$ samples.
\end{enumerate}
In part~(b), the constants hidden by $\widetilde\Omega$ depend on the degrees
of the polynomials; \cref{cor:erm-lb} gives the explicit form.
\end{restatement}

\begin{proof}
\emph{Part~(a), randomized output.}  This is \cref{cor:batch-rates}(i).

\emph{Part~(a), deterministic output: training.}  Build the intervals of \cref{sec:det-intervals} with $\delta_0=\delta/8$.
By \cref{lem:coverage}, the intervals are valid with probability at least
$1-\delta/8$.  The confidence sample is independent of the training examples, and the
intervals depend only on it, so on this event \cref{prop:batch-hints}
applies with $\cS=\Shad$.  Its bound is that of \cref{thm:batch}, so by
\cref{cor:batch-rates}(i) with
$\varepsilon/2$ and $\delta/4$ in place of $\varepsilon$ and $\delta$, the
parameters of \cref{app:params} give $\MC(\Qbar)\le\varepsilon/2$ with
probability at least $1-\delta/4$.  Each forecast $Q_t(x)$ is supported in
$\Lambda_x$, which lies within $1/K$ of $I_x$; hence so is $\Qbar_x$, and the
conditions of \cref{sec:det-rounding} hold for $\Qbar$, the intervals, and
any partition as in \cref{sec:det-intervals}.

\emph{Part~(a), deterministic output: rounding.}  Build the partition of
\cref{sec:det-intervals}.  With $\Xi$ uniform on $[T]\times[0,1]$, the seed $(t_R,\xi_R)$ for cell $R$,
and $G_x(t,\xi)$ the $\xi$-quantile of $Q_t(x)$, we have
$G_x(t_R,\xi_R)\sim\Qbar_x$, so $p$ is the rounding of
\cref{sec:det-rounding} applied to $\Qbar$.  By \cref{lem:budget} with $\delta_0=\delta/8$, with probability at least
$1-2\delta_0=1-\delta/4$, the variance budget is small enough that
the sample sizes of \cref{app:params} make
$C\sqrt{V_{\cP}D_{\mathrm{par}}}\le\varepsilon/4$, where $C$ is the constant of
\cref{thm:vc-rounding} and $D_{\mathrm{par}}=d+m+\log(8/\delta)$ is its
complexity term at failure probability $\delta/4$.  \Cref{thm:vc-rounding} with failure
probability $\delta/4$ then gives, since $K\ge16/\varepsilon$,
\[
  \MC(p)\le\underbrace{\MC(\Qbar)}_{\le\varepsilon/2}
  +\underbrace{C\sqrt{V_{\cP}D_{\mathrm{par}}}}_{\le\varepsilon/4}
  +\underbrace{2/K}_{\le\varepsilon/8}
  \le\varepsilon.
\]
The total failure probability, from coverage, training, the variance
budget, and rounding, is at most $\delta/8+\delta/4+\delta/4+\delta/4<\delta$.
The sample sizes are $T=\wtO((d+\log(1/\delta))\varepsilon^{-3})$ for
training, and
$\ncf,\npt=\wtO(\varepsilon^{-3}+(d+\log(1/\delta))\varepsilon^{-2})$ for
the intervals and the partition, which total
$\wtO((d+\log(1/\delta))\varepsilon^{-3})$.  ERM calls are made only in
training, $T(|\Shad|+L)=\wtO((d+\log(1/\delta))\varepsilon^{-6})$ of them.

\emph{Part~(b).}  Let $k\ge1$ be such that $d\le\varepsilon^{-k}$ and the
learner makes at most $(d/\varepsilon)^k$ oracle calls.  The learner is then
a $\lfloor(d/\varepsilon)^k\rfloor$-query black-box ERM learner
(\cref{def:black-box}), and \cref{cor:erm-lb} shows that it uses more than
$c_k\,d/(\varepsilon^3\log^5(1/\varepsilon))=\widetilde\Omega(d/\varepsilon^3)$
samples once $\varepsilon<c_k$ and $d\ge\varepsilon^{-1}\log(1/\varepsilon)/c_k$.
\end{proof}

\begin{remark}[Computational model and randomness]
\label{rem:model}
In the oracle-call counts, the Gaussian follow-the-perturbed-leader learner
of \cref{thm:hybrid} draws each hypothesis with one weighted ERM call on at
most $T$ weighted points.  The deterministic predictor is a fixed function
after training and needs no randomness or oracle calls at prediction time;
its evaluation cost is given in \cref{rem:representation}.  The training
algorithm is randomized.
\end{remark}

\begin{remark}[Representation of the output]
\label{rem:representation}
Every stored hypothesis is an oracle output and hence a member of $\cH$,
also with an approximate oracle.  The randomized predictor $\Qbar$ is
represented by the $TL$ retained hypotheses and their tests, with the
evaluation costs of \cref{sec:batch}: $L$ hypothesis evaluations per draw
and $TL=\wtO((d+\log(1/\delta))\varepsilon^{-6})$ for the distribution
$\Qbar_x$.  The deterministic predictor is represented by the hypotheses
and tests of the at most $\min\{T,|\cP|\}$ rounds $t_R$, the seeds, the sorted
partition sample, and the set $O$ with its intervals; the last two consist
of $\ncf+\npt=\wtO(\varepsilon^{-3}+(d+\log(1/\delta))\varepsilon^{-2})$
sample inputs.  To evaluate it at an input, one locates the cell of the
input by a lookup in $O$ and a binary search among the sorted partition
inputs, and then computes $Q_{t_R}(x)$; this takes $O(\log(\ncf+\npt))$
order comparisons and equality tests, $L$ hypothesis evaluations, and $O(mL)$ other
operations.  Its representation is polynomial in $d$, $1/\varepsilon$, and
$\log(1/\delta)$, but, unlike that of $\Qbar$, it uses the order on $\cX$ and
stored inputs in addition to members of $\cH$.
\end{remark}

\section{Discussion}
\label{sec:discussion}

Ordinary multicalibration asks for one auditor that sees all prediction
values, with a sign at each value.  Our algorithm keeps the search for the
auditor inside ordinary ERM by fixing the signs in advance, and replaces the
exponentially many sign patterns by the rows of a Hadamard matrix.  Hadamard tests lose a factor $\sqrt m=\Theta(\varepsilon^{-1/2})$, and
because residuals are measured at the mean of each forecast, this factor
multiplies only the regret term, not the grid term $1/(2K)$.  The resulting algorithm has the optimal $\varepsilon^{-3}$ dependence on
the accuracy.  Its sample complexity $\wtO(d\varepsilon^{-3})$ is also
optimal, up to polylogarithmic factors, among black-box ERM learners that
make polynomially many ERM calls, when $d$ is between about $1/\varepsilon$
and a polynomial in $1/\varepsilon$.  For swap multicalibration, $\widetilde\Theta(d\varepsilon^{-3})$ samples are
necessary and sufficient with unlimited computation (\cref{thm:swap}(a)),
and our oracle-efficient algorithm uses $\wtO(d\varepsilon^{-4})$
(\cref{thm:swap}(b)).  For
estimating the error of a given predictor, the computational picture is
reversed: the swap error can be estimated with $2m$ ERM calls, whereas the ordinary error
cannot be approximated efficiently unless $\mathrm{NP}\subseteq\mathrm{BPP}$.

Several questions remain.  First, our algorithm makes $\wtO(d\varepsilon^{-6})$
ERM calls, and its deterministic predictor is produced by randomized
training.  Can the number of ERM calls be reduced, and can a deterministic
predictor for VC classes be obtained oracle-efficiently by a deterministic
training algorithm, as \citet{noarov2026optimal} do for finite classes?  Second, the lower bound leaves two questions open.
When $d\ll1/\varepsilon$, our upper bound $\wtO(d\varepsilon^{-3})$ and the
lower bound $\widetilde\Omega(\varepsilon^{-3})$ of \citet{collina2026sample}
differ by a factor of up to $d$.  Moreover, the hard classes of \cref{thm:erm-lb} depend on hidden
parameters, so it remains open whether the additive rate
$\wtO(\varepsilon^{-3}+d\varepsilon^{-2})$ is attainable with
polynomially many oracle calls for explicitly represented classes, or with
access beyond an ERM oracle.  Third, our deterministic predictor uses a computable order on $\cX$
(\cref{rem:order}) and stored sample inputs (\cref{rem:representation}).  Can a deterministic predictor
learned from $\wtO(d\varepsilon^{-3})$ samples be represented as a
polynomial-size function of members of $\cH$ alone?  Fourth, can swap
multicalibration be achieved oracle-efficiently with
$\wtO(d\varepsilon^{-3})$ samples?  \Cref{rem:quartic}
explains why our approach loses a factor $1/\varepsilon$.  Finally, our
error-estimation bounds leave gaps: for the swap error, between
$\wtO(md/\varepsilon^2)$ and $\Omega(\sqrt{md}/\varepsilon^2)$ samples; for
the ordinary error, between $\wtO((d+m)/\varepsilon^2)$ and
$\Omega(\sqrt d/\varepsilon^2)$ samples, and, for efficient approximation, between the factor $\sqrt m$ achieved by
Hadamard tests and the factor $1\pm c_0$ that \cref{thm:np-hard} shows to be
NP-hard.

\paragraph{Use of AI tools.}
The author came across the results of \citet{buzaglo2026oracle} and
recognized that they imply improved oracle-efficient sample bounds for
multicalibration.  ChatGPT Astra~6 was used to refine the argument and
tighten the bounds.  Claude (Anthropic) was used to assist with drafting,
editing, and restructuring the text, and in developing and checking parts
of the arguments, including the high-probability regret bound of
\cref{thm:hybrid} and the lower bound of \cref{thm:erm-lb}.  The author has
verified all results and proofs and takes full responsibility for the
content of the paper.


\appendix

\section{High-probability regret of Gaussian follow-the-perturbed-leader}
\label{app:hybrid}

This appendix proves \cref{thm:hybrid}.  The algorithm, the potential
argument, and the two disagreement estimates of \cref{lem:emp-dis,lem:pop-dis}
come from the expected-regret analysis of \citet{buzaglo2026oracle}.  The new
ingredients are \cref{lem:approx-stability}, which allows an approximate
oracle, \cref{lem:cond-dis}, which bounds the disagreement mass
conditional on the realized history, and the martingale argument in
\cref{app:hybrid-proof}.  \Cref{app:hybrid-stability} reduces the regret to a
sum of stability terms, \cref{app:hybrid-dis} bounds their conditional means,
and \cref{app:hybrid-proof} completes the proof.

\paragraph{Conventions.}
Fix $\cG$ with VC dimension at most $d=d_{\cG}\ge1$ and a horizon $T$.  If
$T=1$, the regret is at most $2$, which is at most $\Reg(1,d,\zeta)$ once
$C_0\ge5$; so assume $T\ge2$.  Let $\gamma=4\sqrt{\log(2T)}$, so that
$\gamma\ge\sqrt{2\pi}$ and $\gamma\ge2$.  The lemmas below are stated for any
$\gamma$ satisfying their lower bounds and are applied with this value.  We write $\gs$ for a standard
Gaussian variable or vector, always independent of everything else, and
$\Pi_{\cG}(n)$ for the growth function of $\cG$ on $n$ points.  By the
Sauer--Shelah lemma \citep{sauer1972density},
$\log\Pi_{\cG}(n)\le d\log(en)$ for every $n\ge1$.  Let
$\Upsilon_n=\log(2\Pi_{\cG}(n))$, which is nondecreasing in $n$.  The oracle
is $\xi$-approximate for some $\xi\in[0,1]$ (\cref{def:erm}); an exact
oracle has $\xi=0$.  We write $\xi'=\xi/2\le1/2$ for the resulting slack in
the scores $\Psi$ defined below.

\paragraph{Deterministic environments.}
Each input $X_t$ is independent of everything before it, so the
environment's additional randomness and the learner's draws can be generated
from random seeds that are independent of the inputs.  The regret in
\cref{thm:hybrid} depends on the learner only through the distributions
$\pi_t$, which are determined by $(X_s,Z_s)_{s<t}$, and not through its
realized draws.  We condition on all seeds.  The inputs remain independent
draws from $\cD_X$, and each feedback value $Z_t$ becomes a fixed function of
$X_1,\ldots,X_t$.  It suffices to prove the bound with probability at least
$1-\zeta$ for every such deterministic environment.  Let $\cF_n$ be the
$\sigma$-field generated by $X_1,\ldots,X_n$.

\subsection{Reduction to stability terms}
\label{app:hybrid-stability}

\paragraph{Binary reformulation.}
Map each $g\in\cG$ to $h_g=(1+g)/2$, and let
$\cH_{\cG}=\{h_g:g\in\cG\}\subseteq\{0,1\}^{\cX}$, which has the same growth
function as $\cG$.  On round $t$, define the loss
$\ell_t(a)=(1-(2a-1)Z_t)/2\in[0,1]$ of the label $a\in\{0,1\}$, so that
$\ell_t(h_g(x))=(1-g(x)Z_t)/2$.  The mixed loss of the learner is
\[
  \bar\ell_t=\E_{g\sim\pi_t}\bigl[\ell_t(h_g(X_t))\bigr]=\frac{1-\bh_t(X_t)Z_t}{2},
\]
and therefore
\begin{equation}
  \label{eq:cls-reg}
  \sup_{g\in\cG}\sum_{t=1}^TZ_t\bigl(g(X_t)-\bh_t(X_t)\bigr)=2\Rcls_T,
  \qquad
  \Rcls_T=\sum_{t=1}^T\bar\ell_t-\min_{h\in\cH_{\cG}}\sum_{t=1}^T\ell_t(h(X_t)).
\end{equation}
The minimum exists because only the finitely many labelings of
$X_1,\ldots,X_T$ matter.  Since
$\ell_s(h_g(X_s))-\gamma\gs_{t,s}h_g(X_s)
=(1-\gamma\gs_{t,s})/2-(Z_s+\gamma\gs_{t,s})g(X_s)/2$, the rule
\eqref{eq:ftpl} with an exact oracle is equivalent to
\begin{equation}
  \label{eq:ftpl-binary}
  h_{g_t}\in\argmax_{h\in\cH_{\cG}}\ \Bigl\{-\sum_{s<t}\ell_s(h(X_s))
    +\gamma\sum_{s<t}\gs_{t,s}h(X_s)\Bigr\}.
\end{equation}

\paragraph{Scores, near leaders, and the potential.}
For $n\ge0$ let $\cL_n(h)=\sum_{s\le n}\ell_s(h(X_s))$.  For a vector
$\mathsf g\in\R^n$, define the perturbed score and, for $r\ge0$, the set of
nearly leading hypotheses
\[
  \Psi_{n,\mathsf g}(h)=-\cL_n(h)+\gamma\sum_{s\le n}\mathsf g_sh(X_s),
  \qquad
  \cN_{n,r}(\mathsf g)=\Bigl\{h\in\cH_{\cG}:
    \Psi_{n,\mathsf g}(h)\ge\max_{h'\in\cH_{\cG}}\Psi_{n,\mathsf g}(h')-r\Bigr\}.
\]
The sets $\cN_{n,r}(\mathsf g)$ grow with $r$.  The disagreement region of a
set $V$ of hypotheses is
$\DIS(V)=\{x\in\cX:h(x)\ne h'(x)\text{ for some }h,h'\in V\}$.  The
\emph{potential} is $\Phi_n=\E_{\gs}\max_h\Psi_{n,\gs}(h)$ with
$\gs\sim N(0,I_n)$, and $\Phi_0=0$.  Since
$\Psi_{t-1,\mathsf g}(h_g)=\sum_{s<t}(\gamma\mathsf g_s-1)/2
+\frac12\sum_{s<t}(Z_s+\gamma\mathsf g_s)g(X_s)$, a $\xi$-approximate oracle
call on the weights $Z_s+\gamma\mathsf g_s$ returns a hypothesis $h_g$ that
maximizes $\Psi_{t-1,\mathsf g}$ up to $\xi'$.

\begin{lemma}[Binary stability]
\label{lem:binary-stability}
Let $\Psi:\cH_{\cG}\to\R$ take finitely many values, let $x\in\cX$, let
$\ell:\{0,1\}\to[0,1]$, and let $\gamma\ge\sqrt{2\pi}$.  Let $h^*$ be any
maximizer of $\Psi$, and let
$\cN_r(\Psi)=\{h:\Psi(h)\ge\max\Psi-r\}$.  Then
\begin{multline*}
  \ell(h^*(x))+\E_{\gs}\max_{h}\bigl\{\Psi(h)-\ell(h(x))+\gamma\gs h(x)\bigr\}
  -\max_h\Psi(h)\\
  \le2\int_0^\infty\Pr\{\gs\ge s/\gamma\}\,
    \mathbf 1\{x\in\DIS(\cN_{s+1}(\Psi))\}\,ds.
\end{multline*}
\end{lemma}

\begin{proof}
If all hypotheses agree at $x$, the left side is zero.  Otherwise let
$a=h^*(x)$ and $\varsigma=\Psi(h^*)-\max_{h:h(x)=1-a}\Psi(h)\ge0$.  The
inner maximum is the larger of the best scores among hypotheses with label
$a$ and with label $1-a$ at $x$, so the left side equals
$\E\max\{\gamma\gs a,\ \gamma\gs(1-a)-\varsigma+\ell(a)-\ell(1-a)\}$.  For
$a=0$ this is $\E(\gamma\gs-\varsigma+\ell(a)-\ell(1-a))_+$, and for $a=1$
it equals the same expression with $\gs$ replaced by $-\gs$, which has the
same law.  Since $\ell(a)-\ell(1-a)\le1$ and
$(\gamma\gs-\varsigma+1)_+\le(\gamma\gs-(\varsigma-1)_+)_++(1-\varsigma)_+$,
the left side is at most
\[
  \int_{(\varsigma-1)_+}^\infty\Pr\{\gs\ge s/\gamma\}\,ds+(1-\varsigma)_+.
\]
If $\varsigma<1$, then $(1-\varsigma)_+\le1\le\gamma/\sqrt{2\pi}
=\int_0^\infty\Pr\{\gs\ge s/\gamma\}\,ds$, and the integral starts at $0$;
if $\varsigma\ge1$, the last term vanishes.  In both cases the left side is
at most $2\int_{(\varsigma-1)_+}^\infty\Pr\{\gs\ge s/\gamma\}\,ds$.
Finally, $h^*\in\cN_r(\Psi)$, so $x\in\DIS(\cN_r(\Psi))$ if and only if
some hypothesis with label $1-a$ at $x$ lies in $\cN_r(\Psi)$, that is, if
and only if $\varsigma\le r$.  Hence $x\in\DIS(\cN_{s+1}(\Psi))$ if and only
if $s\ge\varsigma-1$.
\end{proof}

An approximate maximizer costs only a larger radius.

\begin{lemma}[Approximate stability]
\label{lem:approx-stability}
Under the conditions of \cref{lem:binary-stability}, let $\xi'\ge0$, and let
$\widehat h$ satisfy $\Psi(\widehat h)\ge\max\Psi-\xi'$.  Then the bound of
\cref{lem:binary-stability} holds with $\ell(\widehat h(x))$ in place of
$\ell(h^*(x))$, with the factor $3$ in place of $2$, and with
$\cN_{s+1+\xi'}(\Psi)$ in place of $\cN_{s+1}(\Psi)$.
\end{lemma}

\begin{proof}
Since $\ell$ takes values in $[0,1]$ and $\widehat h,h^*\in\cN_{\xi'}(\Psi)$,
we have $\ell(\widehat h(x))-\ell(h^*(x))\le\mathbf 1\{x\in\DIS(\cN_{\xi'}(\Psi))\}$.
Since $\int_0^\infty\Pr\{\gs\ge s/\gamma\}\,ds=\gamma/\sqrt{2\pi}\ge1$ and the
sets $\cN_r(\Psi)$ grow with $r$, this indicator is at most
$\int_0^\infty\Pr\{\gs\ge s/\gamma\}\,
\mathbf 1\{x\in\DIS(\cN_{s+1+\xi'}(\Psi))\}\,ds$.  Adding this to the bound
of \cref{lem:binary-stability}, whose right side only grows when $s+1$ is
replaced by $s+1+\xi'$, proves the claim.
\end{proof}

\begin{lemma}[Reduction to stability terms]
\label{lem:stability-sum}
For $2\le t\le T$ let
\[
    \Theta_t=3\int_0^\infty\Pr\{\gs\ge s/\gamma\}\,
  \Pr_{\gs'}\bigl\{X_t\in\DIS(\cN_{t-1,s+1+\xi'}(\gs'))\bigr\}\,ds,
  \qquad\gs'\sim N(0,I_{t-1}).
\]
Then $0\le\Theta_t\le2\gamma$, and $\Rcls_T\le2\gamma+\sum_{t=2}^T\Theta_t$.
\end{lemma}

\begin{proof}
If $h^\circ$ minimizes $\cL_T$, then
$\Phi_T\ge\E\Psi_{T,\gs}(h^\circ)=-\cL_T(h^\circ)$.  Since $\Phi_0=0$,
\[
  \Rcls_T\le\sum_{t=1}^T\bigl(\bar\ell_t+\Phi_t-\Phi_{t-1}\bigr).
\]
Fix $t$.  For $\mathsf g\in\R^{t-1}$, let $h^*_{\mathsf g}$ be the
hypothesis returned by the oracle on the perturbation $\mathsf g$; it
maximizes $\Psi_{t-1,\mathsf g}$ up to $\xi'$, and
$\bar\ell_t=\E_{\gs'}\ell_t(h^*_{\gs'}(X_t))$.  (If the oracle breaks ties at
random, the expectation also covers its randomness; the argument below holds
for every hypothesis within $\xi'$ of the maximum of $\Psi_{t-1,\mathsf g}$.)  Since
$\Psi_{t,(\mathsf g,\mathsf g_t)}(h)=\Psi_{t-1,\mathsf g}(h)-\ell_t(h(X_t))
+\gamma\mathsf g_th(X_t)$,
\[
  \bar\ell_t+\Phi_t-\Phi_{t-1}
  =\E_{\gs'}\Bigl[\ell_t(h^*_{\gs'}(X_t))
   +\E_{\gs}\max_h\bigl\{\Psi_{t-1,\gs'}(h)-\ell_t(h(X_t))+\gamma\gs h(X_t)\bigr\}
   -\max_h\Psi_{t-1,\gs'}(h)\Bigr].
\]
\Cref{lem:approx-stability}, applied inside the expectation with
$\Psi=\Psi_{t-1,\gs'}$, $x=X_t$, $\ell=\ell_t$, and
$\widehat h=h^*_{\gs'}$, and Fubini's theorem
bound this by $\Theta_t$ for $t\ge2$, and by
$3\int_0^\infty\Pr\{\gs\ge s/\gamma\}\,ds=3\gamma/\sqrt{2\pi}<2\gamma$ for
$t=1$.  The same computation shows $0\le\Theta_t\le3\gamma/\sqrt{2\pi}<2\gamma$.
\end{proof}

\subsection{Disagreement estimates}
\label{app:hybrid-dis}

For $1\le n<T$, $r\ge0$, and an auxiliary vector $\gs\sim N(0,I_n)$, let
\[
  \widehat\dis_{n,r}=\sum_{i=1}^n\mathbf 1\bigl\{X_i\in\DIS(\cN_{n,r}(\gs))\bigr\},
  \qquad
  \dis_{n,r}=\cD_X\bigl(\DIS(\cN_{n,r}(\gs))\bigr),
  \qquad
  \overline\dis_{n,r}=\E\bigl[\dis_{n,r}\bigm|\cF_n\bigr].
\]
The first is the number of observed inputs at which nearly leading
hypotheses disagree, the second is the population mass of their disagreement
region, and the third averages the second over the perturbation with the
inputs held fixed.  The following two estimates are due to
\citet{buzaglo2026oracle}.  The first holds for every fixed sequence of
inputs and feedback values.

\begin{lemma}[Empirical disagreement, \citealp{buzaglo2026oracle}]
\label{lem:emp-dis}
For every realization of $X_1,\ldots,X_n$, every $r\ge0$, and every
$\gamma>0$,
\[
  \E\bigl[\widehat\dis_{n,r}\bigm|\cF_n\bigr]
  \le2\sqrt{n\Upsilon_n}\,\exp\Bigl(\frac{r^2}{2\gamma^2}\Bigr).
\]
\end{lemma}

\begin{lemma}[Population-disagreement tail, \citealp{buzaglo2026oracle}]
\label{lem:pop-dis}
Let $\gamma\ge2$ and $\varpi=\Pr\{\gs>1/\gamma\}$.  For every deterministic
environment, every $n<T$, every $r\ge0$, every $k\in\{0,\ldots,n\}$, and
every $\mathsf s\in(0,1)$,
\[
  \Pr\bigl\{\widehat\dis_{n,r}=k,\ \dis_{n,r}>\mathsf s\bigr\}
  \le\binom nk\,\Pi_{\cG}(n)\,\exp\Bigl(\frac k{2\gamma^2}\Bigr)
  \bigl(1-\varpi+\varpi\sqrt{1-\mathsf s}\bigr)^{n-k},
\]
where the probability is over $X_1,\ldots,X_n$ and $\gs$.
\end{lemma}

\Cref{lem:pop-dis} says that the nearly leading hypotheses cannot disagree
on few observed inputs while their disagreement region has large mass:
otherwise many observed inputs would have avoided a set of large mass.  The
factors $\binom nk$, $\Pi_{\cG}(n)$, and $\exp(k/(2\gamma^2))$ pay for the
choice of the empirical disagreement set, for the leading labeling, and for a
change in the mean of the Gaussian perturbation that decouples the
disagreement region from the inputs used to test it.  Its probability is
joint over the inputs and the perturbation, whereas the stability terms of
\cref{lem:stability-sum} average over the perturbation with the history
fixed.  The next lemma bridges the two.

\begin{lemma}[Disagreement conditional on the history]
\label{lem:cond-dis}
Let $\gamma\ge2$, let $\Gamma\subset[0,\infty)$ be finite, let
$\zeta'\in(0,1)$, and let $\mathsf u=\log(T^6|\Gamma|/\zeta')$.  With
probability at least $1-\zeta'$ over the inputs, for every $1\le n<T$ and
every $r\in\Gamma$,
\begin{equation}
  \label{eq:cond-dis}
  \overline\dis_{n,r}\le C\Bigl[(1+\log n)\sqrt{\frac{\Upsilon_n}{n}}\,
  \exp\Bigl(\frac{r^2}{2\gamma^2}\Bigr)+\frac{\Upsilon_n+\mathsf u}{n}+T^{-4}\Bigr].
\end{equation}
\end{lemma}

\begin{proof}
Fix $n$ and $r$, and let $\varpi_0=\Pr\{\gs>1/2\}$, so that
$\varpi\ge\varpi_0>0$ since $\gamma\ge2$.  For $k\le n/2$, let
\[
  \mathsf s_k=\min\Bigl\{1,\ \frac4{\varpi n}\Bigl(\log\Pi_{\cG}(n)
    +\log\binom nk+\frac k{2\gamma^2}+\mathsf u\Bigr)\Bigr\}.
\]
If $\mathsf s_k=1$, the event $\dis_{n,r}>\mathsf s_k$ is empty.
Otherwise, since $1-\sqrt{1-\mathsf s}\ge\mathsf s/2$,
$\log(1-z)\le-z$, and $n-k\ge n/2$, \cref{lem:pop-dis} gives
\[
  \Pr\bigl\{\widehat\dis_{n,r}=k,\ \dis_{n,r}>\mathsf s_k\bigr\}
  \le\exp\Bigl(\log\Pi_{\cG}(n)+\log\binom nk+\frac k{2\gamma^2}
    -\frac{\varpi n\mathsf s_k}4\Bigr)\le e^{-\mathsf u}.
\]
Summing over the at most $T$ values of $k\le n/2$, the event
$\mathrm{Bad}_{n,r}=\{\widehat\dis_{n,r}\le n/2,\
\dis_{n,r}>\mathsf s_{\widehat\dis_{n,r}}\}$ has probability at most
$Te^{-\mathsf u}$.

Let $\mathsf b_{n,r}=\Pr\{\mathrm{Bad}_{n,r}\mid\cF_n\}$, the probability over the
perturbation alone.  Then $\E\mathsf b_{n,r}=\Pr(\mathrm{Bad}_{n,r})$, and Markov's
inequality gives $\Pr\{\mathsf b_{n,r}>T^{-4}\}\le T^5e^{-\mathsf u}$.  A union
bound over the fewer than $T|\Gamma|$ pairs $(n,r)$ shows that, with
probability at least $1-T^6|\Gamma|e^{-\mathsf u}=1-\zeta'$,
$\mathsf b_{n,r}\le T^{-4}$ for all of them.  Fix inputs for which this holds.
Since $0\le\dis_{n,r}\le1$, and $\dis_{n,r}\le\mathsf s_{\widehat\dis_{n,r}}$
whenever $\widehat\dis_{n,r}\le n/2$ outside $\mathrm{Bad}_{n,r}$,
\[
  \overline\dis_{n,r}
  \le\E\bigl[\mathsf s_{\widehat\dis_{n,r}}\mathbf 1\{\widehat\dis_{n,r}\le n/2\}
    \bigm|\cF_n\bigr]
  +\Pr\bigl\{\widehat\dis_{n,r}>n/2\bigm|\cF_n\bigr\}+T^{-4}.
\]
For $k\le n/2$ we have $\log\binom nk\le k\log(en)=k(1+\log n)$ and
$k/(2\gamma^2)\le k$, so
$\mathsf s_k\le\frac4{\varpi_0n}(\Upsilon_n+\mathsf u+2k(1+\log n))$.  By
Markov's inequality,
$\Pr\{\widehat\dis_{n,r}>n/2\mid\cF_n\}\le\frac2n\E[\widehat\dis_{n,r}\mid\cF_n]$.
Bounding $\E[\widehat\dis_{n,r}\mid\cF_n]$ by \cref{lem:emp-dis} in both terms
gives~\eqref{eq:cond-dis}.
\end{proof}

\subsection{Proof of \texorpdfstring{\cref{thm:hybrid}}{Theorem 2.5}}
\label{app:hybrid-proof}

\begin{proof}
Let $r_{\max}=\gamma^2/2=8\log(2T)$ and
$\Gamma=\{2+\xi',3+\xi',\ldots,\lceil r_{\max}\rceil+2+\xi'\}$, so that
$|\Gamma|\le r_{\max}+3\le12\log(2T)\le12T$.  Let $E_1$ be the event of
\cref{lem:cond-dis} with $\zeta'=\zeta/2$, for which
$\mathsf u=\log(2T^6|\Gamma|/\zeta)\le\log(24T^7/\zeta)$.

\emph{Conditional means of the stability terms.}  For $2\le t\le T$, the set
$\cN_{t-1,r}(\gs')$ depends only on $X_1,\ldots,X_{t-1}$ and $\gs'$, and
$X_t$ is a fresh draw from $\cD_X$.  By Fubini's theorem,
\[
  \vartheta_t:=\E[\Theta_t\mid\cF_{t-1}]
    =3\int_0^\infty\Pr\{\gs\ge s/\gamma\}\,\overline\dis_{t-1,s+1+\xi'}\,ds.
\]
Fix $n=t-1$ and work on $E_1$.  For $s>r_{\max}$ we use
$\overline\dis_{n,s+1+\xi'}\le1$ and the Gaussian tail bound
$\Pr\{\gs\ge z\}\le e^{-z^2/2}$:
\[
  \int_{r_{\max}}^\infty e^{-s^2/(2\gamma^2)}\,ds
  \le\frac{\gamma^2}{r_{\max}}e^{-r_{\max}^2/(2\gamma^2)}
  =2e^{-\gamma^2/8}=\frac1{2T^2}.
\]
For $s\in[0,r_{\max}]$, let $j=\lfloor s\rfloor$.  Then $s+1+\xi'\le j+2+\xi'\in\Gamma$,
so $\overline\dis_{n,s+1+\xi'}\le\overline\dis_{n,j+2+\xi'}$ by monotonicity in
the radius, and, since $\xi'\le1/2$ and
$\Pr\{\gs\ge s/\gamma\}\le e^{-s^2/(2\gamma^2)}$,
\[
  \Pr\{\gs\ge s/\gamma\}\exp\Bigl(\frac{(j+2+\xi')^2}{2\gamma^2}\Bigr)
  \le\exp\Bigl(\frac{(s+3)^2-s^2}{2\gamma^2}\Bigr)
  \le\exp\Bigl(\frac{6r_{\max}+9}{2\gamma^2}\Bigr)\le e^3.
\]
Applying~\eqref{eq:cond-dis} at radius $j+2+\xi'$, integrating over
$s\in[0,r_{\max}]$, and using
$\int_0^\infty\Pr\{\gs\ge s/\gamma\}\,ds=\gamma/\sqrt{2\pi}$ for the
remaining terms gives
\[
  \vartheta_{n+1}\le C\Bigl[\gamma^2(1+\log n)\sqrt{\frac{\Upsilon_n}{n}}
  +\gamma\,\frac{\Upsilon_n+\mathsf u}{n}+\gamma T^{-4}+T^{-2}\Bigr].
\]
Summing over $1\le n<T$ with $\Upsilon_n\le\Upsilon_T$,
$\sum_{n<T}(1+\log n)/\sqrt n\le C\sqrt T\log(2T)$,
$\sum_{n<T}1/n\le C\log(2T)$, and $\gamma=4\sqrt{\log(2T)}$, we obtain on
$E_1$
\begin{equation}
  \label{eq:pred-sum}
  \sum_{t=2}^T\vartheta_t\le\bar A,
  \qquad
  \bar A:=C_1\Bigl[\sqrt{T\Upsilon_T}\log^2(2T)
    +(\Upsilon_T+\mathsf u)\log^{3/2}(2T)\Bigr]
\end{equation}
for an absolute constant $C_1$, where we absorbed the last two terms using
$\mathsf u\ge1$.

\emph{From conditional means to realized values.}  We use Freedman's
inequality \citep{freedman1975tail}: if $(D_t)$ is a martingale difference
sequence with respect to $(\cF_t)$ and $D_t\le b$, then for all $\mathsf a,\mathsf v>0$,
\[
  \Pr\Bigl\{\sum_tD_t\ge\mathsf a\ \text{ and }\ \sum_t\E[D_t^2\mid\cF_{t-1}]\le\mathsf v\Bigr\}
  \le\exp\Bigl(-\frac{\mathsf a^2}{2(\mathsf v+b\mathsf a/3)}\Bigr).
\]
Let $D_t=\Theta_t-\vartheta_t$ for $2\le t\le T$; $\Theta_t$ is
$\cF_t$-measurable, and $D_t\le2\gamma$ by \cref{lem:stability-sum}.  Since
$0\le\Theta_t\le2\gamma$, we have
$\E[D_t^2\mid\cF_{t-1}]\le\E[\Theta_t^2\mid\cF_{t-1}]\le2\gamma\vartheta_t$,
so on $E_1$ the sum of conditional variances is at most $\mathsf v=2\gamma\bar A$.  Let
$\mathsf x=\log(2/\zeta)$ and $\mathsf a=\sqrt{2\mathsf v\mathsf x}+4\gamma\mathsf x/3$;
then $\mathsf a^2\ge2\mathsf v\mathsf x+(4\gamma\mathsf x/3)\mathsf a$, so the right side
above is at most $e^{-\mathsf x}=\zeta/2$.  It is essential here that $\bar A$ is deterministic: the
sum of conditional variances is random, and we control it only on $E_1$.
Hence, with probability at least $1-\zeta$, the event $E_1$ holds and
\[
    \sum_{t=2}^T\Theta_t\le\bar A+\sqrt{4\gamma\bar A\mathsf x}+\frac{4\gamma\mathsf x}3
  \le\frac32\bar A+4\gamma\mathsf x,
\]
using $\sqrt{4\gamma\bar A\mathsf x}\le\bar A/2+2\gamma\mathsf x$.

\emph{Conclusion.}  On this event, \cref{lem:stability-sum} gives
$\Rcls_T\le2\gamma+\frac32\bar A+4\gamma\log(2/\zeta)$.  Since
$\log(2T)\ge1$, we have $\gamma\le4\log^{3/2}(2T)$, and
$\log(2/\zeta)\le\mathsf u$, so the first and last terms are absorbed into
$\bar A$.  By the Sauer--Shelah bound,
$\Upsilon_T\le\log2+d\log(eT)\le2d\log(eT)$, and
$\mathsf u\le\log24+7\log T+\log(1/\zeta)\le11d\log(eT)+\log(1/\zeta)$.
Writing $M_\zeta=d\log(eT)+\log(1/\zeta)$, we obtain
\[
  \Rcls_T\le C\Bigl[\sqrt{TM_\zeta}\,\log^2(2T)+M_\zeta\log^{3/2}(2T)\Bigr].
\]
By~\eqref{eq:cls-reg}, the correlation regret is $2\Rcls_T$.  If
$M_\zeta\le T$, then $M_\zeta\le\sqrt{TM_\zeta}$, and the correlation regret
is at most $C\sqrt{TM_\zeta}\log^2(2T)$.  If $M_\zeta>T$, the correlation
regret is at most $2T<2\sqrt{TM_\zeta}$, since $|Z_t|\le1$.  In both cases it
is at most $\Reg(T,d,\zeta)$ for a suitable absolute constant $C_0$.  The
oracle cost was shown after~\eqref{eq:ftpl}.
\end{proof}

\section{Proofs of the online and batch guarantees}
\label{app:mc}

This appendix proves the online guarantee (\cref{thm:online}) in
\cref{app:reduction}, and the martingale lemma (\cref{lem:martingale}) and
the batch guarantee (\cref{thm:batch}) in \cref{app:batch}.  The other
deferred proofs for \cref{sec:prelims,sec:mc} are in \cref{app:hybrid}
(\cref{thm:hybrid}) and \cref{app:erm-lb} (\cref{thm:erm-lb}).

\subsection{The online guarantee}
\label{app:reduction}

We use two standard tools.  Exponential weights with rewards in $[-1,1]$
and learning rate $\eta=\min\{1,\sqrt{\log N/T}\}$ satisfy, for every
adaptively chosen reward sequence,
\begin{equation}
  \label{eq:hedge}
  \max_{\sigma\in\cS}\sum_{t=1}^T\hat s_{t,\sigma}
  \le\sum_{t=1}^T\sum_{\sigma\in\cS}w_t(\sigma)\hat s_{t,\sigma}
  +C\sqrt{T\log N}
\end{equation}
\citep{freund1997decision,cesa2006prediction}.  Azuma's inequality
\citep{azuma1967weighted} bounds sums of martingale differences bounded by
$c$ by $c\sqrt{2T\log(2/\zeta)}$ with probability at least $1-\zeta$.

\begin{proof}[Proof of \cref{thm:online}]
Let $\cF_t$ denote everything that has happened before the draws
$h_{t,\sigma}$ of round $t$: the past, the auxiliary draws of round $t$, and
$X_t$ and $Y_t$.  Let $s_{t,\sigma}=\bh_{t,\sigma}(X_t)Z_{t,\sigma}$.

\emph{Estimation.}  Conditional on the past and the auxiliary draws, $X_t$
is a fresh draw from $\cD_X$; conditional on the past and $X_t$, each
coordinate of $\widehat A_t(X_t)$ is an average of $L$ independent terms in
$[-1,1]$ with mean $A_{t,v}(X_t)$, because $\E[\sigma_vh(X_t)]$ with
$\sigma\sim w_t$ and $h\sim\pi_{t,\sigma}$ is
$\sum_\sigma w_t(\sigma)\sigma_v\bh_{t,\sigma}(X_t)$.  Hoeffding's
inequality and a union bound over $t\le T$ and $v\in\Lambda$ show that,
with probability at least $1-\delta/8$,
\begin{equation}
  \label{eq:approx}
  \max_{t\le T,\,v\in\Lambda}\bigl|\widehat A_{t,v}(X_t)-A_{t,v}(X_t)\bigr|
  \le\alpha.
\end{equation}

\emph{Current-round cancellation.}  Write $Q_t=Q_t(X_t)$ and $\mu_t=\mu_t(X_t)$.
On the event~\eqref{eq:approx},
\begin{align*}
  \sum_{\sigma}w_t(\sigma)s_{t,\sigma}
  &=(\mu_t-Y_t)\sum_vQ_t(v)\sum_\sigma w_t(\sigma)\sigma_v\bh_{t,\sigma}(X_t)
   =(\mu_t-Y_t)\sum_vQ_t(v)A_{t,v}(X_t)\\
  &\le(\mu_t-Y_t)\sum_vQ_t(v)\widehat A_{t,v}(X_t)+\alpha
   \le\alpha,
\end{align*}
where the last step is \cref{lem:root}.

\emph{Exponential weights.}  Given $\cF_t$, the draw $h_{t,\sigma}$ has
distribution $\pi_{t,\sigma}$, so $\hat s_{t,\sigma}-s_{t,\sigma}$ has
conditional mean zero and absolute value at most two.  By Azuma's
inequality and a union bound, with probability at least $1-\delta/4$,
\[
  \max_{\sigma}\Bigl|\sum_t(\hat s_{t,\sigma}-s_{t,\sigma})\Bigr|
  \le C\sqrt{T\log\frac{4N}{\delta}},
  \qquad
  \Bigl|\sum_t\sum_\sigma w_t(\sigma)(\hat s_{t,\sigma}-s_{t,\sigma})\Bigr|
  \le C\sqrt{T\log\frac{4}{\delta}}.
\]
Combining these with~\eqref{eq:hedge} and the cancellation gives
\begin{equation}
  \label{eq:mixed-max}
  \max_{\sigma\in\cS}\sum_ts_{t,\sigma}\le T\alpha+C\sqrt{T\log\frac{4N}{\delta}}.
\end{equation}

\emph{Hybrid learners.}  The feedback $Z_{t,\sigma}$ is a function of the past, the auxiliary draws,
$X_t$, and $Y_t$, none of which depends on the draw $h_{t,\sigma}$; so \cref{thm:hybrid} applies to each $\cA_\sigma$.
With failure probability $\delta/(4N)$ for each test, with probability at
least $1-\delta/4$, for all $\sigma$ simultaneously,
\[
  \sup_{h\in\cH}\sum_th(X_t)Z_{t,\sigma}
  \le\sum_ts_{t,\sigma}+\Reg\Bigl(T,d,\frac{\delta}{4N}\Bigr).
\]

\emph{Conversion.}  For $h\in\cH$, let $\phi^h\in\R^{\Lambda}$ be the
cumulative virtual residual vector,
$\phi^h_v=\sum_th(X_t)Q_t(v)(\mu_t-Y_t)$.  Then
$\langle\sigma,\phi^h\rangle=\sum_th(X_t)Z_{t,\sigma}$, so by the
distortion of $\cS$ and the last two displays,
\[
  \|\phi^h\|_1\le\kappa\max_{\sigma}\langle\sigma,\phi^h\rangle
  \le\kappa\Bigl[\Reg\Bigl(T,d,\frac{\delta}{4N}\Bigr)
     +C\sqrt{T\log\frac{4N}{\delta}}+T\alpha\Bigr].
\]
By \cref{lem:adjacent}, the cumulative actual residual vector of $h$ differs
from $\phi^h$ by at most $T/(2K)$ in $\ell_1$.  Taking the supremum over $h$
proves the bound; the failure probabilities sum to at most $\delta$.

\emph{Cost.}  Each auxiliary draw and each draw $h_{t,\sigma}$ is one ERM
call, $N+L$ in total per round.  Forming $\widehat A_t(X_t)$ takes $O(mL)$
operations, and forming the feedback, rewards, and weights takes $O(mN)$.
\end{proof}

\subsection{The martingale lemma and the batch guarantee}
\label{app:batch}

\begin{proof}[Proof of \cref{lem:martingale}]
(a) Assume $T\ge2$; for $T=1$ the bound holds after enlarging $C$.  Let
$\cH_{1/T}\subseteq\cH$ be a $1/T$-cover of $\cH$ with respect to $\cD_X$,
with $\log|\cH_{1/T}|\le Cd\log(eT)$.  The cover depends only on $\cD_X$ and
is used only in the analysis.  For fixed $h'\in\cH_{1/T}$ and $i\in\cI$, the
terms $h'(X_t)z_{t,i}(X_t,Y_t)-\E_\cD[h'(X)z_{t,i}(X,Y)]$ are martingale
differences bounded by two, because $z_{t,i}$ is predictable and
$(X_t,Y_t)$ is a fresh draw.  Azuma's inequality and a union bound over
$\cH_{1/T}\times\cI$ give the bound for all $h'$ in the cover, with failure
probability $\zeta/2$.

For arbitrary $h$, choose $h'\in\cH_{1/T}$ with $\cD_X(h\ne h')\le1/T$.
Since $|z_{t,i}|\le1$, the two martingale sums differ by at most
$2\sum_t\mathbf 1\{h(X_t)\ne h'(X_t)\}+2T\cD_X(h\ne h')$.  The second term is
at most two.  The sets $\{h\ne h'\}$ with $h,h'\in\cH$ have growth function
at most the square of that of $\cH$, so by uniform convergence for these
sets, with probability at least $1-\zeta/2$, for all pairs simultaneously,
$\sum_t\mathbf 1\{h(X_t)\ne h'(X_t)\}\le T\cD_X(h\ne h')+
C\sqrt{T(d\log(eT)+\log(2/\zeta))}$.  This step uses only the independence
of the inputs, not any property of the multipliers.  Combining the two
events and adjusting $C$ proves~(a).

(b) Let $\cG_\star$ be a $1/(mT)$-cover of $\Hpm$ with respect to $\cD_X$,
with $\log|\cG_\star|\le Cd\log(emT)$.  For a fixed tuple in
$\cG_\star^{\Lambda}$, the terms are martingale differences bounded by two,
since $\sum_v|z_{t,v}|\le1$; Azuma's inequality and a union bound over the
$|\cG_\star|^m$ tuples give the bound for tuples in the cover.  For an
arbitrary tuple $(h_v)$ with representatives $(h'_v)$, let $U(x)$ indicate
that $h_v(x)\ne h'_v(x)$ for some $v$; then $\cD_X(U)\le1/T$, and the two
martingale sums differ by at most $2\sum_tU(X_t)+2T\cD_X(U)$, since
$\sum_v|h_v-h'_v||z_{t,v}|\le2U$.  Unions of $m$ disagreement sets have
growth function at most the $2m$-th power of that of $\Hpm$, so uniform
convergence bounds this difference by
$O(1+\sqrt{T(md\log(emT)+\log(1/\zeta))})$ for all tuples simultaneously.
\end{proof}

\begin{proof}[Proof of \cref{thm:batch}]
We follow the proof of \cref{thm:online}, in two steps.  First, the
estimation event~\eqref{eq:approx} again has probability at least
$1-\delta/8$, because the auxiliary draws precede $X_t$, and on it the
current-round cancellation
$\sum_\sigma w_t(\sigma)s_{t,\sigma}\le\alpha$ holds as before.  Combining with
the exponential weights step gives
$\max_\sigma\sum_ts_{t,\sigma}\le T\alpha+C\sqrt{T\log(16N/\delta)}$, and
\cref{thm:hybrid} with failure probability $\delta/(8N)$ per test gives,
with probability at least $1-\delta/2$ overall (failure probabilities
$\delta/8$ each for estimation, exponential weights, and the hybrid
learners, so $3\delta/8\le\delta/2$),
\begin{equation}
  \label{eq:restricted}
  \sup_{h\in\cH,\sigma\in\cS}\sum_th(X_t)Z_{t,\sigma}
  \le T\alpha+C\sqrt{T\log\frac{16N}{\delta}}
     +\Reg\Bigl(T,d,\frac{\delta}{8N}\Bigr).
\end{equation}

Second, we pass to the population.  The maps
$z_{t,\sigma}(x,y)=(\mu_t(x)-y)\sum_v\sigma_vQ_t(x)(v)$ are predictable,
since $Q_t$ is determined by the past and the auxiliary draws, and take
values in $[-1,1]$; also $Z_{t,\sigma}=z_{t,\sigma}(X_t,Y_t)$.  By
\cref{lem:martingale}(a) with $\cI=\cS$ and $\zeta=\delta/2$, and
by~\eqref{eq:restricted}, with probability at least $1-\delta$, for all
$h$ and $\sigma$,
\[
  \frac1T\sum_t\E_\cD[h(X)z_{t,\sigma}(X,Y)]
  \le\frac1T\Reg\Bigl(T,d,\frac{\delta}{8N}\Bigr)
     +C\sqrt{\frac{d\log(eT)+\log(8N/\delta)}{T}}+\alpha.
\]
For $h\in\cH$ let $\bar\phi^h\in\R^{\Lambda}$ have coordinates
$\bar\phi^h_v=\frac1T\sum_t\E_\cD[h(X)Q_t(X)(v)(\mu_t(X)-Y)]$, so that the left
side is $\langle\sigma,\bar\phi^h\rangle$.  By the distortion of $\cS$,
$\|\bar\phi^h\|_1$ is at most $\kappa$ times the right side.  Finally,
$B_h(\Qbar,v)=\frac1T\sum_t\E_\cD[h(X)Q_t(X)(v)(v-Y)]$, and
\cref{lem:adjacent} inside the expectation gives
$\sum_v|B_h(\Qbar,v)-\bar\phi^h_v|\le1/(2K)$.  Taking the supremum over $h$
completes the proof.
\end{proof}

\section{Proof of the black-box lower bound}
\label{app:erm-lb}

This appendix proves \cref{thm:erm-lb}.  We use the parameters of
\cref{sec:erm-lb} and assume~\eqref{eq:erm-lb-cond}.  Inputs are written
$x=(j,i)\in[J]\times[D]$, and $j(x)=j$ is the block of $x$.  The input
distribution is uniform, the block centers are $c_j=1/4+j\Delta_c$, which
lie in $(1/4,1/2]$ since $J\Delta_c\le1/4$, and for signs
$\tau\in\{-1,+1\}^{\cX}$ the mean outcome is
$\fs(x)=c_{j(x)}+\beta\tau_x$.  For a finitely supported randomized predictor
$Q$ with value set $\cV$, its \emph{displacement} is
\[
  \mathrm{Disp}(Q)=\E_X\Bigl[\sum_{v\in\cV}Q_X(v)\,|v-c_{j(X)}|\Bigr].
\]
\Cref{app:erm-lb-anchor} shows that anchors control the displacement,
\cref{app:erm-lb-code} constructs the code, \cref{app:erm-lb-instance}
defines the random instance and shows that queries rarely reveal a planted
auditor, \cref{app:erm-lb-error} shows that a predictor that has not found a
planted auditor has large error, and \cref{app:erm-lb-proof} combines these.

\subsection{Anchors}
\label{app:erm-lb-anchor}

The anchors are the block-constant sign functions
$\cH_{\mathrm{anc}}=\{x\mapsto a_{j(x)}:a\in\{-1,+1\}^J\}$.

\begin{lemma}[Anchor control]
\label{lem:anchor}
For every $\tau$ and every finitely supported randomized predictor $Q$,
\[
  \mathrm{Disp}(Q)\le\kappa_{\mathrm{anc}}\bigl(\MC(Q;\cH_{\mathrm{anc}},\cD)+\beta\bigr).
\]
\end{lemma}

\begin{proof}
Let $\cV$ be the value set of $Q$, and let
$\mathsf B_{j,v}=\E[\mathbf 1\{j(X)=j\}Q_X(v)(v-c_j)]$ for $j\in[J]$ and
$v\in\cV$.  Each $\mathsf B_{j,v}$ has the sign of $v-c_j$, so
$\mathrm{Disp}(Q)=\sum_{j,v}|\mathsf B_{j,v}|
=\sum_{j,v}\sgn(v-c_j)\mathsf B_{j,v}$.  Let
\[
  \|\mathsf B\|_{\pm}=\max_{a\in\{-1,+1\}^J,\ z\in\{-1,+1\}^{\cV}}
  \sum_{j,v}a_jz_v\mathsf B_{j,v}
  =\max_{x\in[-1,1]^J,\ y\in[-1,1]^{\cV}}x^{\top}\mathsf By,
\]
where the second equality holds because a bilinear form attains its maximum
over a product of cubes at a vertex.  For the anchor $a$,
$\E[a_{j(X)}Q_X(v)(v-Y)]=\sum_ja_j\mathsf B_{j,v}
-\beta\E[a_{j(X)}\tau_XQ_X(v)]$, and
$\sum_v|\beta\E[a_{j(X)}\tau_XQ_X(v)]|\le\beta$.  Hence
$\|\mathsf B\|_\pm\le\MC(Q;\cH_{\mathrm{anc}},\cD)+\beta$.

It remains to show $\sum_{j,v}\sgn(v-c_j)\mathsf B_{j,v}\le
\kappa_{\mathrm{anc}}\|\mathsf B\|_\pm$.  Merge the centers and the values in
increasing order.  Each center, and each value, occupies one of at most
$P=2J+1$ \emph{positions}: one of the $J$ centers or one of the $J+1$ open
gaps between and around them; $\sgn(v-c_j)$ depends only on the positions of
$c_j$ and $v$.  Let $H=\lceil\log_2P\rceil$, and identify the positions with
leaves of a complete binary tree of depth $H$, leaving some leaves unused.
For an internal node $u$, let $\mathbf 1_{\mathrm{left}(u)}$ and
$\mathbf 1_{\mathrm{right}(u)}$ be the indicator vectors of the positions
below its left and right children.  Two distinct positions $p<p'$ separate
at exactly one node, with $p$ on the left, so the matrix
$S_{p,p'}=\sgn(p'-p)$ satisfies
\[
  S=\sum_{h=1}^{H}\sum_{u\text{ at depth }h-1}\Bigl(
  \mathbf 1_{\mathrm{left}(u)}\mathbf 1_{\mathrm{right}(u)}^{\top}
  -\mathbf 1_{\mathrm{right}(u)}\mathbf 1_{\mathrm{left}(u)}^{\top}\Bigr).
\]
For a fixed depth and independent uniform signs $\eta_u$,
\[
  \sum_u\mathbf 1_{\mathrm{left}(u)}\mathbf 1_{\mathrm{right}(u)}^{\top}
  =\E_\eta\Bigl[\Bigl(\sum_u\eta_u\mathbf 1_{\mathrm{left}(u)}\Bigr)
  \Bigl(\sum_u\eta_u\mathbf 1_{\mathrm{right}(u)}\Bigr)^{\top}\Bigr],
\]
and both factors have entries in $[-1,1]$ because the subtrees at a fixed
depth are disjoint; the same holds for the second sum.  Pulling the factors
back from positions to centers and values, $\sum_{j,v}\sgn(v-c_j)\mathsf
B_{j,v}$ is a sum of $2H$ expectations of bilinear forms
$x^{\top}\mathsf By$ with $x,y$ in cubes, so it is at most
$2H\|\mathsf B\|_\pm$.  Finally, $2J+1\le1/\beta$ because
$J\le1/(4\Delta_c)$, so $2H\le2\lceil\log_2(1/\beta)\rceil\le
\kappa_{\mathrm{anc}}$.
\end{proof}

\subsection{A code that answers every fixed query}
\label{app:erm-lb-code}

Let $N_{\mathrm{code}}=\lceil\exp(8d_0+8D)\rceil$, and let
$\cH_{\mathrm{code}}$ consist of $N_{\mathrm{code}}$ independent uniform
vectors in $\{-1,+1\}^{\cX}$.  Call a nonzero vector $u\in\R^{\cX}$
\emph{spread} if $\|u\|_\infty\le\|u\|_2/(2\sqrt D)$.

\begin{lemma}[Coverage by the code]
\label{lem:code}
With probability at least $0.99$ over the draw of $\cH_{\mathrm{code}}$, for
every $I\subseteq\cX$ with $|I|\le d_0$ and every $\pi\in\{-1,+1\}^I$:
\begin{enumerate}[label=(\alph*),leftmargin=2em]
\item some $h\in\cH_{\mathrm{code}}$ agrees with $\pi$ on $I$; and
\item for every spread vector $u$ supported on $\cX\setminus I$, some
  $h\in\cH_{\mathrm{code}}$ agrees with $\pi$ on $I$ and satisfies
  $\langle u,h\rangle\ge(\sqrt D/8)\|u\|_2$.
\end{enumerate}
In particular, $\VC(\cH_{\mathrm{code}})\ge d_0$.
\end{lemma}

\begin{proof}
Let $\upsilon$ be uniform on $\{-1,+1\}^{\cX}$, and fix a spread unit vector
$u$ supported on $\cX\setminus I$.  We first show
$\Pr\{\langle u,\upsilon\rangle\ge\sqrt D/4\}\ge\frac12e^{-D/2}$.  Let
$\lambda=\sqrt D/2$, so that $|\lambda u_x|\le1/4$, and tilt the law of
$\upsilon$ by the density proportional to $e^{\lambda\langle
u,\upsilon\rangle}$.  Under the tilted law the coordinates are independent
with means $\tanh(\lambda u_x)$, so the mean $\mu$ of
$\langle u,\upsilon\rangle$ satisfies $0.48\sqrt D\le\mu\le\sqrt D/2$ (using
$y\ge\tanh y\ge0.979y$ for $0\le y\le1/4$), and its variance is at most
$\|u\|_2^2=1$.  By Chebyshev's inequality and $D\ge128$, the tilted
probability of $\sqrt D/4\le\langle u,\upsilon\rangle\le\sqrt D$ is at least
$1/2$.  On this event the density of the original law with respect to the
tilted law is
$\exp(-\lambda\langle u,\upsilon\rangle+\sum_x\log\cosh(\lambda u_x))\ge
e^{-\lambda\sqrt D}=e^{-D/2}$, which proves the claim.  The restriction of
$\upsilon$ to $I$ is independent of $\langle u,\upsilon\rangle$ and equals
$\pi$ with probability $2^{-|I|}\ge2^{-d_0}$.

Let $r=\sqrt D/(8\sqrt{|\cX|})$, and for each $I$ fix an $r$-net, in
Euclidean distance, of the spread unit vectors supported on
$\cX\setminus I$, with centers in that set and at most
$(1+2/r)^{|\cX|}=(1+16\sqrt{|\cX|/D})^{|\cX|}$ elements.  For a fixed
$I$, $\pi$, and center $u'$, the probability that no codeword agrees with
$\pi$ on $I$ and has $\langle u',h\rangle\ge\sqrt D/4$ is at most
$\exp(-N_{\mathrm{code}}2^{-d_0}e^{-D/2}/2)\le\exp(-\frac12e^{7d_0+7D})$.
The number of triples $(I,\pi,u')$ is at most
$(d_0+1)(2|\cX|)^{d_0}(1+16\sqrt{|\cX|/D})^{|\cX|}$, whose logarithm is at
most $2|\cX|\log(8|\cX|)$ since $d_0\le|\cX|/2$.  Since
$d_0\ge16D\log(8|\cX|)$ and $D\ge128$, the union bound gives failure
probability below $0.005$.  On the complementary event, every spread unit
vector $u$ has a center $u'$ with $\|u-u'\|_2\le r$, and every codeword has
norm $\sqrt{|\cX|}$, so the codeword found for $u'$ satisfies
$\langle u,h\rangle\ge\sqrt D/4-r\sqrt{|\cX|}=\sqrt D/8$; scaling gives~(b).
Similarly, the probability that some $(I,\pi)$ is not realized is at most
$(d_0+1)(2|\cX|)^{d_0}\exp(-N_{\mathrm{code}}2^{-d_0})<0.005$, which
gives~(a).
\end{proof}

We fix a code satisfying \cref{lem:code} for the rest of the proof.  It is
public: the learner may know it.

\subsection{The random instance and the oracle}
\label{app:erm-lb-instance}

Draw $\tau\in\{-1,+1\}^{\cX}$ uniformly, and, for $\ell\in[16]$, draw
independent uniform block signs $\theta^{(\ell)}\in\{-1,+1\}^J$.  The
outcome $Y$ given $X=x$ is Bernoulli with mean $\fs(x)=c_{j(x)}+\beta\tau_x$.
The \emph{planted auditors} are $h^{(\ell)}(x)=\tau_x\theta^{(\ell)}_{j(x)}$,
and the class is
\[
  \cH=\cH_{\mathrm{code}}\cup\cH_{\mathrm{anc}}\cup\{h^{(1)},\ldots,h^{(16)}\}.
\]
After aggregating repeated inputs, a weighted ERM query is a vector
$u\in\R^{\cX}$.  The oracle returns the first maximizer of $\langle
u,h\rangle$ in a fixed order that lists $\cH_{\mathrm{code}}$, then
$\cH_{\mathrm{anc}}$, then the planted auditors.  Thus a planted auditor is
returned only if it strictly beats every code word and every anchor.  Since
$J\le D\le d_0$,
\begin{equation}
  \label{eq:erm-lb-vc}
  d_0\le\VC(\cH)\le\log_2\bigl(N_{\mathrm{code}}+2^J+16\bigr)\le24d_0.
\end{equation}

Let $\mathbf S$ denote the labeled sample, and let
$\bar\tau_x=\E[\tau_x\mid\mathbf S]$ and $\bar\tau_*=\max_x|\bar\tau_x|$.

\begin{lemma}[The sample reveals little]
\label{lem:posterior}
If $n\le c|\cX|/(\beta^2\mathsf A\mathsf t)$, then with probability at least
$0.99$ over the sample, $\bar\tau_*\le1/(64\sqrt{\mathsf A})$.  Conditional on
the sample, the signs $\tau_x$ are independent, and the block signs
$\theta^{(\ell)}$ are independent, uniform, and independent of $\tau$.
\end{lemma}

\begin{proof}
The likelihood factors over inputs, which gives the conditional
independence; the block signs do not enter the likelihood.  Let $\mathsf n_x$
be the number of samples at $x$ and $\mathsf l_x$ the log-likelihood ratio of
$\tau_x=1$ against $\tau_x=-1$.  Since $c_j\in(1/4,1/2]$ and $\beta$ is
small, each observation changes $\mathsf l_x$ by at most $C_1\beta$, and
under either value of $\tau_x$ the expected change has absolute value at
most $C_1\beta^2$.  By Hoeffding's inequality and a union bound over $x$,
with probability at least $0.995$,
$|\mathsf l_x|\le C_1(\beta^2\mathsf n_x+\beta\sqrt{\mathsf n_x\mathsf t})$
for all $x$.  By Bernstein's inequality and a union bound,
$\max_x\mathsf n_x\le2n/|\cX|+C_1\mathsf t$ with probability at least
$0.995$.  Since $\bar\tau_x=\tanh(\mathsf l_x/2)$, on both events
\[
  \bar\tau_*\le C_2\Bigl[\beta^2\Bigl(\frac n{|\cX|}+\mathsf t\Bigr)
  +\beta\sqrt{\Bigl(\frac n{|\cX|}+\mathsf t\Bigr)\mathsf t}\Bigr].
\]
With $n/|\cX|\le c/(\beta^2\mathsf A\mathsf t)$ and
$\beta\mathsf t\sqrt{\mathsf A}\le c$, every term is at most a constant
multiple of $\sqrt c/\sqrt{\mathsf A}$, which is at most
$1/(64\sqrt{\mathsf A})$ for small $c$.
\end{proof}

\begin{lemma}[A fixed query rarely exposes a planted auditor]
\label{lem:exposure}
Condition on a sample with $\bar\tau_*\le1/(64\sqrt{\mathsf A})$.  Let
$u\in\R^{\cX}$ be determined by the sample, the public objects, and
randomness independent of $(\tau,\theta)$.  For each $\ell$, the conditional
probability that $h^{(\ell)}$ strictly beats every code word on $u$ is at
most
\[
  p_*=\exp\Bigl(-\frac D{128(1+D\bar\tau_*^2)}\Bigr)+e^{-D/2}.
\]
\end{lemma}

\begin{proof}
Write $h=h^{(\ell)}$ and $\theta=\theta^{(\ell)}$.  For $v\in\R^{\cX}$ and
$\lambda\in\R$, conditional on $\theta$ the variables $\tau_x$ are
independent with means $\bar\tau_x$, so Hoeffding's lemma gives
$\E[e^{\lambda\langle v,h\rangle}\mid\theta]\le\exp(\lambda\sum_jm_j\theta_j
+\lambda^2\|v\|_2^2/2)$ with $m_j=\sum_{x:j(x)=j}v_x\bar\tau_x$.  Averaging
over $\theta$ and using $m_j^2\le D\bar\tau_*^2\sum_{x:j(x)=j}v_x^2$,
\begin{equation}
  \label{eq:plant-subg}
  \E\bigl[e^{\lambda\langle v,h\rangle}\bigr]
  \le\exp\Bigl(\frac{\lambda^2}2(1+D\bar\tau_*^2)\|v\|_2^2\Bigr).
\end{equation}
Order the inputs so that $|u_x|$ is nonincreasing, write $u^\downarrow_l$ for
the $l$-th largest absolute value and $\mathsf r_l^2=\sum_{l'>l}
(u^\downarrow_{l'})^2$ for the squared norm of the tail after the first $l$
inputs, and let $I_l$ be the set of the first $l$ inputs.  There are three
cases.

\emph{Some tail vanishes.}  If $\mathsf r_l=0$ for some $l\le d_0$, then $u$
is supported on $I_l$, and by \cref{lem:code}(a) some code word matches the
signs of $u$ on $I_l$ and attains the maximum $\|u\|_1$, so no hypothesis
beats it strictly.

\emph{Some tail is spread.}  Suppose that for some $l<d_0$ the tail
$u|_{\cX\setminus I_l}$ is nonzero and spread, that is,
$u^\downarrow_{l+1}\le\mathsf r_l/(2\sqrt D)$.  By \cref{lem:code}(b), some
code word matches the signs of $u$ on $I_l$ and has score at least
$\sum_{l'\le l}u^\downarrow_{l'}+(\sqrt D/8)\mathsf r_l$.  The score of $h$ on
$I_l$ is at most $\sum_{l'\le l}u^\downarrow_{l'}$, so $h$ wins only if its
score on the tail exceeds $(\sqrt D/8)\mathsf r_l$.  By~\eqref{eq:plant-subg}
and Chernoff's bound, this has probability at most the first term of $p_*$.

\emph{No tail is spread.}  Otherwise, for every $l<d_0$ the tail is nonzero
and $u^\downarrow_{l+1}>\mathsf r_l/(2\sqrt D)$.  Then $\mathsf r_{l+1}^2<
(1-1/(4D))\mathsf r_l^2$, so
$\mathsf r_{d_0}\le e^{-(d_0-D+1)/(8D)}\mathsf r_{D-1}$, while
$u^\downarrow_D>\mathsf r_{D-1}/(2\sqrt D)$.  Since
$d_0\ge16D\log(8|\cX|)$, this gives
$u^\downarrow_D>2\sqrt{|\cX|}\,\mathsf r_{d_0}$.  By \cref{lem:code}(a), some
code word matches the signs of $u$ on $I_{d_0}$; its score is at least
$\sum_{l\le d_0}u^\downarrow_l-\sqrt{|\cX|}\mathsf r_{d_0}$.  If $h$ disagrees
with the sign of $u$ at one of the first $D$ inputs, its score is at most
$\sum_{l\le d_0}u^\downarrow_l-2u^\downarrow_D+\sqrt{|\cX|}\mathsf r_{d_0}$,
which is smaller.  So $h$ wins only if it agrees with the signs of $u$ on
the first $D$ inputs.  Conditional on $\theta$, the values $h(x)=
\tau_x\theta_{j(x)}$ are independent and each takes a given sign with
probability at most $(1+\bar\tau_*)/2$, so this has probability at most
$((1+\bar\tau_*)/2)^D\le e^{-D/2}$.
\end{proof}

\begin{lemma}[Transcript coupling]
\label{lem:coupling}
Fix a sample with $\bar\tau_*\le1/(64\sqrt{\mathsf A})$ and the learner's
random seed.  Consider a second run of the learner, with the same sample and
seed, against the \emph{baseline oracle}, which uses the same order but
optimizes only over $\cH_{\mathrm{code}}\cup\cH_{\mathrm{anc}}$.  The two
runs make the same queries and produce the same output, except with
conditional probability at most $0.01$ over $(\tau,\theta)$.
\end{lemma}

\begin{proof}
The baseline run, including its queries $u_1,\ldots,u_q$ and its output, is
determined by the sample, the seed, and the public objects, and is therefore
independent of $(\tau,\theta)$ given the sample.  If at no baseline query
does a planted auditor strictly beat every code word, the actual oracle
returns the same hypothesis as the baseline oracle at every query, and by
induction the two runs coincide.  By \cref{lem:exposure} and a union bound
over the $q$ queries and the $16$ planted auditors, the runs differ with
probability at most $16qp_*$.  Let $\mathsf b=\log(100q)$.  Since
$\mathsf A\ge\mathsf b$, $\bar\tau_*^2\le1/(4096\mathsf A)$, and
$D\ge C\mathsf b$ with $C\ge1024$,
\[
  \frac D{128(1+D\bar\tau_*^2)}
  =\Bigl(\frac{128}D+128\bar\tau_*^2\Bigr)^{-1}
  \ge\Bigl(\frac18+\frac1{32}\Bigr)^{-1}\mathsf b\ge4\mathsf b,
  \qquad\frac D2\ge4\mathsf b.
\]
Hence $16qp_*\le32q(100q)^{-4}<0.01$.
\end{proof}

\subsection{An undiscovered planted auditor forces error}
\label{app:erm-lb-error}

\begin{lemma}[Error of the baseline output]
\label{lem:plant-error}
Fix a sample with $\bar\tau_*\le1/(64\sqrt{\mathsf A})$ and the learner's
seed, and let $Q$ be the output of the baseline run.  Then, over
$(\tau,\theta)$ given the sample,
$\Pr\{\MC(Q;\cH,\cD)>\varepsilon\}\ge0.99\bigl(1-(11/12)^{16}\bigr)$.
\end{lemma}

\begin{proof}
By \cref{lem:coupling}, $Q$ is independent of $(\tau,\theta)$ given the
sample.  If $\mathrm{Disp}(Q)>\kappa_{\mathrm{anc}}(\varepsilon+\beta)$, then
\cref{lem:anchor} gives $\MC(Q;\cH_{\mathrm{anc}},\cD)>\varepsilon$ for every
$\tau$, and we are done.  Otherwise
$\mathrm{Disp}(Q)\le\frac98\kappa_{\mathrm{anc}}\beta$, since
$\varepsilon=\beta/8$.  Let $U_j=\{v:|v-c_j|\le\Delta_c/4\}$, which are
disjoint, and let
\[
  \mathsf P=\E_X\Bigl[\sum_{v\in U_{j(X)}}Q_X(v)\Bigr],\qquad
  \mathsf R=\E_X\Bigl[\tau_X\sum_{v\in U_{j(X)}}Q_X(v)\,(v-c_{j(X)})\Bigr].
\]
By Markov's inequality, $\mathsf P\ge1-4\mathrm{Disp}(Q)/\Delta_c\ge
1-9/256$.  Given the sample, $|\E\mathsf R|\le\bar\tau_*\mathrm{Disp}(Q)
\le\beta/16$, because $\sqrt{\mathsf A}\ge\kappa_{\mathrm{anc}}$; and since
the $\tau_x$ are independent and each term of $\mathsf R$ is at most
$\Delta_c/(4|\cX|)$ in absolute value, $\mathrm{Var}(\mathsf R)\le
\Delta_c^2/(16|\cX|)$.  By Chebyshev's inequality,
\[
  \Pr\{\beta\mathsf P-\mathsf R<\beta/2\}
  \le\frac{\Delta_c^2/(16|\cX|)}{(3\beta/8)^2}
  =\frac{65536\,\kappa_{\mathrm{anc}}^2}{9|\cX|}\le0.01,
\]
using $|\cX|\ge C\kappa_{\mathrm{anc}}^2$.

Fix $\tau$ with $\beta\mathsf P-\mathsf R\ge\beta/2$.  For a planted auditor
with block signs $\theta$,
\begin{align*}
  \sum_v\bigl|\E[h^{(\ell)}(X)Q_X(v)(v-Y)]\bigr|
  &=\sum_v\Bigl|\sum_j\theta_j\mathsf C_{j,v}\Bigr|=:\mathsf F(\theta),
  \qquad\text{where}\\
  \mathsf C_{j,v}&=\E\bigl[\mathbf 1\{j(X)=j\}\tau_XQ_X(v)(v-c_j-\beta\tau_X)\bigr].
\end{align*}
For each $j$ and $v$, averaging over the other block signs and Jensen's
inequality give $\E_\theta|\sum_{j'}\theta_{j'}\mathsf C_{j',v}|\ge|\mathsf
C_{j,v}|$.  Since the sets $U_j$ are disjoint and $\tau_x^2=1$,
\begin{equation}
  \label{eq:plant-mean}
  \E_\theta\mathsf F\ge\sum_j\sum_{v\in U_j}|\mathsf C_{j,v}|
  \ge\Bigl|\sum_j\sum_{v\in U_j}\mathsf C_{j,v}\Bigr|
  =|\mathsf R-\beta\mathsf P|\ge\frac\beta2.
\end{equation}
Let $\mathsf Z_v=\sum_j\theta_j\mathsf C_{j,v}$.  Khintchine's inequality
gives $\E\mathsf Z_v^4\le3(\E\mathsf Z_v^2)^2$, so Hölder's inequality gives
$(\E\mathsf Z_v^2)^{1/2}\le\sqrt3\,\E|\mathsf Z_v|$, and Minkowski's
inequality gives $(\E\mathsf F^2)^{1/2}\le\sum_v(\E\mathsf Z_v^2)^{1/2}\le
\sqrt3\,\E\mathsf F$.  By the Paley--Zygmund inequality
and~\eqref{eq:plant-mean}, $\Pr_\theta\{\mathsf F\ge\beta/4\}\ge
\frac14(\E\mathsf F)^2/\E\mathsf F^2\ge1/12$.  The $16$ planted auditors have
independent block signs given $\tau$ and the sample, so with probability at
least $1-(11/12)^{16}$ one of them has error at least $\beta/4=2\varepsilon$.
\end{proof}

\subsection{Proof of \texorpdfstring{\cref{thm:erm-lb}}{the black-box lower bound}}
\label{app:erm-lb-proof}

\begin{proof}
Fix a code as in \cref{lem:code}, and draw $(\tau,\theta)$, the sample, and
the learner's seed.  By \cref{lem:posterior}, the sample satisfies
$\bar\tau_*\le1/(64\sqrt{\mathsf A})$ with probability at least $0.99$.  On
such a sample and for every seed, \cref{lem:coupling,lem:plant-error} give
\[
  \Pr\{\MC(Q;\cH,\cD)>\varepsilon\mid\text{sample, seed}\}
  \ge0.99\bigl(1-(11/12)^{16}\bigr)-0.01,
\]
where $Q$ is the actual output.  Hence the overall failure probability is at
least $0.99\bigl(0.99(1-(11/12)^{16})-0.01\bigr)>0.72>2/3$.  This probability
averages over $(\tau,\theta)$, so some fixed $(\tau,\theta)$ has failure
probability greater than $2/3$ over the sample and the seed alone.  It
defines a fixed class, a weighted ERM oracle (the ordered one above), and a
distribution over $\cX\times\{0,1\}$ with uniform input marginal; the
dimension bounds are~\eqref{eq:erm-lb-vc}.  The code, the anchors, the
centers, and the input distribution are the same for every $(\tau,\theta)$,
so the learner may know them.
\end{proof}

\section{Proofs for Section~\ref*{sec:swap}}
\label{app:swap-all}

This appendix proves the upper bounds of \cref{sec:swap}
(\cref{thm:swap-finite,thm:swap-vc,thm:swap-poly}) in \cref{app:swap}, and
the lower bounds (\cref{lem:lb-class}, \cref{thm:swap-lb}, and
\cref{prop:swap-lb-vc1}) in \cref{app:lb}.

\subsection{Upper bounds}
\label{app:swap}

\begin{proof}[Proof of \cref{thm:swap-finite}]
For each $v\in\Lambda$, run exponential weights $w_{t,v}\in\Delta(\cG)$ with a
common learning rate $\eta$, and let $\bar g_{t,v}(x)=\sum_gw_{t,v}(g)g(x)$.
Forecast $Q_t(x)=\Root((\bar g_{t,v}(x))_v)$, and after observing $Y_t$ set
$Z_{t,v}=Q_t(X_t)(v)(\mu_t(X_t)-Y_t)$ and update
$w_{t+1,v}(g)\propto w_{t,v}(g)\exp(\eta g(X_t)Z_{t,v})$.  Since the reward
$g(X_t)Z_{t,v}$ lies in an interval of length $2|Z_{t,v}|$, Hoeffding's lemma
gives, for each $v$,
\[
  \max_{g\in\cG}\sum_tg(X_t)Z_{t,v}\le\sum_t\bar g_{t,v}(X_t)Z_{t,v}
  +\frac{\log N}{\eta}+\frac\eta2\sum_tZ_{t,v}^2.
\]
Sum over $v$.  By \cref{lem:root}, $\sum_v\bar g_{t,v}(X_t)Z_{t,v}=
(\mu_t-Y_t)\sum_vQ_t(X_t)(v)\bar g_{t,v}(X_t)\le0$, and
$\sum_vZ_{t,v}^2\le(\sum_v|Z_{t,v}|)^2\le1$.  The virtual swap error, which
equals $\sum_v\max_{g\in\cG}\sum_tg(X_t)Z_{t,v}$ because $\cG$ is closed
under negation, is therefore at most $m\log N/\eta+\eta T/2$, and
$\eta=\sqrt{2m\log N/T}$ gives $\sqrt{2Tm\log N}$.  \Cref{lem:adjacent} adds
$T/(2K)$.

For the batch bound, index tests by tuples $\mathbf g=(g_v)_v\in\cG^{\Lambda}$
and let $z_{t,\mathbf g}(x,y)=(\mu_t(x)-y)\sum_vQ_t(x)(v)g_v(x)$, a
predictable map into $[-1,1]$.  For a fixed tuple, the empirical minus
expected sums form a martingale with differences bounded by two; Azuma's
inequality and a union bound over the $N^m$ tuples bound the deviation by
$C\sqrt{T(m\log N+\log(1/\delta))}$ for all tuples.  The supremum over tuples
of the empirical sum is the virtual swap error, and the supremum over tuples
of $\frac1T\sum_t\E_\cD[z_{t,\mathbf g}]$ is the virtual swap error of
$\Qbar$.  Dividing by $T$ and adding $1/(2K)$ from \cref{lem:adjacent}
completes the proof.
\end{proof}

\begin{proof}[Proof of \cref{thm:swap-vc}]
Let $K=\lceil C/\varepsilon\rceil$, $m=K+1$, and $\rho=\varepsilon/(8m)$.  Draw
$\ncv=O((d\log(1/\rho)+\log(2/\delta))/\rho)$ unlabeled inputs, obtained by
discarding the labels of examples.  The disagreement sets $\{h\ne h'\}$ have
VC dimension $O(d)$, so with probability at least $1-\delta/2$ the sample is a $\rho$-net for
them: any two auditors that agree on the sample disagree with probability at
most $\rho$.  Let $\cG$ consist of one auditor for each labeling of the
sample realized by $\cH$, together with the negations; then
$\log|\cG|\le\log2+d\log(e\ncv/d)$.  Such representatives can be found by
extending partial labelings one point at a time and testing consistency with
an ERM call, using $O(\ncv|\cG|)$ calls, which is exponential in $d$.

Run the forecaster of \cref{thm:swap-finite} on $\cG$ with an independent
training sample of size
$T=\wtO((m\log|\cG|+\log(1/\delta))/\varepsilon^2)
=\wtO((d+\log(1/\delta))\varepsilon^{-3})$, so that
$\SMC(\Qbar;\cG,\cD)\le\varepsilon/2$ with probability at least
$1-\delta/2$.
For any randomized predictor and any auditors $h_v\in\cH$, replacing each $h_v$ by
the element of $\cG$ that agrees with it on the sample changes the bias at
$v$ by at most $2\rho$; hence
$\SMC(\Qbar;\cH,\cD)\le\SMC(\Qbar;\cG,\cD)+2m\rho\le\varepsilon$.  This
approximation is uniform over predictors, so it applies to the learned one.
The unlabeled sample size $\ncv=\wtO((d+\log(1/\delta))\varepsilon^{-2})$ is
smaller than $T$.
\end{proof}

\begin{proof}[Proof of \cref{thm:swap-poly}]
For each $v$, the estimate $\hat h_{t,v}(X_t)$ averages $L$ independent draws
with mean $\bh_{t,v}(X_t)$, so by Hoeffding's inequality and a union bound,
with probability at least $1-\delta/2$,
$|\hat h_{t,v}(X_t)-\bh_{t,v}(X_t)|\le\alpha$ for all $t$ and $v$.  Let
$Z_{t,v}=Q_t(X_t)(v)(\mu_t(X_t)-Y_t)$.  On this event,
\[
  \sum_v\bh_{t,v}(X_t)Z_{t,v}\le(\mu_t-Y_t)\sum_vQ_t(X_t)(v)\hat h_{t,v}(X_t)
  +\alpha\le\alpha
\]
by \cref{lem:root}.  \Cref{thm:hybrid} for $\Hpm$, with failure probability
$\delta/(2m)$ for each $v$, gives
\[
  \sup_{h\in\Hpm}\sum_th(X_t)Z_{t,v}\le\sum_t\bh_{t,v}(X_t)Z_{t,v}+
  \Reg\Bigl(T,Cd,\frac{\delta}{2m}\Bigr).
\]
Summing over $v$ bounds the virtual swap error by
$T\alpha+m\Reg(T,Cd,\delta/(2m))$, since $\Hpm$ is closed under negation;
\cref{lem:adjacent} adds $T/(2K)$.

For the batch bound, halve the failure probabilities above, and apply
\cref{lem:martingale}(b) with $z_{t,v}(x,y)=Q_t(x)(v)(\mu_t(x)-y)$, which are
predictable because the auxiliary draws are made before the example, and
satisfy $\sum_v|z_{t,v}|\le1$; then add the population form of
\cref{lem:adjacent}.  With $K=\lceil C/\varepsilon\rceil$,
$\alpha=\varepsilon/C$, and
$T=\wtO((d+\log(1/\delta))m^2/\varepsilon^2)=\wtO((d+\log(1/\delta))\varepsilon^{-4})$
the error is at most $\varepsilon$.  Each round uses $m(2L+2)$ ERM calls over
$\cH$ with $L=\wtO(\varepsilon^{-2})$, for a total of
$\wtO((d+\log(1/\delta))\varepsilon^{-7})$.
\end{proof}

\subsection{Lower bounds}
\label{app:lb}

Throughout, we use the balanced family of \cref{sec:swap-lb}, write $j(x)$ for the
block of an input $x$, and $\fs=\fs_\theta$.

\begin{proof}[Proof of \cref{lem:lb-class}]
The local auditors of one block shatter its $d$ inputs.  If $J=1$ there are
only $d$ inputs.  Otherwise, any $d+1$ inputs span at least two blocks, so
their first and last inputs lie in different blocks; since $d+1\ge3$, choose
an interior input as well.  The pattern $(-,+,-)$ on these three inputs is
impossible: a threshold cannot change from $-1$ back to $+1$, and a local
auditor cannot be $-1$ in two blocks.

For weights $u_{j,k}$, the best threshold is found from prefix sums, and the
best local auditor of block $j$ has score
$\sum_{j'\ne j,k}u_{j',k}+\sum_k|u_{j,k}|$; scanning the $J$ blocks and
comparing with the best threshold takes $O(Jd)$ time.

Under $\cD_\emptyset$ every residual $c_j-\fs_\emptyset(j,k)$ is zero, so all
biases vanish.  Under $\cD_\theta$, the bias of $Q^0$ at $c_j$ with respect to
$h$ is
$(Jd)^{-1}\sum_kh(j,k)(c_j-\fs_\theta(j,k))=-\beta(Jd)^{-1}\sum_kh(j,k)\tau_{j,k}$.
Since $\sum_k\tau_{j,k}=0$, an auditor constant on block $j$ has zero bias
there.  A local auditor is constant on all blocks but its own, where its bias
is at most $\beta/J$ in absolute value, with equality for
$u=-\tau_{j,\cdot}$.  A threshold is constant on all blocks but one, where
its bias is at most $\beta/J$.  Hence $\MC(Q^0)=\beta/J$ and, since each block
attains $\beta/J$, $\SMC(Q^0)=J\cdot\beta/J=\beta$.
\end{proof}

Let $Q$ be a finitely supported randomized predictor, let $V\sim Q_X$ be drawn
independently of $Y$ given $X$, write $\varepsilon_Q=\SMC(Q)$, and let
$I_j=\{v:|v-c_j|<100\beta/3\}$.  These intervals are disjoint since
$c_{j+1}-c_j=100\beta$.

\begin{lemma}[Transport and decoding]
\label{lem:decode}
For every $\theta$,
\begin{align}
  \E|V-c_{j(X)}|&\le\varepsilon_Q+\beta,\label{eq:transport}\\
  \psi:=\frac1{Jd}\sum_{j,k}\sum_{v\in I_j}Q_{(j,k)}(v)\,|v-\fs(j,k)|
    &\le\varepsilon_Q.\label{eq:near}
\end{align}
Let $\Delta_{j,k}=\sum_{v\in I_j}Q_{(j,k)}(v)\clip(v-c_j,[-\beta,\beta])$, and
let $\hat\tau_{j,k}=+1$ if $\Delta_{j,k}\ge0$ and $-1$ otherwise.  The
fraction of inputs with $\hat\tau_{j,k}\ne\tau_{j,k}$ is at most
$\varepsilon_Q/\beta+3(\varepsilon_Q+\beta)/(100\beta)$.
\end{lemma}

\begin{proof}
For a value $v$, the auditor $x\mapsto\sgn(v-c_{j(x)})$ (with $\sgn(0)=1$) is
a threshold, since $c_j$ increases with $j$.  Its bias at $v$ is
$\E[Q_X(v)\sgn(v-c_{j(X)})(v-c_{j(X)}-\beta\tau_X)]\ge
\E[Q_X(v)|v-c_{j(X)}|]-\beta\E[Q_X(v)]$.  Summing over $v$ gives
$\varepsilon_Q\ge\E|V-c_{j(X)}|-\beta$, which is~\eqref{eq:transport}.  By
Markov's inequality,
\begin{equation}
  \label{eq:outlier}
  \Pr\{V\notin I_{j(X)}\}\le\frac{3(\varepsilon_Q+\beta)}{100\beta}.
\end{equation}

Fix $v\in I_j$.  The bias of the local auditor $h_{j,u}$ at $v$ is
$C_v+(Jd)^{-1}\sum_ku_kQ_{(j,k)}(v)(v-\fs(j,k))$, where $C_v$, the
contribution of the other blocks, does not depend on $u$.  Its maximum
absolute value over $u$ is $|C_v|+(Jd)^{-1}\sum_kQ_{(j,k)}(v)|v-\fs(j,k)|$,
which is at least the second term.  Summing over $v\in I_j$ and $j$, which
are distinct values, gives~\eqref{eq:near}.

For $a\in\R$ and $\tau\in\{-1,+1\}$, one checks that
$|a-\beta\tau|\ge\beta-\tau\clip(a,[-\beta,\beta])$.  Applying this with
$a=v-c_j$, the near loss at $(j,k)$ satisfies
$\sum_{v\in I_j}Q_{(j,k)}(v)|v-\fs(j,k)|\ge\beta Q_{(j,k)}(I_j)-\tau_{j,k}
\Delta_{j,k}$.  At a mistaken input $\tau_{j,k}\Delta_{j,k}\le0$, so its near
loss is at least $\beta Q_{(j,k)}(I_j)$.  Writing
$1=Q_{(j,k)}(I_j)+Q_{(j,k)}(\R\setminus I_j)$ and averaging over mistaken
inputs, the fraction of mistakes is at most
$\psi/\beta+\Pr\{V\notin I_{j(X)}\}$, which gives the claim
by~\eqref{eq:near} and~\eqref{eq:outlier}.
\end{proof}

\begin{proof}[Proof of \cref{thm:swap-lb}]
The first two claims are \cref{lem:lb-class}, and $\beta/J\le\beta/20$ since
$J\ge21$.  Suppose a learner uses $n\le Jd/(6400\beta^2)$ samples and
succeeds with probability at least $2/3$ for every $\theta$.  Put the uniform
prior on $\theta$.  On success, $\varepsilon_Q\le\beta/20$, and
\cref{lem:decode} gives a mistake fraction at most
$1/20+3(1+1/20)/100=0.0815$.  Estimate each coordinate $\theta_{j,i}$ by
choosing one of the inputs $(j,2i-1)$ and $(j,2i)$ uniformly at random and
reading off $\hat\tau$ there, negated at the even input.  The average over
coordinates of the error probability of this estimate equals the expected
mistake fraction, which is at most $(2/3)(0.0815)+1/3<0.388$.

On the other hand, two parameters that differ in one coordinate change the
mean outcome only at two inputs, by $2\beta$.  For Bernoulli means $a,b\in
[1/8,5/8]$ with $|a-b|=2\beta$,
$\KL(\Ber(a)\,\|\,\Ber(b))\le(a-b)^2/(b(1-b))\le64\beta^2$, so the
Kullback--Leibler divergence between the laws of one example is at most
$128\beta^2/(Jd)$, and between the laws of $n$ examples at most
$128n\beta^2/(Jd)\le0.02$.  By Pinsker's inequality, the total variation
distance is at most $0.1$, so any estimate of the differing coordinate errs
with probability at least $(1-0.1)/2=0.45$ on average over the two
parameters.  Pairing the parameters that differ in each coordinate and
averaging over the prior (Assouad's lemma, \citealp{assouad1983deux,
yu1997assouad}), the average error over coordinates is at least $0.45$, also
for randomized estimates.  This contradicts the bound $0.388$.
\end{proof}

\begin{proof}[Proof of \cref{prop:swap-lb-vc1}]
With $d=2$ and only thresholds, the class has VC dimension one.  The
predictor $Q^0$ still has $\MC(Q^0)=\beta/J$ and $\SMC(Q^0)=\beta$: a
threshold cutting block $j$ between its two inputs has bias $\beta/J$ there.
The proof of~\eqref{eq:transport} used only thresholds.  For~\eqref{eq:near},
fix $v\in I_j$, let $u_a$ and $u_b$ be the contributions of the two inputs of
block $j$ to the bias at $v$ with auditor value $+1$, and let $C$ be the
contribution of the other blocks for thresholds that cut in or next to block
$j$.  The thresholds cutting just before, between, and just after the two
inputs have biases $C-u_a-u_b$, $C+u_a-u_b$, and $C+u_a+u_b$.  Half the
differences of consecutive biases are $u_a$ and $u_b$, so the largest
absolute bias is at least $\max\{|u_a|,|u_b|\}\ge(|u_a|+|u_b|)/2$.  Hence
$\psi\le2\varepsilon_Q$, and the mistake fraction is at most
$2\varepsilon_Q/\beta+3(\varepsilon_Q+\beta)/(100\beta)\le0.1315$ when
$\varepsilon_Q\le\beta/20$.  The average coordinate error is then at most
$(2/3)(0.1315)+1/3<0.422<0.45$, and the rest of the proof of
\cref{thm:swap-lb} applies unchanged.
\end{proof}

\section{Proofs for Section~\ref*{sec:audit}}
\label{app:audit}

This appendix proves the hardness of approximating the ordinary error
(\cref{thm:np-hard}) and the sample lower bounds for error estimation
(\cref{thm:audit-lb}).

\begin{proof}[Proof of \cref{thm:np-hard}]
\emph{Gap Max-Cut.}  By \citet{hastad2001some}, using the gadgets of
\citet{trevisan2000gadgets}, there is an absolute constant $c_g\in(0,1)$ such
that, given a graph $G$ with edge set $E$ and an integer $k$, it is NP-hard to
distinguish $\textsc{MaxCut}(G)\ge k$ from $\textsc{MaxCut}(G)<(1-c_g)k$.
Every graph has a cut with at least $|E|/2$ edges, so we may assume
$k>|E|/2$.

\emph{Construction.}  Let $G$ have vertex set $[n_G]$, $n_G\ge2$, and nonempty
edge set $E$.  Use the grid with $K=2n_G$, so $m=2n_G+1$, and the vertex values
$u/n_G$ and all midpoints $(u+u')/(2n_G)$ are grid values.  For an edge
$e=\{u,u'\}$ with $u<u'$, let $k_e=u'-u$ and $y_e=(u+u')/(2n_G)$.  The input
space consists of three inputs $(e,\mathrm{lo})$, $(e,\mathrm{hi})$,
$(e,\mathrm{mid})$ for each edge, with probabilities $1/(2k_e|E|)$,
$1/(2k_e|E|)$, and $(1-1/k_e)/|E|$; at each of them the outcome is Bernoulli
with mean $y_e$.  The deterministic predictor is $p(e,\mathrm{lo})=u/n_G$,
$p(e,\mathrm{hi})=u'/n_G$, and $p(e,\mathrm{mid})=y_e$.  The class $\cH$
consists of all functions of the edge alone, $x\mapsto g(e(x))$ for
$g\in\{-1,+1\}^E$, where $e(x)$ is the edge of input $x$.  Its VC dimension is
$d=|E|$, and weighted ERM chooses each edge's sign to match the sign of the
total weight of its three inputs, in linear time.

\emph{The error.}  The inputs $(e,\mathrm{mid})$ have zero residual.  The
input $(e,\mathrm{lo})$ contributes
$g(e)\cdot\frac1{2k_e|E|}(u/n_G-y_e)=-g(e)/(4n_G|E|)$ to the bias at $u/n_G$,
and $(e,\mathrm{hi})$ contributes $+g(e)/(4n_G|E|)$ to the bias at $u'/n_G$.
Writing $\sigma_u$ for the sign assigned to the value $u/n_G$,
\cref{lem:signed} gives
\[
  \MC(p)=\max_{g,\sigma}\frac1{4n_G|E|}\sum_{\{u,u'\}\in E}g(e)(\sigma_{u'}-\sigma_u)
  =\max_\sigma\frac1{4n_G|E|}\sum_{\{u,u'\}\in E}|\sigma_{u'}-\sigma_u|
  =\frac{\textsc{MaxCut}(G)}{2n_G|E|};
\]
signs at the other grid values multiply zero biases.  Computing $\MC(p)$
exactly therefore computes $\textsc{MaxCut}(G)$, which is NP-hard
\citep{karp1972reducibility}.

\emph{Approximation.}  Let $a=k/(2n_G|E|)$.  Since $k>|E|/2$ and $m=2n_G+1$,
we have $a>1/(4n_G)>1/(2m)$.  In the first case of the gap problem
$\MC(p)\ge a$, and in the second $\MC(p)<(1-c_g)a$.  Let $c=c_g/8$ and
$c_0=c_g/4$.  An estimate with $|\hat e-\MC(p)|\le c/m<c_ga/4$ satisfies
$\hat e>(1-c_g/4)a$ in the first case and $\hat e<(1-3c_g/4)a$ in the second.
An estimate with $(1-c_0)\MC(p)\le\hat e\le(1+c_0)\MC(p)$ satisfies
$\hat e\ge(1-c_g/4)a$ in the first case and
$\hat e<(1+c_g/4)(1-c_g)a<(1-3c_g/4)a$ in the second.  In both cases,
comparing $\hat e$ with $(1-c_g/2)a$ decides the gap problem.

\emph{Simulation.}  An algorithm as in the theorem can be run on this
instance in time polynomial in the size of $G$.  Indeed, $m$ and $d$ are
polynomial in it, and the oracle is implemented in linear time.  A sample is
generated by choosing a uniform edge, a uniform integer in $[2k_e]$ to select
the input ($1$ and $2$ select $\mathrm{lo}$ and $\mathrm{hi}$, and the other
values select $\mathrm{mid}$), and a uniform integer $\ell\in[2n_G]$ to set
$Y=\mathbf 1\{\ell\le u+u'\}$.  Uniform integers can be generated exactly in
expected polynomial time, so with probability at least $0.99$ the simulation
finishes within a polynomial time bound, and the resulting decision is
correct with probability at least $2/3-0.01>0.65$.  This places the gap
problem in BPP, so $\mathrm{NP}\subseteq\mathrm{BPP}$.
\end{proof}

\begin{proof}[Proof of \cref{thm:audit-lb}]
Write $|\cX|=Jd$, and let $P_\emptyset$ and $P_\theta$ be the laws of one example
under $\cD_\emptyset$ and $\cD_\theta$.  At an input $x$ in block $j$, the
likelihood ratio is $1+r_\theta(x,y)$ with
$r_\theta(x,1)=\beta\tau_x/c_j$ and $r_\theta(x,0)=-\beta\tau_x/(1-c_j)$.
Hence $\E_{P_\emptyset}r_\theta=0$, and for parameters $\theta,\theta'$
with signs $\tau,\tau'$,
\[
  \E_{P_\emptyset}[r_\theta r_{\theta'}]
  =\frac1{|\cX|}\sum_x\frac{\beta^2\tau_x\tau'_x}{c_{j(x)}(1-c_{j(x)})}
  =\frac2{|\cX|}\sum_{(j,i)}a_j\,\theta_{j,i}\theta'_{j,i},
  \qquad a_j=\frac{\beta^2}{c_j(1-c_j)}\le\frac{16\beta^2}{3},
\]
where $j(x)$ is the block of $x$, the last sum ranges over the $|\cX|/2$ pairs
$(j,i)$, and we used
$\tau_{j,2i-1}\tau'_{j,2i-1}=\tau_{j,2i}\tau'_{j,2i}=\theta_{j,i}\theta'_{j,i}$
and $c_j\in(1/4,1/2]$.  Let $\bar P^n=\E_\theta P_\theta^n$ for uniform
$\theta$.  For independent uniform $\theta,\theta'$, the products
$\upsilon_{j,i}=\theta_{j,i}\theta'_{j,i}$ are independent uniform signs, so
\begin{align*}
  1+\chi^2\bigl(\bar P^n\,\big\|\,P_\emptyset^n\bigr)
  &=\E_{\theta,\theta'}\bigl(1+\E_{P_\emptyset}[r_\theta r_{\theta'}]\bigr)^n
  \le\E\exp\Bigl(\frac{2n}{|\cX|}\sum_{(j,i)}a_j\upsilon_{j,i}\Bigr)\\
  &=\prod_{(j,i)}\cosh\Bigl(\frac{2na_j}{|\cX|}\Bigr)
  \le\exp\Bigl(\frac{256\,n^2\beta^4}{9|\cX|}\Bigr),
\end{align*}
using $1+z\le e^z$, $\cosh z\le e^{z^2/2}$, and the $|\cX|/2$ pairs.  If
$n\le0.1\sqrt{|\cX|}/\beta^2$, then $\chi^2\le e^{0.29}-1<0.34$, and the total
variation distance between $P_\emptyset^n$ and $\bar P^n$ is at most
$\sqrt{\chi^2}/2<0.3$.  For any algorithm, including randomized ones, the
probability of outputting $1$ under $P_\emptyset^n$ plus the average over
$\theta$ of the probability of outputting $0$ under $P_\theta^n$ is therefore
at least $0.7$.  The assumed guarantees make this sum at most $2/3$, so
$n>0.1\sqrt{Jd}/\beta^2$.

For~(a), let $J=m$ and $\beta=3\varepsilon$; then $\beta\le1/400$ and
$J\le1/(400\beta)$.  By \cref{lem:lb-class}, $Q^0$ has $m$ values, swap error
$0$ under $\cD_\emptyset$, and swap error $3\varepsilon$ under every
$\cD_\theta$.  An $\varepsilon$-estimator applied to $Q^0$, followed by the test
$\hat e>3\varepsilon/2$, satisfies the hypothesis, so it uses more than
$0.1\sqrt{md}/(9\varepsilon^2)$ samples.  For~(b), let $J=1$ and
$\beta=3\varepsilon$, and use the ordinary errors $0$ and
$\beta/J=3\varepsilon$ of the constant predictor $Q^0$.
\end{proof}

\section{Proofs for Section~\ref*{sec:det}}
\label{app:det}

This appendix proves the rounding lemma for VC classes of tests
(\cref{lem:cell-tests}), cell rounding for VC classes
(\cref{thm:vc-rounding}) and its corollary for inputs without atoms
(\cref{cor:atomless}), the variance budget (\cref{lem:budget}), and the
batch guarantee with interval hints (\cref{prop:batch-hints}).

\begin{proof}[Proof of \cref{lem:cell-tests}]
For a cell $R$ and a test $g$, let
\[
  A_R(g)=\int_Rg\bigl(x,G_x(\xi_R)\bigr)\,\varphi\bigl(x,G_x(\xi_R)\bigr)\,d\cD_X(x).
\]
The deterministic expectation is $\sum_RA_R(g)$.  Since $G_x(\xi)\sim Q_x$,
the randomized expectation is $\sum_R\E A_R(g)$, where the expectation is
over $\xi_R$.  The contributions of different cells are independent, and
$|A_R(g)|\le \varphi_R\cD_X(R)$.  If $V_\varphi=0$, every contribution vanishes, so
assume $V_\varphi>0$.

\emph{Expected supremum.}  Let $\xi'_R$ be independent copies of the seeds,
let $A'_R(g)$ be the corresponding contributions, and let $s_R$ be
independent uniform signs.  By symmetrization,
\[
  \E\sup_{g\in\cG}\Bigl|\sum_R\bigl(A_R(g)-\E A_R(g)\bigr)\Bigr|
  \le\E\sup_{g\in\cG}|Z_g|,
  \qquad
  Z_g=\sum_Rs_R\bigl(A_R(g)-A'_R(g)\bigr).
\]
Conditional on both families of seeds, $(Z_g)_{g\in\cG}$ is a Rademacher
process in the signs $s_R$, so its
increments are sub-Gaussian with respect to the metric $\varrho$ given by
$\varrho(g,g')^2=\sum_R\bigl(A_R(g)-A'_R(g)-A_R(g')+A'_R(g')\bigr)^2$.  By
the Cauchy--Schwarz inequality,
$(A_R(g)-A_R(g'))^2\le \varphi_R^2\cD_X(R)\int_R(g-g')^2(x,G_x(\xi_R))\,d\cD_X(x)$,
and similarly for $A'_R$, so
\begin{align*}
  \varrho(g,g')^2&\le2\sum_R\varphi_R^2\cD_X(R)\int_R\Bigl[(g-g')^2\bigl(x,G_x(\xi_R)\bigr)
  +(g-g')^2\bigl(x,G_x(\xi'_R)\bigr)\Bigr]\,d\cD_X(x)\\
  &=4V_\varphi\,\|g-g'\|_{L_2(\nu_{\mathrm{cell}})}^2,
\end{align*}
where $\nu_{\mathrm{cell}}$ is the probability measure on $\cX\times\Lambda$ that, for each
cell $R$ and $x\in R$, places density $\varphi_R^2\cD_X(R)/(2V_\varphi)$ with respect to
$\cD_X$ on each of the points $(x,G_x(\xi_R))$ and $(x,G_x(\xi'_R))$.  Since
$\cG$ has VC dimension $d_{\cG}$ and takes values $\pm1$, its covering
numbers satisfy $\log N(u,\cG,L_2(\nu_{\mathrm{cell}}))\le Cd_{\cG}\log(C/u)$ for every
probability measure on $\cX\times\Lambda$ and $0<u\le2$ \citep{haussler1995sphere}.
Dudley's entropy integral then bounds $\E\sup_g|Z_g-Z_{g_0}|$ by
$C\sqrt{V_\varphi}\int_0^2\sqrt{d_{\cG}\log(C/u)}\,du\le C\sqrt{V_\varphi d_{\cG}}$ for a
fixed $g_0\in\cG$, and $\E|Z_{g_0}|\le(\E Z_{g_0}^2)^{1/2}\le2\sqrt{V_\varphi}$.
Hence the expected supremum is at most $C\sqrt{V_\varphi d_{\cG}}$.

\emph{Concentration.}  Changing one seed $\xi_R$ changes the supremum by at
most $2\varphi_R\cD_X(R)$.  By the bounded-differences inequality, the supremum
exceeds its expectation by more than
$\sqrt{2V_\varphi\log(1/\delta)}\le\sqrt{2V_\varphi\log(2/\delta)}$ with
probability at most $\delta$, which proves the lemma.
\end{proof}

\begin{proof}[Proof of \cref{thm:vc-rounding}]
Both expectations below are linear in $Y$ given $X$, and neither the tests
nor $p$ depends on $Y$, so we may replace $Y$ by $\fs(X)$.  Let
$a(x,v)=\clip(v,I_x)-\fs(x)$ for $x\in O$ and $a(x,v)=v-\fs(x)$ otherwise.
Since $Q_x$, and hence $p(x)$, is supported within $1/K$ of $I_x$, using $a$
in place of the residual changes each of the two expectations by at most
$\cD_X(O)/K\le1/K$.  We apply \cref{lem:cell-tests} with $\varphi=a$.  Since $\clip(v,I_x)$ and
$\fs(x)$ both lie in $I_x$, we have $|a(x,v)|\le2r_x$ for $x\in O$, so we may
take $\varphi_{\{x\}}=2r_x$ for $x\in O$ and $\varphi_R=1$ for
$R\in\cP_{\mathrm{rest}}$, which gives $V_\varphi\le4V_{\cP}$.  When $O=\emptyset$, no
residual is modified.

For ordinary error, let $\cG_o=\{(x,v)\mapsto h(x)\sigma_v:h\in\cH,\
\sigma\in\{-1,+1\}^{\Lambda}\}$.  On any $n$ points of $\cX\times\Lambda$, a
test is determined by the labeling of their inputs by $\cH$ and by the signs
at the at most $m$ grid values that occur, so $\cG_o$ has at most
$2^m(en)^d$ patterns by the Sauer--Shelah lemma \citep{sauer1972density}.  If
$\cG_o$ shatters the $n$ points, then $2^n\le2^m(en)^d$, which gives
$n\le C(d+m)$.  For $g=h\sigma$, the deterministic expectation with the
true residual is $\sum_v\sigma_vB_h(p,v)$, and the randomized expectation is
$\sum_v\sigma_vB_h(Q,v)\le\MC(Q)$.  On the event of \cref{lem:cell-tests},
applied to $\cG_o$ and $\varphi$, every such test therefore satisfies
$\sum_v\sigma_vB_h(p,v)\le\MC(Q)+C\sqrt{V_{\cP}(d+m+\log(2/\delta))}+2/K$,
and by \cref{lem:signed} the supremum of the left side over $h\in\cH$ and
$\sigma$ is $\MC(p)$.

For swap error, let $\cG_s=\{(x,v)\mapsto h_v(x):(h_v)_{v\in\Lambda}\in
(\Hpm)^{\Lambda}\}$.  If $\cG_s$ shatters a set of points, then for each
$v$ the inputs of the points with grid value $v$ are shattered by $\Hpm$, so
the set has at most $m\VC(\Hpm)\le Cmd$ points.  For a tuple $(h_v)$, the
deterministic expectation with the true residual is $\sum_vB_{h_v}(p,v)$,
and the randomized expectation is $\sum_vB_{h_v}(Q,v)\le\SMC(Q)$.  Since
$\Hpm$ is closed under negation, the supremum of $\sum_vB_{h_v}(p,v)$ over
tuples is $\SMC(p)$, and \cref{lem:cell-tests} applied to $\cG_s$ gives the
claim.
\end{proof}

\begin{proof}[Proof of \cref{cor:atomless}]
Take $O=\emptyset$.  Draw $\npt$ independent unlabeled inputs, and let the
cells be the open order intervals between consecutive sample points, the two
unbounded intervals at the ends, and the singletons of the sample points,
which have mass zero.  The open order intervals of $\cX$ form a class of VC
dimension at most two, so by the net bound of \cref{sec:prelims-tools}
\citep{haussler1987epsilon}, with probability at least $1-\delta/2$, every
interval of mass greater than $\tau_p=\min\{1,C\log(e\npt/\delta)/\npt\}$ contains
a sample point.  On this event every cell has mass at most $\tau_p$, so
$V_{\cP}=\sum_R\cD_X(R)^2\le\tau_p\sum_R\cD_X(R)\le\tau_p$.  With $\npt$ as in the
statement and suitable constants,
$C\sqrt{\tau_p(d+m+\log(4/\delta))}\le\varepsilon$, and \cref{thm:vc-rounding},
with failure probability $\delta/2$ and without the term $2/K$, gives the
claim.  For swap error, replace $d+m$ by $md$.
\end{proof}

\begin{proof}[Proof of \cref{lem:budget}]
For a sufficiently large absolute constant $C$, let
\[
  \tau_c=\min\Bigl\{1,\frac{C\log(e\ncf/\delta_0)}{\ncf}\Bigr\},\qquad
  \tau_p=\min\Bigl\{1,\frac{C\log(e\npt/\delta_0)}{\npt}\Bigr\}.
\]
If $\tau_c=1$ or $\tau_p=1$, the claim is trivial: since $r_x\le1$,
$V_{\cP}\le\sum_{x\in O}\cD_X(\{x\})^2+\sum_{R\in\cP_{\mathrm{rest}}}\cD_X(R)^2\le1$.  So assume
$\tau_c,\tau_p<1$, and write $\omega_x=\cD_X(\{x\})$.

\emph{Heavy atoms.}  At most $1/\tau_c\le\ncf$ atoms have mass greater than
$\tau_c$.  For such an atom, $n_x\sim\mathrm{Bin}(\ncf,\omega_x)$, and the
Chernoff bound gives $\Pr\{n_x<\ncf\omega_x/2\}\le e^{-\ncf\omega_x/8}\le
e^{-\ncf\tau_c/8}\le\delta_0/\ncf$ for $C$ large.  By a union bound, with
probability at least $1-\delta_0$, every atom of mass greater than $\tau_c$ has
$n_x\ge\ncf\omega_x/2\ge2$.  On this event, each such atom is in $O$ and
has $r_x^2\le\iota/n_x\le2\iota/(\ncf\omega_x)$, so its term in $V_{\cP}$ is
$\omega_x^2r_x^2\le2\iota\omega_x/\ncf$; these terms sum to at most
$2\iota/\ncf$.

\emph{Light atoms and gaps.}  The atoms in $O$ of mass at most $\tau_c$
contribute at most $\sum_x\omega_x^2\le\tau_c$, since $r_x\le1$.  By the net bound of \cref{sec:prelims-tools} for the open order intervals, which form a class of
VC dimension at most two, with probability at least $1-\delta_0$ every open
order interval that contains no point of the partition sample has mass at
most $\tau_p$.  On both events, every cell of $\cP_{\mathrm{rest}}$ has mass at
most $\max\{\tau_c,\tau_p\}$: a singleton $\{z_i\}\setminus O$ either has mass zero or
is an atom not in $O$, which on the heavy-atom event has mass at most
$\tau_c$, and the other cells lie in intervals
that contain no sample point.  Hence
$\sum_{R\in\cP_{\mathrm{rest}}}\cD_X(R)^2\le\max\{\tau_c,\tau_p\}\le\tau_c+\tau_p$.

Adding the three contributions, and using
$\iota\le C\log(e\ncf/\delta_0)$, gives the claim with probability at least
$1-2\delta_0$.
\end{proof}

\begin{proof}[Proof of \cref{prop:batch-hints}]
We follow the proof of \cref{thm:batch}; only the current-round
cancellation changes.  Write $Q_t=Q_t(X_t)$, $\mu_t=\mu_t(X_t)$, and
$c_t=\sum_vQ_t(v)\widehat A_{t,v}(X_t)\in[-1,1]$.  On the estimation
event~\eqref{eq:approx}, the computation in the proof of \cref{thm:online}
gives $\sum_\sigma w_t(\sigma)s_{t,\sigma}\le(\mu_t-Y_t)c_t+\alpha$.
Conditional on the past, the auxiliary draws, and $X_t$, the outcome $Y_t$
has mean $\fs(X_t)$, and the intervals are fixed, so
\[
  \E\bigl[(\mu_t-Y_t)c_t\,\big|\,\text{past},\text{auxiliary draws},X_t\bigr]
  =(\mu_t-\fs(X_t))\,c_t\le0,
\]
by \cref{lem:root-sub} with $\Lambda'=\Lambda_{X_t}$ when $X_t\in O$, since
valid intervals give $\fs(X_t)\in I_{X_t}\subseteq[\vlo_{X_t},\vhi_{X_t}]$,
and by \cref{lem:root} otherwise.  By Azuma's inequality,
$\sum_t(\mu_t-Y_t)c_t\le C\sqrt{T\log(8/\delta)}$ with probability at least
$1-\delta/8$.  This term is absorbed in the term
$C\sqrt{T\log(16N/\delta)}$ of the exponential weights step, and adding
$\delta/8$ for Azuma's inequality to the $3\delta/8$ of the proof of
\cref{thm:batch} gives exactly $\delta/2$.  Hence~\eqref{eq:restricted}
holds with probability at least $1-\delta/2$, and the rest of the proof of
\cref{thm:batch} applies unchanged.
\end{proof}

\section{Proper-loss guarantees and multicalibration}
\label{app:proper}

This appendix records a standalone observation, which no other result in
the paper uses.  Deterministic omnipredictors for proper losses
\citep{gibbs2026sample} do not directly give multicalibration.  The
following bridge requires competing with predictors that depend on the
learned predictor's values, and loses a square root.

\begin{proposition}[Brier competition against value-dependent patches]
\label{prop:brier}
Let $p$ be a deterministic predictor with values in $\Lambda$, and suppose
that its squared error $\E[(p(X)-Y)^2]$ is at most $e_{\mathrm{sq}}$ larger than that of
every predictor of the form
$p'(x)=\clip\bigl(p(x)+\lambda_{p(x)}h_{p(x)}(x),[0,1]\bigr)$ with
$h_v\in\cH$ and $\lambda_v\in[-1,1]$.  Then $\SMC(p)\le\sqrt{e_{\mathrm{sq}}}$.  If the
condition holds only for $h_v$ equal to a common $h\in\cH$, then
$\MC(p)\le\sqrt{e_{\mathrm{sq}}}$.
\end{proposition}

\begin{proof}
Let $b_v=\Pr\{p(X)=v\}$ and $A_v=\E[h_v(X)\mathbf 1\{p(X)=v\}(p(X)-Y)]$.  For
$b_v>0$ let $\lambda_v=-A_v/b_v\in[-1,1]$.  The unclipped update decreases
the squared error by $\sum_{v:b_v>0}A_v^2/b_v$, and clipping to $[0,1]$ can
only decrease it further since $Y\in[0,1]$.  By the Cauchy--Schwarz
inequality and $\sum_vb_v=1$, $\sum_vA_v^2/b_v\ge(\sum_v|A_v|)^2$.  Hence
$\sum_v|A_v|\le\sqrt{e_{\mathrm{sq}}}$; choosing auditors arbitrarily close to the
optimum at each value proves the first claim, and the same choice with a
single auditor proves the second.
\end{proof}

\section{Parameter choices}
\label{app:params}

We list the parameters of \cref{alg:main}.  Let $C$ be a sufficiently large
absolute constant.  The choice of $T$ below is explicit once the absolute
constants of \cref{thm:batch} and the constant $C_0$ of \cref{thm:hybrid}
are fixed.

\emph{Randomized mode.}  Let $K=\lceil C/\varepsilon\rceil$, $m=K+1$,
$\alpha=\varepsilon/(C\sqrt m)$, and choose $T$ as the smallest integer for
which the bound of \cref{thm:batch} with $\cS=\Shad$, $N=2\nu<4m$, and
$\kappa=\sqrt m$ is at most $\varepsilon$; by \cref{thm:hybrid} this is
$T=\wtO(m(d+\log(1/\delta))/\varepsilon^2)=\wtO((d+\log(1/\delta))\varepsilon^{-3})$.
Let $L=\lceil2\alpha^{-2}\log(16mT/\delta)\rceil=\wtO(\varepsilon^{-3})$ and
$\eta=\min\{1,\sqrt{\log N/T}\}$.  The number of ERM calls is $T(N+L)$.

\emph{Deterministic mode.}  Use the randomized-mode parameters for
accuracy $\varepsilon/2$ and confidence $\delta/4$; in particular
$K=\lceil2C/\varepsilon\rceil\ge16/\varepsilon$.  Let
$D_{\mathrm{par}}=d+m+\log(8/\delta)$, the complexity term of
\cref{thm:vc-rounding} at failure probability $\delta/4$.  Choose $\ncf$ and $\npt$ so that
the bound of \cref{lem:budget}, with $\delta_0=\delta/8$, is at
most $\varepsilon^2/(C\,D_{\mathrm{par}})$, which for $C$ large enough makes the
rounding term of \cref{thm:vc-rounding} at most $\varepsilon/4$; for
instance,
$\ncf=\npt=\lceil C'D_{\mathrm{par}}\varepsilon^{-2}
\log(eD_{\mathrm{par}}/(\varepsilon\delta))\rceil$ for a suitable absolute
constant $C'$.  With $m=\Theta(1/\varepsilon)$ these are
$\wtO(\varepsilon^{-3}+(d+\log(1/\delta))\varepsilon^{-2})$.

\end{document}